\documentclass{article}

\PassOptionsToPackage{numbers}{natbib}

\usepackage[final]{neurips_2022}

\usepackage[utf8]{inputenc} 
\usepackage[T1]{fontenc}    
\usepackage{hyperref}       
\usepackage{url}            
\usepackage{booktabs}       
\usepackage{amsfonts}       
\usepackage{nicefrac}       
\usepackage{microtype}      
\usepackage{xcolor}         

\usepackage{subcaption}
\usepackage{amsmath}
\usepackage{amssymb}
\usepackage{mathtools}
\usepackage{amsthm}
\usepackage{latexsym, color}
\usepackage{verbatim}
\usepackage{wrapfig}
\usepackage{makecell}

\theoremstyle{plain}
\newtheorem{mydef}{Definition}
\newtheorem{myassm}{Assumption}
\newtheorem{mythm}{Theorem}
\newtheorem{mylem}{Lemma}

\newtheorem{mycor}{Corollary}
\newtheorem{myprop}{Proposition}

\def\tcr{\textcolor{red}}
\def\tcb{\textcolor{blue}}

\def\rm{\textrm}
\def\mcal{\mathcal}

\def\KL{\text{KL}}
\def\Pr{\text{Pr}}
\def\cvar{\text{CVaR}}
\def\Ebb{\mathbb{E}}
\def\Rbb{\mathbb{R}}
\def\Scal{\mathcal{S}}
\def\Acal{\mathcal{A}}

\def\qtargt{\frac{q^\pi_u(s_t) - c(s_t, a_t)}{\gamma}}

\def\pqtargt{p_{X^\pi \left(s_{t+1}\right)} \left( \frac{q^\pi_u(s_t) - c(s_t, a_t)}{\gamma}\right)}
\def\Fqtarg{F_{X^\pi \left( s' \right)} \left( \frac{q^\pi_u(s) - c(s, a)}{\gamma}\right)}
\def\Fqtargt{F_{X^\pi \left(s_{t+1}\right)} \left( \frac{q^\pi_u(s_t) - c(s_t, a_t)}{\gamma}\right)}

\title{Quantile Constrained Reinforcement Learning: \\A Reinforcement Learning Framework Constraining Outage Probability}

\author{%
  Whiyoung Jung, Myungsik Cho, Jongeui Park, Youngchul Sung\thanks{Corresponding author} \\
  School of Electrical Engineering, KAIST\\
  Daejeon 34141, Republic of Korea \\
  \texttt{\{wy.jung, ms.cho, jongeui.park, ycsung\}@kaist.ac.kr} \\
}

\begin{document}

\maketitle


\begin{abstract}
Constrained reinforcement learning (RL) is an area of RL whose objective is to find an optimal policy that maximizes expected cumulative return while satisfying a given constraint. Most of the previous constrained RL works consider expected cumulative sum cost as the constraint. However, optimization with this constraint cannot guarantee a target probability of outage event that the cumulative sum cost exceeds a given threshold. This paper proposes a framework, named Quantile Constrained RL (QCRL), to constrain the quantile of the distribution of the cumulative sum cost that is a necessary and sufficient condition to satisfy the outage constraint. This is the first work that tackles the issue of applying the policy gradient theorem to the quantile and provides theoretical results for approximating the gradient of the quantile. Based on the derived theoretical results and the technique of the Lagrange multiplier, we construct a constrained RL algorithm named Quantile Constrained Policy Optimization (QCPO). We use distributional RL with the Large Deviation Principle (LDP) to estimate quantiles and tail probability of the cumulative sum cost for the implementation of QCPO. The implemented algorithm satisfies the outage probability constraint after the training period.
\end{abstract}

\section{Introduction} \label{sec:introduction}

Reinforcement learning (RL) has been developed in the direction of finding an optimal policy that maximizes expected cumulative return for a given environment. Thus, most of the works in RL  consider only rewards given by the environment to optimize the policy. However,  many real-world control problems impose constraints on the behavior of a policy. Constrained RL is an area of RL whose objective is to find an optimal policy that maximizes expected cumulative return while satisfying a certain constraint on the cumulative cost.  A conventional constrained RL problem \eqref{problem:expectation} can be written as
\begin{equation}\tag{ExpCP}
    \begin{array}{ll} \text{Maximize} \quad &V^\pi(s_0) := \Ebb_{\pi}\left[ \sum^\infty_{t=0} \gamma^t r(s_t, a_t) \right] \\ \text{subject to} \quad &C^\pi(s_0) := \Ebb_{\pi}\left[ \sum^\infty_{t=0} \gamma^t c(s_t, a_t) \right] \leq d_{th}, \end{array} \label{problem:expectation}
\end{equation}
where the cost constraint is  that the expectation of the sum of costs is less than or equal to a threshold parameter  $d_{th}$. Note that the threshold $d_{th}$ is set on the average (i.e., expectation) of the cumulative sum cost to avoid undesired high-cost events  in this formulation.   If we do not  want any event causing a positive cost, $d_{th}$ should be set as a sufficiently small value, i.e., $d_{th} \approx 0$. 
On the other hand, if we can afford events with low costs, we can set $d_{th}$ properly as we desire.  
Most of the previous constrained RL works solved the problem \eqref{problem:expectation} \cite{achiam2017constrained, ding2021provably, liu2020ipo, turchetta2020safe, xu2021crpo, yang2020projection, yu2019convergent} partly because the constraint on the expectation of the cumulative sum cost in  \eqref{problem:expectation} is well fit with the objective given by the expectation of the sum reward, and this makes the problem amenable.   However, solving the problem \eqref{problem:expectation} 
may have an undesirable outcome for real environments that typically need a constrained behavior on the event that the cost exceeds the threshold $d_{th}$. For example, in the case of an autonomous driving car, what we want for our sure safety is to control and limit the probability of accident itself. In the case of a telecommunication system, what we want to control is the probability of packet loss through the communication system. These probabilities are called `outage probability' in general. Thus, in many real-world systems, the system requires a constraint on the outage probability, i.e., the probability of critical or unsafe events. In this case, a constraint on the expectation of critical events, as in \eqref{problem:expectation}, cannot guarantee the desired target probability of critical events. To illustrate this, let us consider the following example.

\begin{wrapfigure}{r}{0.5\textwidth}
    \vspace{-15pt}
    \centering
    \begin{subfigure}[b]{0.24\columnwidth}
        \centering
        \includegraphics[width=\textwidth]{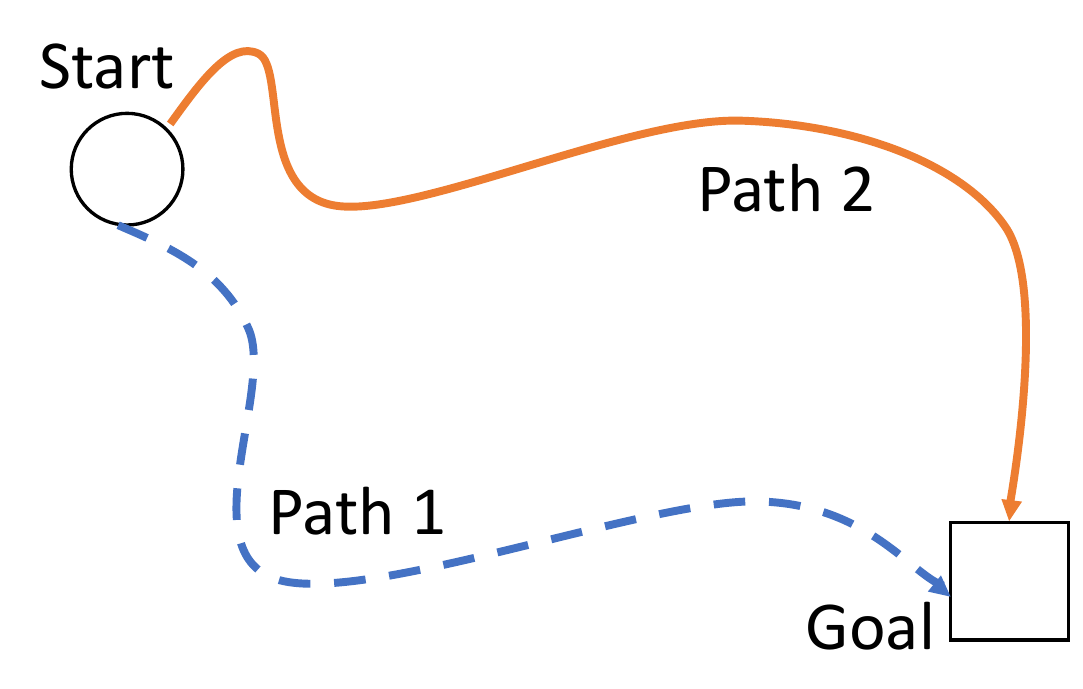}
        \caption{Two-path environment}
        \label{fig:two_path_env}
    \end{subfigure}
    \begin{subfigure}[b]{0.24\columnwidth}
        \centering
        \includegraphics[width=\textwidth]{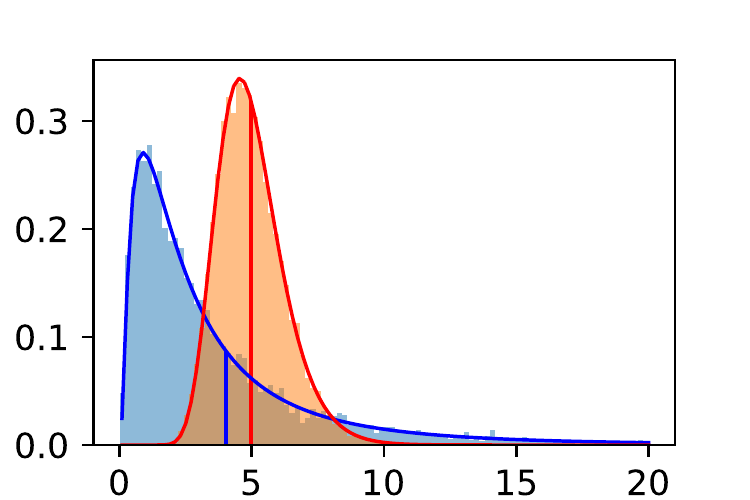}
        \caption{Distribution of cost}
        \label{fig:simple_example_cost_distribution}
    \end{subfigure}
    \caption{A simple example of environment.}
    \label{fig:simple_example}
    \vspace{-10pt}
\end{wrapfigure}

Consider a simple two-path environment, as shown in Fig. \ref{fig:two_path_env}. The objective of the environment is for an agent to reach the goal by driving a car. The environment gives a reward when the agent reaches the goal, and gives costs until the agent reaches the goal. There exist two paths to reach the goal, and each path has a different cost distribution. Fig. \ref{fig:simple_example_cost_distribution} shows the distribution of the cumulative sum cost (curve) and its mean (vertical line) for each path: blue for path 1 and red for path 2. As we can see in Fig. \ref{fig:simple_example_cost_distribution}, the distribution for path 1 has a lower average than that for path 2, but has a longer right tail than that for path 2. If we use the expectation of cost as a constraint, then following path 1 is a better choice, but in this case, the probability of a high cumulative sum cost (e.g. > 10.0 in Fig. \ref{fig:simple_example_cost_distribution}) is higher than following path 2. If the threshold of 10 represents a catastrophic event, we should constrain the probability of events exceeding 10 to a small value.   Thus, solving the problem with the expectation constraint \eqref{problem:expectation} does not necessarily have precise control over the target outage probability. When the event that the cumulative sum cost exceeds $d_{th}$ is a critical unsafe event, this means that such critical event can occur in high probability even if we solve the constrained problem \eqref{problem:expectation}. 
Therefore, in this paper, we aim to solve the following constrained RL problem with an outage probability constraint:
\begin{equation}\tag{ProbCP}
    \begin{array}{ll} \text{Maximize} \quad &V^\pi(s_0) = \Ebb_{\pi}\left[ \sum^\infty_{t=0} \gamma^t r(s_t, a_t) \right] \\ \text{Subject to} \quad &\Pr\left[ \sum^\infty_{t=0} \gamma^t c(S_t, A_t) > d_{th} \right] \leq \epsilon_0 \\
    &~~\text{for } S_0 = s_0, A_t \sim \pi(\cdot | S_t), S_{t+1} \sim M(\cdot | S_t, A_t). \end{array} \label{problem:probabilistic}
\end{equation}

Our approach to this problem is  first to convert the outage probability constraint in \eqref{problem:probabilistic} into a  quantile constraint $q^\pi_{1-\epsilon_0}(s_0) := \inf \left\{ x ~|~ \Pr\left(\sum^\infty_{t=0} \gamma^t c(S_t, A_t) \leq x \right) \geq 1-\epsilon_0 \right\} \leq d_{th}$ which is equivalent to the outage probability constraint
(See Fig. \ref{fig:relation_between_probcp_quantcp}), and then to solve the optimization: 
\begin{equation}
     \min_{\lambda \geq 0} \max_{\pi} ~ \{V^\pi(s_0) - \lambda \left( q^\pi_{1-\epsilon_0}(s_0) - d_{th} \right)\}, \label{eq:lagrangian_quantile}
\end{equation}

\begin{wrapfigure}{r}{0.5\textwidth}
    \vspace{-15pt}
    \centering
    \includegraphics[width=0.45\textwidth]{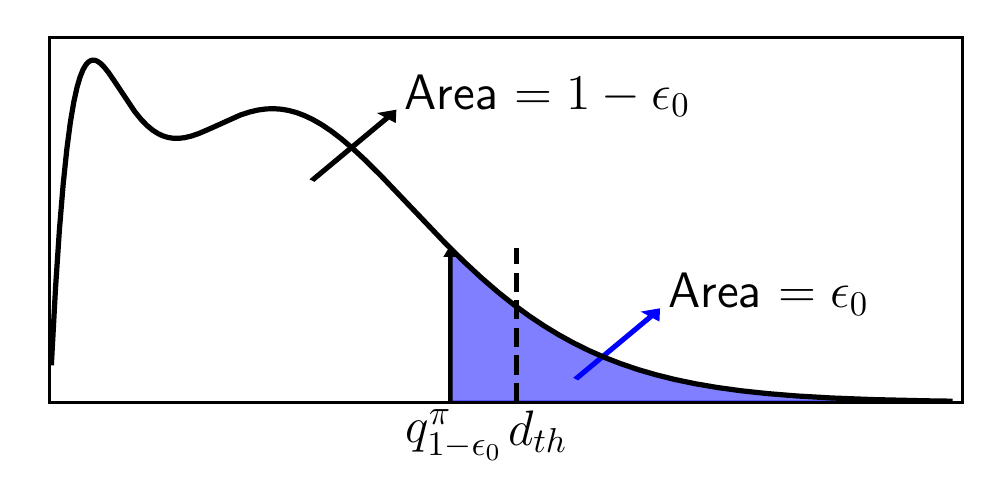}
    \caption{Equivalence between the outage probability constraint and the quantile constraint}
    \label{fig:relation_between_probcp_quantcp}
    \vspace{-10pt}
\end{wrapfigure}
where $\lambda$ is the Lagrange multiplier, based on policy gradient with the parameterized policy.  However, we note that the policy gradient theorem, which is the most basic theorem for on-policy RL, cannot be applied directly to compute the gradient of \eqref{eq:lagrangian_quantile} with respect to (w.r.t.) the policy parameter due to the quantile term $q^\pi_{1-\epsilon_0}(s_0)$. Therefore,  we  derive theoretical results for approximating the gradient of the quantile (in Section \ref{sec:QCRL}). Then, based on the derived theoretical results and the technique of the Lagrange multiplier, we construct our algorithm named Quantile Constrained Policy Optimization (QCPO) to solve the outage probability constrained RL problem (in Section \ref{sec:QCPO}). Here, we use distributional RL with the Large Deviations Principle (LDP) to estimate quantiles and tail probability of the cumulative sum cost for implementation of QCPO. The implemented algorithm satisfies the outage probability constraint after the training period. To the best of our knowledge, this is the first work  that tackles the issue of applying the policy gradient theorem to the quantile and obtains an (approximate) policy gradient for the quantile, and this is one of the main contributions of this paper, together with the QCPO algorithm.


\section{Background and Related Works} \label{sec:background_relative_works}

\textbf{Constrained RL} 
A constrained Markov decision process (CMDP) is defined as a tuple $\langle \Scal, \Acal, r, c, M, \gamma \rangle$, where $\Scal$ is the state space, $\Acal$ is the action space, $r: \Scal \times \Acal \rightarrow \Rbb$ is the reward function, $c: \Scal \times \Acal \rightarrow \Rbb_{\geq 0}$ is the cost function, $M: \Scal \times \Acal \times \Scal \rightarrow [0, 1]$ is the state transition probability, and $\gamma$ is the discount factor.

Constrained RL or Safe RL is an area of RL whose objective is to find an optimal policy for a given CMDP that maximizes the expected return $\Ebb_{\pi}\left[ \sum^{\infty}_{t=0} \gamma^t r(s_t, a_t) \right]$ while satisfying a constraint on the cumulative sum cost $\sum^{\infty}_{t=0} \gamma^t c(s_t, a_t)$. If one wants to constrain the average cumulative sum cost of the policy $\pi$, then constraint 
$C^\pi(s_0) := \Ebb_{\pi}\left[ \sum^\infty_{t=0} \gamma^t c(s_t, a_t) \right] \leq d_{th}$
can be considered. On the other hand, if one wants to constrain the outage behavior of the policy $\pi$, then constraint $\Pr\left( X^\pi(s_0) > d_{th} \right) \leq \epsilon_0$ should be considered, where $X^\pi(s_0)$ is a random variable defined as $X^\pi(s_0) := \sum^\infty_{t=0} \gamma^t c(S_t, A_t)$ with $S_0 = s_0$, $A_t \sim \pi(\cdot | S_t)$, $S_{t+1} \sim M(\cdot | S_t, A_t)$ for $t = 0, 1, \cdots$.

Most of the previous constrained RL works considered a constraint on the expectation of the cumulative sum cost: $C^\pi(s_0) \leq d_{th}$ \cite{achiam2017constrained, ding2021provably, liu2020ipo, turchetta2020safe, xu2021crpo, yang2020projection, yu2019convergent}. In order to solve the constrained optimization problem with this expectation-based constraint, researchers considered the Lagrangian multiplier method \cite{achiam2017constrained, stooke2020responsive}, Lyapunov-based methods \cite{chow2018lyapunov, chow2019lyapunov}, projection-based methods \cite{yang2020projection}, safety-layer methods \cite{dalal2018safe}. 
\vspace{-0.7em}
\begin{table}[ht]
    \caption{Comparison of constrained RL with probabilistic constraints}
    \label{table:comparison_other_algorithms}
    \centering
    \begin{tabular}{|c|c|c|c|c|}
        \hline
        Papers & Algorithm Type & Distribution Modeling & Theory & Deep RL       \\
        \hline
        \makecell{Risk-Const. RL \cite{chow2017risk}} &  \makecell{On-policy \\ (Trajectory-based) } & \makecell{No \\ (need only trajectory samples)} & Yes & No\\
        \hline
        \makecell{WCSAC \cite{yang2021wcsac}} & Off-policy & \makecell{Gaussian \\ (on all range of distribution)} & No & Yes\\
        \hline
        \makecell{QCPO \\ (This paper)} & \makecell{On-policy \\ (State-based)}  & \makecell{LDP with Weibull  \\ (only on tail) } & Yes & Yes\\
        \hline
    \end{tabular}
\end{table}

{\em Quantile (i.e., Value at Risk, VaR)} and {\em Conditional Value at Risk (CVaR)} are two well-known techniques to manage undesirable events in the domain of finance\cite{rockafellar2002conditional}. In the context of RL, the definitions of the quantile and the CVaR for the distribution of the cumulative sum cost for a given $\pi$ are given by $q^\pi_u(s_0) := \inf \{ x ~|~ \Pr(X^\pi(s_0) \leq x) \geq u \}$ and $\cvar^\pi_u(s_0) := \Ebb_{\pi}\left[ X^\pi(s_0) ~\vert~ X^\pi(s_0) \geq q^\pi_u(s_0) \right]$, respectively. 

The CVaR was previously used in RL to constrain undesirable events and the problem with a CVaR constraint is explicitly formulated as
\begin{equation}\tag{CVaR-CP}
    \begin{array}{ll} \text{Maximize} \quad &\Ebb_{\pi}\left[ \sum^\infty_{t=0} \gamma^t r(s_t, a_t) \right] \\ \text{Subject to} \quad &\cvar^\pi_{1-\epsilon_0}(s_0) \leq d_{th}.  \end{array} \label{problem:cvar}
\end{equation}
Note that the $(1-\epsilon_0)$-CVaR denoted as $\cvar^\pi_{1-\epsilon_0}(s)$ is always greater than or equal to  the $(1-\epsilon_0)$-quantile denoted as $q^\pi_{1-\epsilon_0}(s)$ for all $s \in \mcal{S}$ because of the definition of the CVaR. Therefore, satisfying the CVaR constraint $\cvar^\pi_{1-\epsilon_0}(s_0) \leq d_{th}$ in \eqref{problem:cvar} is a sufficient condition for satisfying the probabilistic constraint $\Pr\left( X^\pi(s_0) > d_{th} \right) \leq \epsilon_0$ in \eqref{problem:probabilistic}, and hence  \eqref{problem:cvar} is a stricter problem than \eqref{problem:probabilistic}. Therefore, algorithms proposed to solve \eqref{problem:cvar} can be used for solving \eqref{problem:probabilistic}, and this should satisfy the probabilistic constraint in theory.

\citet{chow2017risk} proposed a trajectory-based CVaR method and provided convergence for their method.  They used trajectory-based policy gradient to their Lagrangian, and it is simple to compute. Like most trajectory-based RL algorithms, however, it suffers from sample inefficiency since it collects a number of trajectories and updates its parameter once. Recently, \citet{yang2021wcsac} proposed an off-policy algorithm to solve \eqref{problem:cvar}. They only estimated the mean and variance of the cost distribution using a technique in distributional RL and computed $\cvar^\pi_u(s)$\footnote{They actually consider the CVaR $\cvar^\pi_u(s, a)$ of the cumulative sum cost $X^\pi(s,a)$ for a given $(s,a)$ pair.} as the CVaR of the Gaussian distribution of the estimated mean and variance. 
However, the distribution of $X^\pi(s)$  is not Gaussian in general, and the Gaussian approximation has limited capability to capture the decay rate of the tail probability because a Gaussian probability density function (PDF) has the form of $\exp(-\beta x^2)$ with fixed rate function $x^2$. 
Therefore, this algorithm can yield a poor estimation of the CVaR of the tail, especially for small tail probability, and cannot guarantee to satisfy the CVaR constraint (see Section \ref{sec:experiments}). Furthermore, note that the CVaR and the quantile are two  different measures for undesirable events, 
and the choice between the two depends on what  we desire. 
For example, an insurance company prefers  the CVaR of undesirable events to determine an insurance premium. On the other hand, a company developing an autonomous driving car system needs the quantile of undesirable events to guarantee the accident probability for safety. Thus, in the context of safe learning, our work  focuses directly on {\em the constraint on the quantile}, which is an equivalent (i.e., necessary and sufficient) constraint to the outage probability constraint in \eqref{problem:probabilistic}. To the best of our knowledge, this is the first work that provides a state-based policy gradient for the quantile and required theoretical results  regarding the quantile-constrained RL problem. Moreover, our implementation approximates  the tail distribution of $X^\pi(s)$ with a Weibull distribution (a particular case of  generalized Gamma distribution), which is general enough to capture various rates of decay of the tail probability. Table \ref{table:comparison_other_algorithms} summarizes the previous constrained RL methods and this paper.

\textbf{Large Deviation Principle (LDP)} 
Large deviation principle (LDP) \cite{dembo1998large} is a technique for estimating the limiting behavior of a sequence of distributions, especially on the tail. A simple example is the empirical mean $\bar{X}_n = \frac{1}{n} \sum^n_{k=1} X_k$ of i.i.d. random variables $X_i$. We say that a sequence $\{\bar{X}_n\}$ satisfies LDP if the sequence of its log probability distribution $\frac{1}{n} \log{\Pr\left( \bar{X}_n \in \Gamma \right)}$ satisfies the following condition $\frac{1}{n} \log{\Pr\left( \bar{X}_n \in \Gamma \right)} \overset{n \rightarrow \infty}{\longrightarrow} - \inf_{x \in \Gamma } I(x)$ for some function $I(x)$. The function $I(x)$ satisfying such limiting behavior is called the rate function of $\bar{X}_n$. The rate function $I(x)$ is also related to the cumulative distribution function (CDF) $F_{\bar{X}_n}(x)$ since $1 - F_{\bar{X}_n}(x_0) = \Pr\left(\bar{X}_n \in [x_0, \infty) \right) \approx \exp{\left(-n \inf_{x \in [x_0, \infty)} I(x) \right)}$ for some $x_0 > \Ebb[X]$ and sufficiently large $n$. LDP can be applied to finite-state Markov chains, and there exists a rate function for a given Markov chain \cite{dembo1998large}. In this paper, we consider the tail probability of the distribution of the cumulative sum cost $X^\pi(s_0) = \sum^\infty_{t=0} \gamma^t c(s_t, a_t)$. Finding its analytic rate function is hard. Therefore, we instead approximate the rate function directly as $I_{X^\pi(s)}(x) \approx (x / \beta(s))^{\alpha(s)}$  with learnable parameters $\alpha(s)$ and $\beta(s)$, which results in a Weibull distribution: $1 - F_{X^\pi(s)}(x) = \exp{\left\{ -(x / \beta(s))^{\alpha(s)} \right\}}$. We use this distribution to approximate the tail probability of $p_{X^\pi(s)}(x)$ of $X^\pi(s)$.


\section{Quantile Constrained RL} \label{sec:QCRL}

In this section, we explain an equivalent form of \eqref{problem:probabilistic} that we use to learn an optimal constrained policy under the outage probability constraint and then explain  the difficulty of applying the policy gradient theorem to optimize the Lagrangian of the equivalent problem. Finally, we provide theoretical results that circumvent this difficulty in Section \ref{subsec:theoretical_results}.

\subsection{Motivation: Problem of Applying Policy Gradient Theorem to Quantile} \label{subsec:motivation}


Solving \eqref{problem:probabilistic} with a direct approach is too hard in making a loss function for $\pi$ based on the outage probability. Thus, we convert the probability constrained problem to an equivalent form of a quantile constrained problem:
\begin{equation}\tag{QuantCP}
    \begin{array}{ll} \text{Maximize} \quad &\Ebb_{\pi}\left[ \sum^\infty_{t=0} \gamma^t r(s_t, a_t) \right] \\ \text{Subject to} \quad &q^\pi_{1-\epsilon_0}(s_0) \leq d_{th},  \end{array} \label{problem:quantile}
\end{equation}
where $q^\pi_u(s) = \inf \{ x ~|~ \Pr(X^\pi(s) \leq x) \geq u \}$  is the $u$-quantile of the random variable $X^\pi(s)$ of the cumulative sum cost: $X^\pi(s) = \sum^\infty_{t=0} \gamma^t c(S_t, A_t)$ with $S_0 = s$, $A_t \sim \pi(\cdot | S_t)$, $S_{t+1} \sim M(\cdot | S_t, A_t)$, $t = 0, 1, 2, \cdots$. 
Note that $q^\pi_{1-\epsilon_0}(s_0) \leq d_{th}$ is equivalent to $\Pr\left[ X^\pi(s_0) > d_{th} \right] \leq \epsilon_0$ due to the definition of the quantile. 
We propose a direct approach to solve the equivalent problem \eqref{problem:quantile} instead of \eqref{problem:probabilistic}. 
Although we have an equivalent form of \eqref{problem:probabilistic}, it is still difficult to solve \eqref{problem:quantile}. 
We explain what makes solving the problem \eqref{problem:quantile} still hard below. 

In the case of \eqref{problem:expectation}, the Lagrange-based optimization  of \eqref{problem:expectation} is given by $\min_{\lambda \geq 0} \max_{\pi}  L_{exp}(\pi, \lambda) := V^\pi(s_0) - \lambda \left( C^\pi(s_0) - d_{th} \right)$, where $V^\pi(s_0) = \Ebb_{\pi}\left[ \sum^\infty_{t=0} \gamma^t r(s_t, a_t) \right]$ and $C^\pi(s_0) = \Ebb_{\pi}\left[ \sum^\infty_{t=0} \gamma^t c(s_t, a_t) \right]$. Then, due to the form of $C^\pi(s_0)$, the policy gradient theorem \cite{sutton2018reinforcement} can directly be applied, and the gradient of the Lagrangian w.r.t. $\pi$ is given by the expectation form:
\begin{align}
    \nabla_\pi L_{exp}(\pi, \lambda) &= \sum_s \rho^\pi(s) \sum_a \nabla \pi(a | s) \left\{ A^\pi_r(s,a) - \lambda A^\pi_c(s,a) \right\},  \label{eq:policy_gradient_theorem_expectation}
\end{align}
where $\rho^\pi(s) := \sum^{\infty}_{t=0} \gamma^t ~ \Pr(S_t = s | s_0, \pi)$, $A^\pi_r(s, a) := r(s, a) + \gamma \Ebb_{s' \sim M}\left[ V^\pi(s') \right] - V^\pi(s)$, and $A^\pi_c(s, a) := c(s, a) + \gamma \Ebb_{s' \sim M}\left[ C^\pi(s') \right] - C^\pi(s)$. However, the gradient of the Lagrangian of the problem \eqref{problem:quantile}
\begin{equation}
    \min_{\lambda \geq 0} \max_{\pi}  L_{quant}(\pi, \lambda) := V^\pi(s_0) - \lambda \left( q^\pi_{1 - \epsilon_0}(s_0) - d_{th} \right) \label{eq:lagrange_quantile}
\end{equation}
w.r.t. the policy $\pi$ cannot be expressed as an expectation form:
\begin{align}
    &\nabla_\pi L_{quant}(\pi, \lambda)  \neq \Ebb_{\pi}\left[ \nabla \log{\pi(a|s)} \left\{ A^\pi_r(s,a) - \lambda \bar{A}^\pi_{1-\epsilon_0}(s,a) \right\} \right] \label{eq:policy_gradient_theorem_quantile_simple}
\end{align}
where $\bar{A}^\pi_{u}(s, a) := c(s, a) + \gamma \Ebb_{s' \sim M}\left[ q^\pi_u(s') \right] - q^\pi_u(s)$. This is because the $u$-quantile $q^\pi_u(s)$ is not the expectation of the cumulative sum cost. However, if the $u$-quantile $q^\pi_u(s)$ can be written as 
\begin{equation}
    q^\pi_u(s_0) = \Ebb_{\pi}\left[ \sum^\infty_{t=0} \gamma^t \left\{ c(s_t, a_t) + \tilde{c}_u(s_t, a_t) \right\} \right] \label{eq:expectation_form_quantile}
\end{equation}
for some function $\tilde{c}_u(s,a)$ that is independent of the policy $\pi$, we can apply the policy gradient theorem by defining the advantage function for the quantile term:
\begin{equation}
    A^\pi_u(s,a) := c(s,a) + \tilde{c}_u(s,a) + \gamma \Ebb_{s' \sim M}\left[ q^\pi_u(s') \right] - q^\pi_u(s).
\end{equation}
This fact motivates us to search for such $\tilde{c}_u(s,a)$. For this, under mild assumptions, we first show the existence of a policy-dependent additional cost $\tilde{c}^\pi_u(s,a)$ and then show that the additional cost can be approximated by a cost $\tilde{c}^{\pi'}_u(s,a)$ for some fixed $\pi'$ independent of $\pi$ except the requirement $\max_s \KL(\pi'(\cdot |s) ~\Vert~ \pi(\cdot |s)) \leq \delta$.

\subsection{Theoretical Results} \label{subsec:theoretical_results}

We here provide theoretical results showing the existence of an additional cost $\tilde{c}^\pi_u(s, a)$ and showing that this can be approximated as another cost $\tilde{c}^{\pi'}_u(s, a)$ for a base policy $\pi'$ independent of $\pi$, only requiring  $\max_s \KL\left( \pi'(\cdot | s) ~\big\Vert~ \pi(\cdot | s) \right) \leq \delta$ for some $\delta > 0$. These theoretical results make the quantile constrained policy optimization tractable by enabling application of the policy gradient theorem. For the theoretical results, we assume that the CDF $F_{X^\pi(s)}(x)$ is strictly increasing on $[0, \infty)$, and it is continuously differentiable for all $s \in \Scal$. The proofs of the theoretical results are in Appendix \ref{appendix:proof}.

We begin with deriving the temporal-difference (TD) relation between the $u$-quantiles of $X^\pi(s)$ at $s_t$ and $s_{t+1}$. Theorem \ref{thm:td_relation_quantile_bound} states the TD relation for the $u$-quantile under the following assumptions of boundness of quantile difference and smoothness of CDF of $X^\pi(s)$.
\begin{myassm}[Boundness of quantile difference] \label{assm:boundness}
    For a given policy $\pi$, the following two quantities are bounded
    \begin{align}
        \left\vert c(s,a) + \gamma q^\pi_u(s') - q^\pi_u(s) \right\vert &\leq \gamma R \\
        \left\vert q^\pi_u(s) - F^{-1}_{X^\pi(s)}\left( \Fqtarg \right) \right\vert &\leq R
    \end{align}
    for all $(s, a, s') \in \Scal \times \Acal \times \Scal$ such that $\pi(a | s) \cdot M(s' | s, a) > 0$.
\end{myassm}
\begin{myassm}[Smoothness of CDF of $X^\pi(s)$] \label{assm:smoothness}
    For each state $s$, the average slope of $F_{X^\pi(s)}(x)$ between $q^\pi_u(s)$ and $y \in [q^\pi_u(s)-R, q^\pi_u(s)+R]$ is bounded by
    \begin{equation}
        \frac{1}{1 + \epsilon} \cdot p_{X^\pi(s)} \left( q^\pi_u(s) \right) \leq \frac{F_{X^\pi(s)}\left( q^\pi_u(s) \right) - F_{X^\pi(s)}\left( y \right) }{ q^\pi_u(s) - y } \leq \frac{1}{1 - \epsilon} \cdot p_{X^\pi(s)} \left( q^\pi_u(s) \right)
    \end{equation}
    for small $0 < \epsilon < \frac{1}{2}$.
\end{myassm}
\begin{mythm} \label{thm:td_relation_quantile_bound}
    Under Assumptions \ref{assm:boundness} and \ref{assm:smoothness}, the $u$-quantile of the random variable $X^\pi(s_t)$ satisfies the following temporal-difference (TD) relation. For some constant $R$ and small $\epsilon > 0$,
    \begin{equation}
        \biggl\vert \Ebb_{\pi} \biggl[ \mu^{\pi}_u \left( s_t, a_t, s_{t+1} \right) \bigl\{ c(s_t, a_t) + \gamma q^\pi_u(s_{t+1}) - q^\pi_u(s_t) \bigr\} \biggr] \biggr\vert \leq \frac{\epsilon}{1-\epsilon} R, \label{eq:td_relation_quantile_bound}
    \end{equation}
    where $\mu^{\pi}_u \left( s, a, s' \right) = p_{X^\pi(s')}\left( \frac{q^\pi_u(s) - c(s, a)}{\gamma} \right) \bigm/ \gamma p_{X^\pi(s)}\left( q^\pi_u(s) \right)$. Here, the expectation is for the action $a_t \sim \pi(\cdot | s_t)$ and the next state $s_{t+1} \sim M(\cdot | s_t, a_t)$. ($s_t$ is given.)
\end{mythm}
Note that for the expectation of the cumulative sum cost $C^\pi(s)$ considered in \eqref{problem:expectation}, the expectation of TD under the policy $\pi$ follows 
\begin{equation}
    \Ebb_{\pi} \left[ c(s_t, a_t) + \gamma C^\pi(s_{t+1}) - C^\pi(s_t) \right] = 0 \label{eq:td_relation_expectation}
\end{equation}
by the Bellman equation. The TD relation \eqref{eq:td_relation_expectation} for expectation  has a similar form to that for the $u$-quantile \eqref{eq:td_relation_quantile_bound}, but the  difference is that \eqref{eq:td_relation_quantile_bound} is the weighted expectation of the TD ($c(s_t, a_t) + \gamma q^\pi_u(s_{t+1}) - q^\pi_u(s_t)$). 
The numerator $p_{X^\pi(s_{t+1})}\left( \frac{q^\pi_u(s_t) - c(s_t, a_t)}{\gamma} \right)$ of the weight $\mu_u^\pi(s_t,a_t,s_{t+1})$ in \eqref{eq:td_relation_quantile_bound} involves two quantities: 1) a target quantile $\frac{q^\pi_u(s_t) - c(s_t, a_t)}{\gamma}$ and 2) the PDF of the sum of costs $X^{\pi}(s_{t+1})=\sum^\infty_{k=0} \gamma^{k} c(s_{t+k+1}, a_{t+k+1})$ starting from state $s_{t+1}$.  
Here, the  value $\frac{q^\pi_u(s_t) - c(s_t, a_t)}{\gamma}$ is the target value of the  sum of costs $\sum^\infty_{k=0} \gamma^{k} c(s_{t+k+1}, a_{t+k+1})$ from the next state $s_{t+1}$ such that the  sum of costs $\sum^\infty_{k=0} \gamma^{k} c(s_{t+k}, a_{t+k})$ for a given pair $(s_t, a_t)$ at $t$ is the $u$-quantile $q^\pi_u(s_t)$. 
Thus, the numerator $p_{X^\pi(s_{t+1})}\left( \frac{q^\pi_u(s_t) - c(s_t, a_t)}{\gamma} \right)$ of the weight $\mu_u^\pi(s_t,a_t,s_{t+1})$ in \eqref{eq:td_relation_quantile_bound} is the probability of the event that the cumulative sum cost starting from $(s_t, a_t)$ becomes the  $u$-quantile $q^\pi_u(s_t)$ at $s_t$ from the perspective of the next state $s_{t+1}$.  Based on Theorem \ref{thm:td_relation_quantile_bound}, we  obtain the following corollary: 
\begin{mycor} \label{cor:td_relation_quantile_bound}
    Under Assumptions \ref{assm:boundness} and \ref{assm:smoothness}, the $u$-quantile $q^\pi_u(s_t)$ of the random variable $X^\pi(s_t)$ is bounded as
    \begin{equation} \label{eq:lem1main}
        \biggl\vert q^\pi_u(s_t) - \Ebb_{\pi} \biggl[ \mu^{\pi}_u \left( s_t, a_t, s_{t+1} \right) \bigl\{ c(s_t, a_t) + \gamma q^\pi_u(s_{t+1}) \bigr\} \biggr] \biggr\vert \leq \frac{\epsilon}{1-\epsilon} R.
    \end{equation}
\end{mycor}
{\em Proof:} Note that the term $q_u^\pi(s_t)$ can go outside  the expectation in \eqref{eq:td_relation_quantile_bound} since the expectation is over $(a_t,s_{t+1})$. From eq. \eqref{appendix:eq:relation_pdf} in Appendix \ref{appendix:distributional_rl}, the expectation of the numerator of $\mu^{\pi}_u \left( s_t, a_t, s_{t+1} \right)$ is the same as the the denominator of the weight, i.e.,  $\Ebb_{\pi}\left[ p_{X^\pi(s_{t+1})}\left( \frac{q^\pi_u(s_t) - c(s_t, a_t)}{\gamma} \right) \right] = \gamma p_{X^\pi(s_t)}\left( q^\pi_u(s_t) \right)$, and this leads to $\mathbb{E}_\pi [\mu_u^\pi(s_t,a_t,s_{t+1})] = 1$. So, we have the claim. \hfill{$\square$}

As seen in Corollary \ref{cor:td_relation_quantile_bound}, the $u$-quantile at $s_t$ can be approximated as a weighted expectation of $c(s_t, a_t) + \gamma q^\pi_u(s_{t+1})$, and the weight is proportional to $p_{X^\pi(s_{t+1})}\left( \frac{q^\pi_u(s_t) - c(s_t, a_t)}{\gamma} \right)$. This means that the more probable is the pair $(s_t, a_t, s_{t+1})$ to achieve $q^\pi_u(s_t)$, the higher weight is multiplied to $c(s_t, a_t) + \gamma q^\pi_u(s_{t+1})$ for approximating $q^\pi_u(s_t)$. Furthermore, if we assume that the transition dynamics of CMDP are deterministic, i.e., $s_{t+1} = h(s_t, a_t)$ as in many real-world control problem,  we can approximate the $u$-quantile $q^\pi_u(s_0)$ at $s_0$ as the expectation of the sum of costs under a distorted policy $\tilde{\pi}_u$, as stated in the following lemma:

\begin{mylem} \label{lem:expectation_form_quantile}
    Suppose that the state transition dynamics are deterministic, i.e., $s_{t+1} = h(s_t, a_t)$. Then, under Assumptions \ref{assm:boundness} and \ref{assm:smoothness}, the $u$-quantile $q^\pi_u(s_0)$ of the random variable $X^\pi(s_0)$ is expressed as
    \begin{equation} \label{eq:lemma2main}
         \left\vert q^\pi_u(s_0) - \Ebb_{\tilde{\pi}_u} \left[ \sum^\infty_{t=0} \gamma^t c(s_t, a_t) \right] \right\vert \leq \frac{\epsilon R}{(1-\epsilon)(1-\gamma)},
    \end{equation}
    where $\tilde{\pi}_u(a | s) = \pi(a | s) \cdot \mu^{\pi}_u \left( s, a, h(s,a) \right)  \propto \pi(a | s) \cdot p_{X^\pi(h(s, a))}\left( \frac{q^\pi_u(s) - c(s, a)}{\gamma} \right)$.
\end{mylem}
Now, plugging \eqref{eq:lemma2main} into the quantile term in the Lagrangian \eqref{eq:lagrange_quantile} of the problem \eqref{problem:quantile}, we may apply the policy gradient theorem based on the chain rule since the $u$-quantile $q^\pi_u(s_0)$ is expressed as the expectation of the sum of costs. However, the gradient of $\tilde{\pi}_u$ w.r.t. $\pi$ for chain rule is too complicated due to the $\mu_u^\pi$ term in Lemma \ref{lem:expectation_form_quantile}.  Thus, we find another expectation form of $q^\pi_u(s_0)$ using an additional cost function $\tilde{c}^\pi_u(s, a)$, as stated in the following theorem: 
\begin{mythm} \label{thm:expectation_form_quantile_policy_dependent_cost}
Under deterministic dynamics $s_{t+1} = h(s_t, a_t)$ and  Assumptions \ref{assm:boundness} and \ref{assm:smoothness}, $q_u^\pi(s)$ can be expressed as 
    \begin{equation}
        \left\vert q^{\pi}_u (s_0) - \Ebb_{\pi} \left[ \sum^\infty_{t=0} \gamma^t \left\{ c(s_t, a_t) + \tilde{c}^\pi_u(s_t, a_t) \right\} \right] \right\vert \leq \frac{\epsilon R}{(1-\epsilon)(1-\gamma)}, \label{eq:quantile_approximation_policy_dependent_cost}
    \end{equation}
where $\tilde{c}^\pi_u(s,a) = \left( \mu^{\pi}_u \left( s, a, h(s,a) \right) - 1 \right) \cdot \left[ c(s, a) + \gamma q^\pi_u(h(s, a)) \right]$.
\end{mythm}

Note that the additional cost $\tilde{c}^\pi_u(s_t, a_t)$ in Theorem \ref{thm:expectation_form_quantile_policy_dependent_cost}  is a policy-dependent cost function.
Under an additional mild assumption, we can find an upper bound of \eqref{eq:quantile_approximation_policy_dependent_cost} which replaces the policy-dependent cost function $\tilde{c}^\pi_u(s,a)$ with another cost $\tilde{c}^{\pi'}_u(s,a)$ for some fixed $\pi'$ independent of $\pi$, only requiring $\max_s \KL\left( \pi'(\cdot | s) ~\big\Vert~ \pi(\cdot | s) \right) \leq \delta$ for some $\delta > 0$. The additional assumption  is as follows:
\begin{myassm}[Lipschitz continuity of $\tilde{c}^\pi_u(s,a)$ over $\pi$] \label{assm:additional_cost}
    For any given fixed $u \in (0, 1)$ and any policies $\pi$ and $\pi'$, there exists a coefficient $C_u$ such that 
    \begin{equation}
        \left\vert \tilde{c}^{\pi'}_u(s, a) - \tilde{c}^{\pi}_u(s, a) \right\vert \leq C_u \cdot \max_{s'} \KL \left( \pi'(\cdot | s') ~\big\Vert~ \pi(\cdot | s') \right), ~~~\forall s \in \Scal, a \in \Acal.
    \end{equation}
\end{myassm}
Basically, Assumption \ref{assm:additional_cost} is that the function $\tilde{c}_u^\pi$ as a function of $\pi$ is continuous, which is expected to be satisfied if there is no abrupt change in the associated distributions. 
With Assumption \ref{assm:additional_cost} and Theorem \ref{thm:expectation_form_quantile_policy_dependent_cost}, we obtain an expression for the quantile $q_u^\pi(s_0)$ as a form of desired expected sum:
\begin{mythm} \label{thm:expectation_form_quantile_policy_independent_cost}
    Under  deterministic dynamics $s_{t+1} = h(s_t, a_t)$ and Assumptions \ref{assm:boundness}, \ref{assm:smoothness}, and \ref{assm:additional_cost}, the $u$-quantile $q^\pi_u(s_0)$ is expressed as  the expectation of the sum of actual cost and a $\pi$-independent additional cost  $\tilde{c}^{\pi'}_u(s, a)$ for $\pi'$ satisfying $\max_s \KL(\pi'(\cdot |s) ~\Vert~ \pi(\cdot |s)) \leq \delta$:
    \begin{equation}
        \left\vert q^{\pi}_u (s_0) - \Ebb_{\pi} \left[ \sum^\infty_{t=0} \gamma^t \left\{ c(s_t, a_t) + \tilde{c}^{\pi'}_u(s_t, a_t) \right\} \right] \right\vert \leq \frac{\epsilon R}{(1-\epsilon)(1-\gamma)} + \frac{C_u}{1-\gamma} \delta. \label{eq:quantile_approximation_expectation_form}
    \end{equation}
\end{mythm}
By Theorem \ref{thm:expectation_form_quantile_policy_independent_cost}, we can approximate the $u$-quantile $q^\pi_u(s_0)$ as the expectation of the sum of costs plus $\pi$-independent additional costs $\tilde{c}^{\pi'}_u(s,a)$ for a base policy $\pi'$, and this approximation is tighter when the distance between the current policy $\pi$ and the base policy $\pi'$ is smaller. 
Theorem \ref{thm:expectation_form_quantile_policy_independent_cost} can be interpreted the other way around. As in the case of PPO \cite{schulman2017proximal}, if we first simply set  $\pi'$ as the policy before the update, denoted as $\pi_{old}$, then the updated  $\pi$ is near from the base policy $\pi_{old}=\pi'$, and we can compute the corresponding KL distance between $\pi_{old}$ (= $\pi'$) and $\pi$. Then, still, the inequality \eqref{eq:quantile_approximation_expectation_form} holds for $\delta = \max_s KL\left( \pi_{old}(\cdot | s) || \pi(\cdot | s) \right)$. Now, this result enables us to solve  the quantile constrained problem \eqref{problem:quantile} by applying the policy gradient theorem.


\section{Quantile Constrained Policy Optimization} \label{sec:QCPO}

\begin{wrapfigure}{r}{0.5\textwidth}
    \vspace{-20pt}
  \begin{center}
    \includegraphics[width=0.45\textwidth]{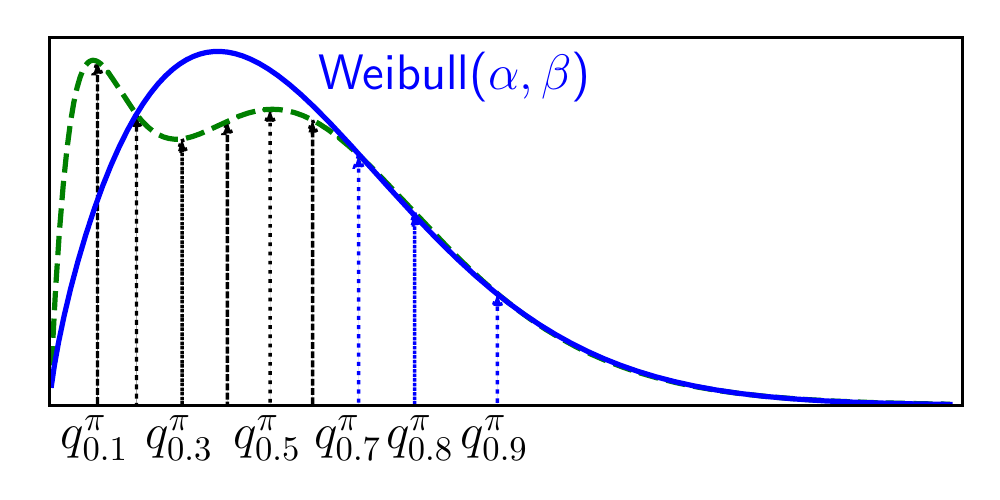}
  \end{center}
  \caption{Illustration of the quantile approximation and tail-probability approximation. Green-dash curve is the unknown PDF $p_{X^\pi(s)}(x)$ and up-arrow points  are the estimated $u$-quantiles $q^\pi_u(s)$. We approximate the PDF $p_{X^\pi(s)}(x)$ on the right tail by a Weibull distribution (blue curve) and this is approximated using the 4 rightmost quantile points (blue arrow).}
  \label{fig:quantile_and_right_tail_approximation}
  \vspace{-10pt}
\end{wrapfigure}

Using the theoretical results in Section \ref{sec:QCRL}, we now construct an algorithm named quantile constrained policy optimization (QCPO) to solve \eqref{problem:quantile} based on 
an on-policy RL algorithm:  PPO \cite{schulman2017proximal}.  The QCPO is a direct method to constrain the outage probability and  consists of three parts: 1) estimation of the $u$-quantile $q^\pi_u(s)$ of $X^\pi(s)$ for a given policy $\pi$, 2) estimation of PDF of $X^\pi(s)$ to compute the additional cost $\tilde{c}^{\pi'}_u(s,a)$, and 3) updating method of the Lagrange multiplier to control the outage probability. 
We first explain the overall structure of QCPO and the base loss function for the policy, which has a similar form to that in \cite{schulman2015trust}. Then, we provide a condition for policy improvement for the proposed method. The implementation of the proposed algorithm is based on the implementation of \cite{stooke2020responsive}, and details of the implementation, including the network structure, the loss functions, the Lagrangian multiplier update method, and the hyper-parameters, are in Appendix \ref{appendix:implementation_details}. The implementation code of QCPO is available at \href{https://github.com/wyjung0625/QCPO}{github.com/wyjung0625/QCPO}.

\subsection{Overall Structure of QCPO} \label{subsec:qcpo}

The agent of QCPO uses function approximators for the policy $\pi$, the value function $V^\pi(s)=\Ebb_{\pi}\left[ \sum^\infty_{t=0} \gamma^t r(s_t, a_t) \right]$ and the quantile function $q^\pi(s) = \left[q^\pi_{u_1}(s), q^\pi_{u_2}(s), \ldots, q^\pi_{u_{n_q}}(s) \right]$ of $X^\pi(s) = \sum^\infty_{t=0} \gamma^t c(s_t, a_t)$ for the policy $\pi$. These functions are parameterized by deep neural networks with parameters $\theta$, $\phi$, and $\psi$, respectively. We  denote $\theta_{old}$, $\phi_{old}$, and $\psi_{old}$ as their old parameters. 
Note that the quantile function $q_\psi(s)$ outputs $n_q$ values, and the $i$-th value  represents an estimate of the $u_i$-quantile $q^\pi_{u_i}(s)$ of $X^\pi(s)$ for  fixed target CDF values $\left[u_1, u_2, \ldots, u_{n_q}\right]$ with $u_1 < u_2 < \cdots < u_{n_q}$. 
In addition to these parameterized functions, we need another function that approximates the PDF of $X^\pi(s)$ on the right tail. (Here, the right-tail probability most matters since the target outage probability is typically small.)

As aforementioned in Section \ref{sec:background_relative_works}, we know that $X^\pi(s)$ follows LDP with a rate function $I_{X^\pi(s)}(x)$. However, finding the rate function $I_{X^\pi(s)}(x)$ in an analytic approach is hard. 
Therefore, in QCPO, the agent approximates the rate function of the form of $I_{X^\pi(s)}(x) \approx \left(  x / \beta(s)  \right)^{\alpha(s)}$ and learns the state-dependent parameters $\alpha(s)$ and $\beta(s)$ by using the quantiles on the right tail approximated by its quantile function for the right tail: $q^\pi_{u_{n_q - k + 1}}(s), \ldots, q^\pi_{u_{n_q}}(s)$. 
This approximation of the rate function results in approximation probability on the right tail as a Weibull distribution, whose tail distribution is $1 - F_{X^\pi(s)}(x) = e^{- \left(  x / \beta(s)  \right)^{\alpha(s)}}$. 
In order to obtain the state-dependent parameters $\alpha(s)$ and $\beta(s)$ for right-tail distribution approximation, we again parameterize them by neural networks with parameters $\xi$ and $\zeta$, respectively.
Fig.  \ref{fig:quantile_and_right_tail_approximation} shows both the quantile approximation and right-tail approximation of QCPO. Note that our approach actually learns the rate function governing the tail-probability decay rate, whereas the previous Gaussian approximation \cite{yang2021wcsac}  on the PDF of $X^\pi(s)$ fixes the rate function as quadratic $x^2$, which is not the correct rate function in general.

The overall procedure of QCPO is as follows: 1) estimate the value function $V^\pi(s)$ for return and estimate the quantile function $q^\pi_u(s)$, $u \in \left[u_1, u_2, \ldots, u_{n_q}\right]$ for the cumulative sum cost, 2) approximate tail distribution $p_{X^\pi(s)}(x)$ on the right tail using a Weibull distribution with parameters $\alpha(s)$, $\beta(s)$, 3) compute the additional cost for the base policy $\pi'$, $\tilde{c}^{\pi'}_{1-\epsilon_0}(s,a)$ for the quantile advantage $A^\pi_{1-\epsilon_0}(s, a) :=  c(s,a) + \tilde{c}^{\pi'}_{1-\epsilon_0}(s,a) + \gamma q^\pi_{1-\epsilon_0}(s') - q^\pi_{1-\epsilon_0}(s)$ , 4) take policy gradient using the sum of the value advantage and the quantile advantage $A_r(s,a) - \lambda A^\pi_{1-\epsilon_0}(s, a)$, 5) update the Lagrange multiplier $\lambda$. Since QCPO is based on PPO \cite{schulman2017proximal}, the loss functions for the policy and the value function are similar to those of PPO \cite{schulman2017proximal}. Please see Appendix \ref{appendix:loss_functions} and \ref{appendix:policy_loss} for detail.

\subsection{Policy Loss Function and Policy Improvement Condition} \label{subsec:policy_loss}
Let us consider  the policy loss function of QCPO to solve \eqref{problem:quantile}. The basic loss function of QCPO for a given Lagrange multiplier $\lambda$ is given by
\begin{align}
    L^{\pi_{old}}(\pi_\theta) - \tilde{C}_1 \max_s \KL(\pi_{old}(\cdot |s) ~\Vert~ \pi_{\theta}(\cdot |s)) \label{eq:policy_loss_function}
\end{align}
where 
\begingroup
\allowdisplaybreaks
\begin{align}
    L^{\pi_{old}}(\pi_\theta) &= \left( V^{\pi_{old}}(s_0) - \lambda q^{\pi_{old}}_{1-\epsilon_0}(s_0) \right) + \Ebb_{s \sim \rho^{\pi_{old}}, a \sim \pi_\theta}\left[  A^{\pi_{old}}_r(s_t, a) - \lambda A^{\pi_{old}}_{1-\epsilon_0}(s_t, a) \right] \\
    A^{\pi_{old}}_r (s, a) &= r(s,a) + \gamma \Ebb_{s' \sim M(\cdot | s, a)}\left[  V^{\pi_{old}}(s') \right] - V^{\pi_{old}}(s) \\
    A^{\pi_{old}}_{1-\epsilon_0}(s, a) &= c(s,a) + \tilde{c}^{\pi_{old}}_{1-\epsilon_0}(s, a) + \gamma \Ebb_{s' \sim M(\cdot | s, a)} \left[  q^{\pi_{old}}_{1-\epsilon_0}(s') \right] - q^{\pi_{old}}_{1-\epsilon_0}(s),
\end{align}
\endgroup
and $\pi_{old} := \pi_{\theta_{old}}$ is the policy that collects the most recent batch of samples, $\rho^{\pi_{old}}(s) := \sum^{\infty}_{t=0} \gamma^t ~ \Pr(S_t = s | s_0, \pi_{old})$ is the stationary state distribution under $\pi_{old}$, $\tilde{C}_1$ is a constant (see Appendix \ref{appendix:proof_improvement_theorem}),
Now, we consider the relationship between the actually-desired maximization objective $L_{quant}(\pi, \lambda) = V^\pi(s_0) - \lambda \left( q^\pi_{1 - \epsilon_0}(s_0) - d_{th} \right)$ in \eqref{eq:lagrange_quantile} and the practical QCPO objective $L^{\pi_{old}}(\theta) - \tilde{C}_1 \max_s \KL(\pi_{old}(\cdot |s) ~\Vert~ \pi_{\theta}(\cdot |s))$ in  \eqref{eq:policy_loss_function}. The relationship between the two is given by the following theorem.

\begin{mythm} \label{thm:policy_improvement_condition}
    Let $\pi_{new} := \pi_{\theta_{new}}$ be the solution of the problem of maximizing 
    \begin{align}
        L^{\pi_{old}}(\pi_\theta) - \tilde{C}_1 \max_s \KL(\pi_{old}(\cdot |s) ~\Vert~ \pi_{\theta}(\cdot |s))
    \end{align}
    for some constant $\tilde{C}_1 > 0$. Then, under deterministic dynamics $s_{t+1} = h(s_t, a_t)$ and Assumptions \ref{assm:boundness}, \ref{assm:smoothness}, and \ref{assm:additional_cost}, the following inequality holds:
    \begin{align}
        &L_{quant}(\pi_{new}, \lambda) - L_{quant}(\pi_{old}, \lambda) \\
        &\geq L^{\pi_{old}}(\pi_{new}) - L^{\pi_{old}}(\pi_{old}) - \tilde{C}_1 \KL_{max}(\pi_{old} || \pi_{new}) - \underbrace{\tilde{C}_2 \frac{\epsilon}{1-\epsilon}}_{\text{approximation loss}} \label{eq:theorem4}
    \end{align}
    for a given Lagrange multiplier $\lambda > 0$, some constant $\tilde{C}_2$ and small $\epsilon > 0$.
\end{mythm}

Note that the term $\tilde{C}_2 \frac{\epsilon}{1-\epsilon}$ in \eqref{eq:theorem4} is due to our approximation of the quantile as an expected sum to apply policy gradient. Therefore, by Theorem \ref{thm:policy_improvement_condition}, when the improvement $L^{\pi_{old}}(\pi_{new}) - L^{\pi_{old}}(\pi_{old}) - \tilde{C}_1 \KL_{max}(\pi_{old} || \pi_{new}) ~(>0)$ by the policy update from the QCPO loss function is large enough to compensate for the approximation loss, the desired quantity will also be improved by our policy update. That is, the Lagrangian for the quantile constrained problem for $\pi_{new}$ will be higher than that for $\pi_{old}$.


\section{Experiments} \label{sec:experiments}

\subsection{Environments} \label{subsec:environments}

\begin{wrapfigure}{r}{0.65\textwidth}
    \vspace{-70pt}
    \centering
    \begin{subfigure}[b]{0.2\columnwidth}
        \centering
        \includegraphics[width=\textwidth]{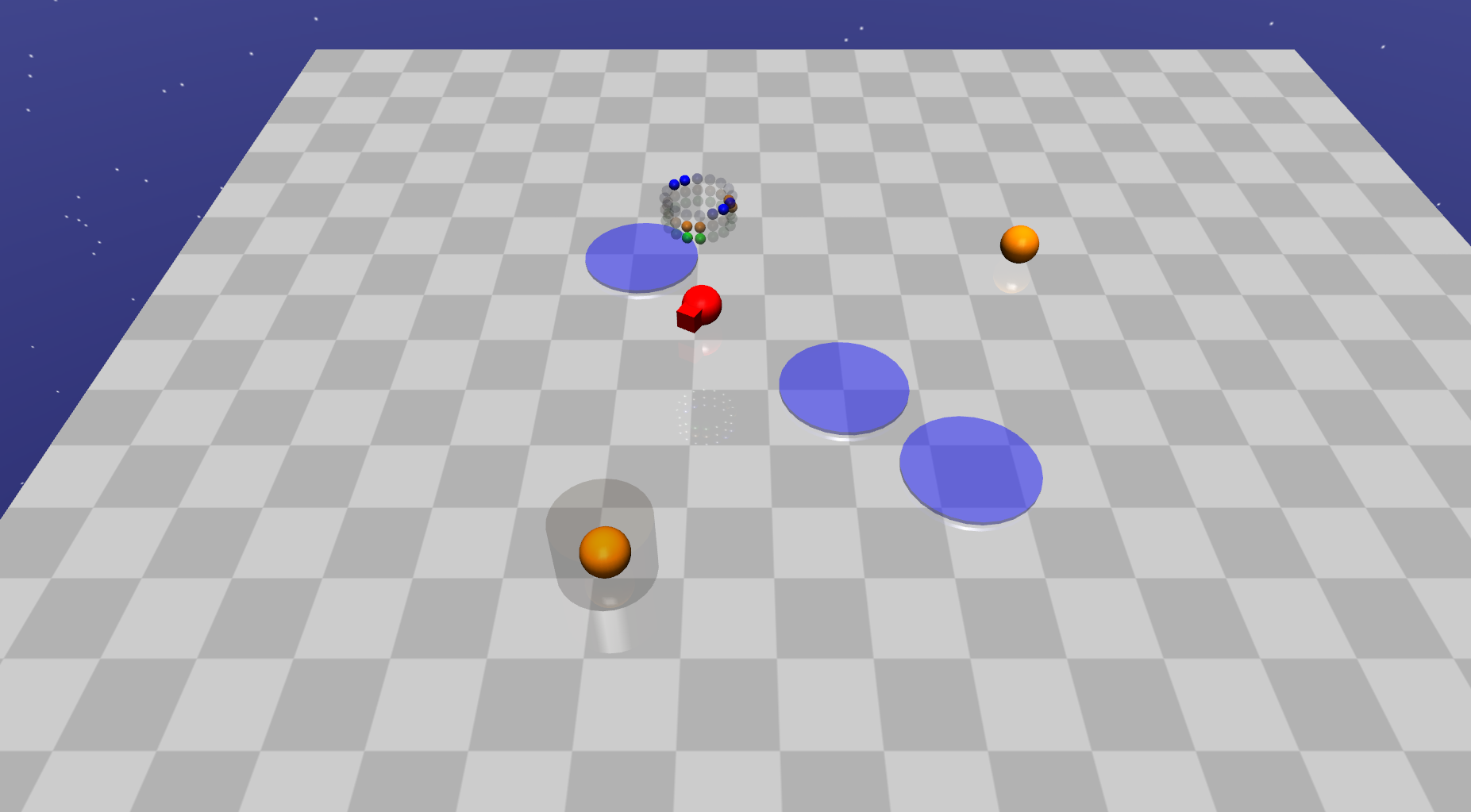}
        \caption{SimpleButtonEnv}
        \label{fig:SimpleButtonEnv}
    \end{subfigure}
    \begin{subfigure}[b]{0.2\columnwidth}
        \centering
        \includegraphics[width=\textwidth]{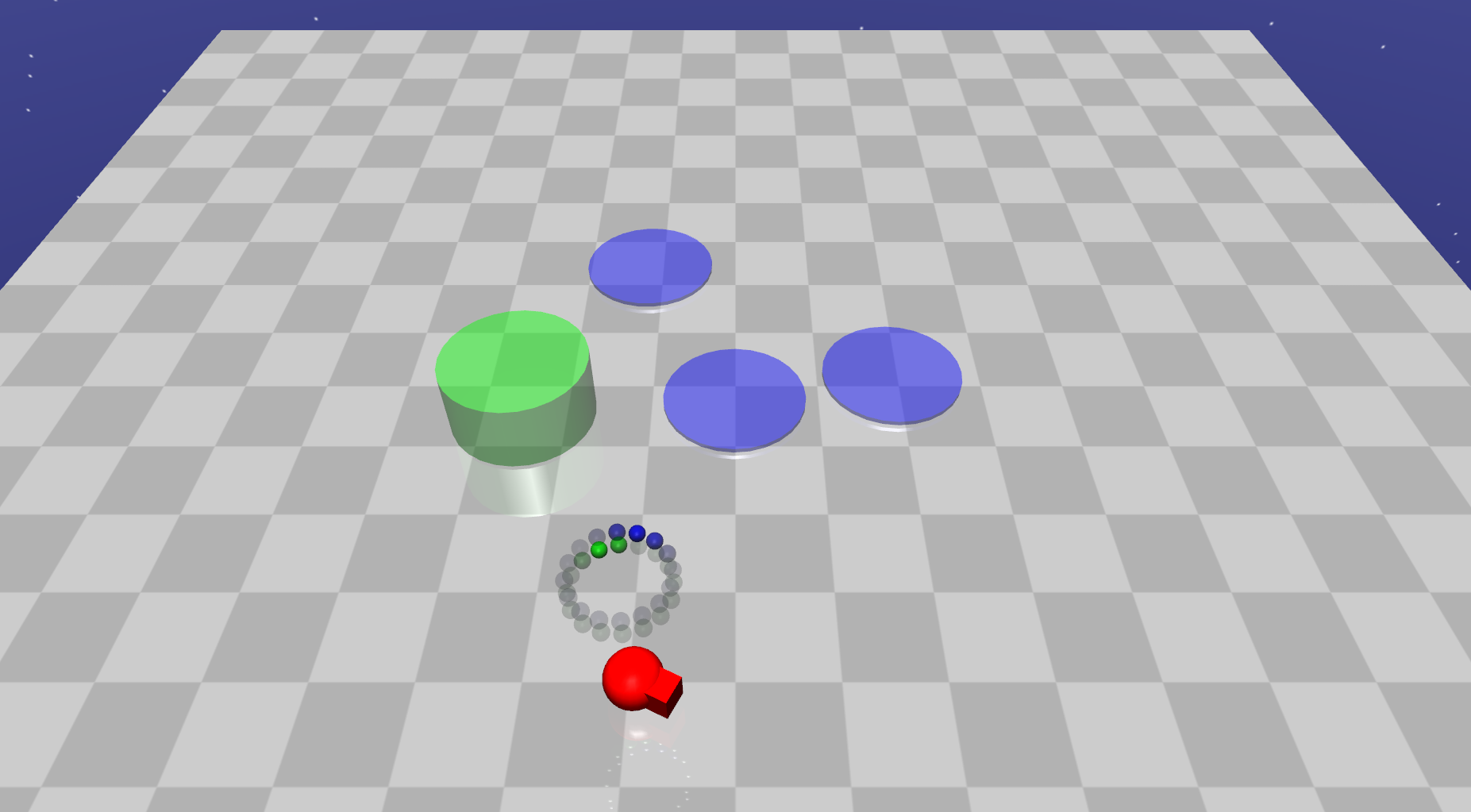}
        \caption{DynamicEnv}
        \label{fig:DynamicEnv}
    \end{subfigure}
    \begin{subfigure}[b]{0.2\columnwidth}
        \centering
        \includegraphics[width=\textwidth]{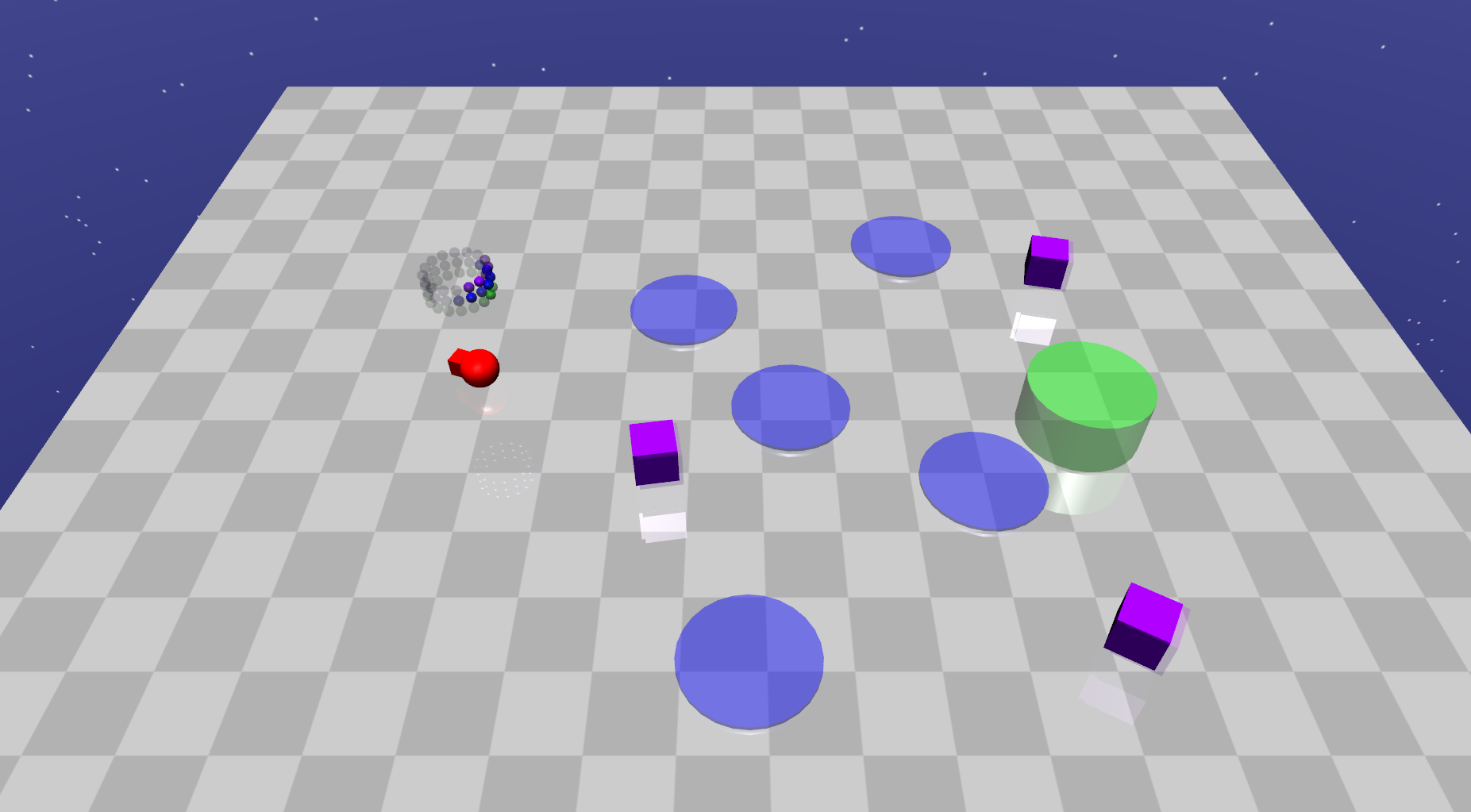}
        \caption{GremlinEnv}
        \label{fig:GremlinEnv}
    \end{subfigure}
    \caption{The considered environments}
    \label{fig:environments}
    \vspace{-10pt}
\end{wrapfigure}

We examined the performance of the proposed QCPO and compared it to that of WCSAC, which uses the CVaR constraint. The environments we considered are SimpleButtonEnv, DynamicEnv \cite{yang2021wcsac}, and GremlinEnv, which are based on Safety Gym \cite{ray2019benchmarking}, MuJoCo \cite{todorov2012mujoco}, and OpenAI Gym \cite{openaigym}. The environments can be considered as simplified versions of a real environment of an automatic serving robot and are illustrated in Fig. \ref{fig:environments}. The goal of these environments is for a robot (red sphere) to reach a goal (orange sphere wrapped by a grey translucent pillar, or green pillar) while avoiding the non-goal button (orange sphere), hazards (blue circle) or moving gremlins (purple box). Once the robot reaches the current goal, the environments generate the next goal deterministically (SimpleButtonEnv) or randomly (DynamicEnv, GremlinEnv), so the task complexity increases in the order of SimpleButtonEnv, DynamicEnv, and GremlinEnv. When the robot performs an action at time $t$, it receives a reward $\left\{ \left\Vert p_{t+1} - p_{\text{goal}} \right\Vert_2 - \left\Vert p_{t} - p_{\text{goal}} \right\Vert_2  \right\} + 1_{\text{goal reached}}$, where $p_{t}$ is the position ($x$, $y$) of the robot at time $t$ and $p_{\text{goal}}$ is the current goal position at time $t$. It also receives a cost $+1$ if the robot touches a non-goal object (the non-goal button, a hazard, or a gremlin) and $0$ otherwise. Thus, for the robot, it receives a higher return when the robot touches more goals in  maximum timesteps $T=1000$, and a higher sum of costs when the robot touches one of the other objects more often. A more detailed explanation of the environments is in Appendix \ref{appendix:environments}.

\subsection{Empirical Results} \label{subsec:empirical_results}

We compared the performance of the proposed algorithm (QCPO) with that of PPO with the Lagrangian multiplier method (PPO\_Lag)\footnote{We used the implementation code in \url{https://github.com/astooke/rlpyt/tree/master/rlpyt/projects/safe}, (MIT License)} for \eqref{problem:expectation} and that of WCSAC \cite{yang2021wcsac}\footnote{We used the github code that the authors of the paper uploaded: \url{https://github.com/AlgTUDelft/WCSAC}, (MIT License)} for \eqref{problem:cvar} which is a  stricter problem than \eqref{problem:probabilistic}. We set the threshold $d_{th} = 15$ in \eqref{problem:expectation}, \eqref{problem:cvar}, and \eqref{problem:quantile} and the target outage probability $\epsilon_0 = 0.1, 0.2$ in \eqref{problem:cvar} and \eqref{problem:quantile}.

Fig. \ref{fig:result} shows the results of the considered algorithms on SimpleButtonEnv, DynamicEnv, and GremlinEnv. All experiments were done with 10 different random seeds, and the real line and the shaded area represent the average and average $\pm$ standard deviation, respectively. PPO with the Lagrangian multiplier method for \eqref{problem:expectation} (green) keeps the average of the sum cost around the threshold $d_{th} = 15$ well (please see the graph in Appendix \ref{appendix:performance_comparison}), and its outage probability becomes around $0.35$ as we can observe in Fig. \ref{fig:SimpleButtonEnv_ProbOutage}, \ref{fig:DynamicEnv_ProbOutage}, and \ref{fig:GremlinEnv_ProbOutage}. As aforementioned, the CVaR approach (WCSAC) should satisfy a sufficient condition for satisfying the outage probability constraint in \eqref{problem:probabilistic}. It is seen  that WCSAC ($\epsilon_0 = 0.2$ (purple), $\epsilon_0 = 0.1$ (red)) achieves a lower or similar outage probability to the threshold $\epsilon_0$ in Fig. \ref{fig:SimpleButtonEnv_ProbOutage}, but the algorithm does not satisfy the outage probability constraint exactly in Fig. \ref{fig:DynamicEnv_ProbOutage} and \ref{fig:GremlinEnv_ProbOutage}. This means that the Gaussian distribution approximation of the distribution of $X^\pi(s)$ has limited capability to capture the decay rate of the tail probability. On the other hand, the proposed QCPO ($\epsilon_0 = 0.2$ (blue), $\epsilon_0 = 0.1$ (orange)) maintains the outage probability around the desired target outage probability very well, as shown in Fig. \ref{fig:SimpleButtonEnv_ProbOutage}, \ref{fig:DynamicEnv_ProbOutage}, and \ref{fig:GremlinEnv_ProbOutage}.

Now consider the average return of these algorithms. In constrained RL, in general, if an algorithm is allowed to have a higher sum of costs, then it has a higher return. Thus, as seen in Fig. \ref{fig:SimpleButtonEnv_ProbOutage}, \ref{fig:DynamicEnv_ProbOutage}, and \ref{fig:GremlinEnv_ProbOutage}, PPO\_Lag induces the highest outage probability, so it has the highest average return, as shown in Fig. \ref{fig:SimpleButtonEnv_Return}, \ref{fig:DynamicEnv_Return}, and \ref{fig:GremlinEnv_Return}. The direct comparison between WCSAC and QCPO is less meaningful in DynamicEnv and GremlinEnv, since WCSAC does not satisfy the outage probability constraint, but it is fair in SimpleButtonEnv because both algorithms satisfy the outage probability constraint. As seen in Fig. \ref{fig:SimpleButtonEnv_Return}, QCPO achieves a higher average return than WCSAC for the same target probability constraint $\epsilon_0=0.1, 0.2$. This is because QCPO satisfies the target outage probability exactly, i.e., uses the given cost budget fully for a higher return. 
We provided more results in Appendix \ref{appendix:more_results}.

\begin{figure}[t]
    \centering
    \begin{subfigure}[b]{0.32\columnwidth}
        \centering
        \includegraphics[width=\textwidth]{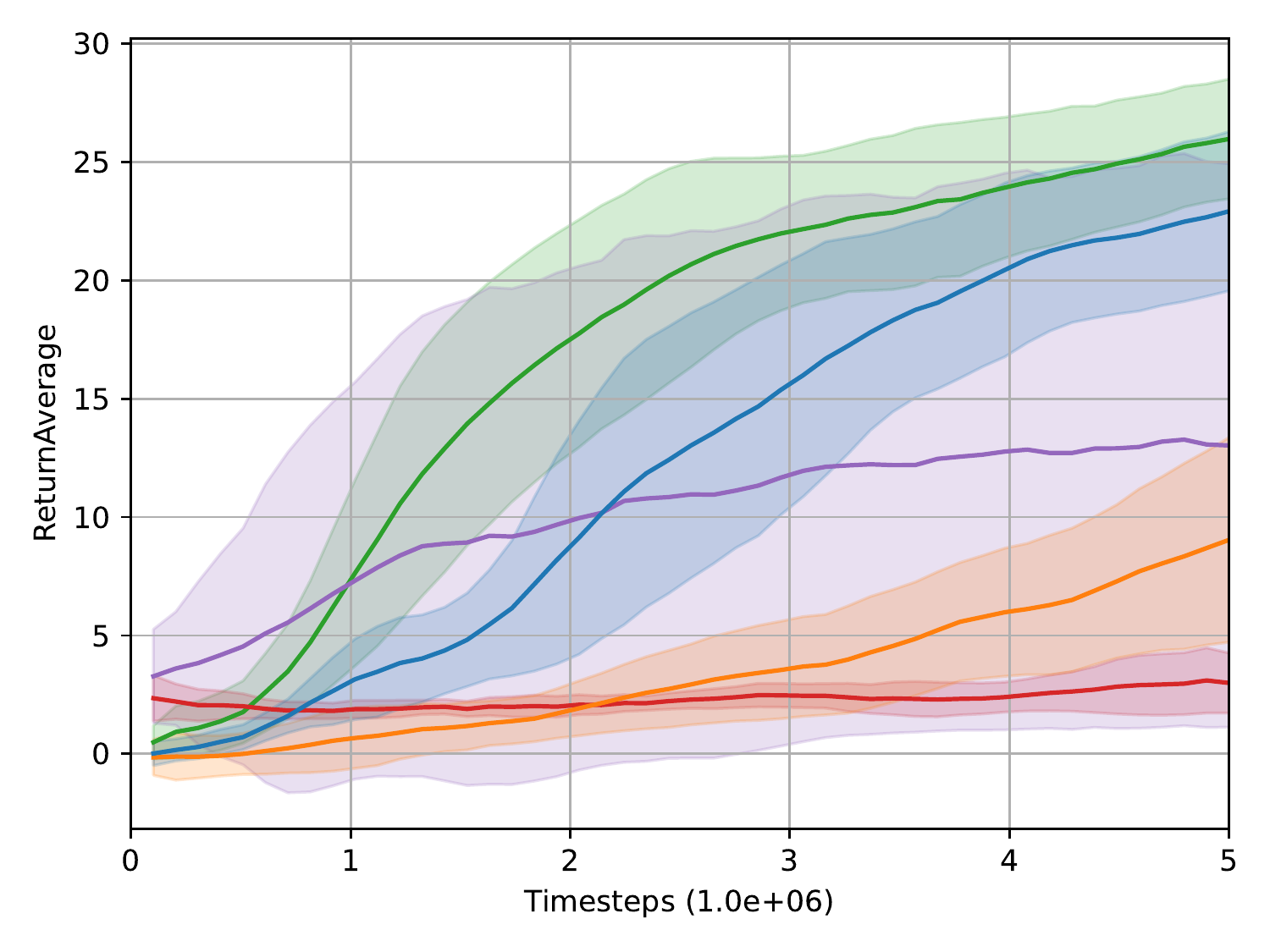}
        \caption{Average Return}
        \label{fig:SimpleButtonEnv_Return}
    \end{subfigure}
    \begin{subfigure}[b]{0.32\columnwidth}
        \centering
        \includegraphics[width=\textwidth]{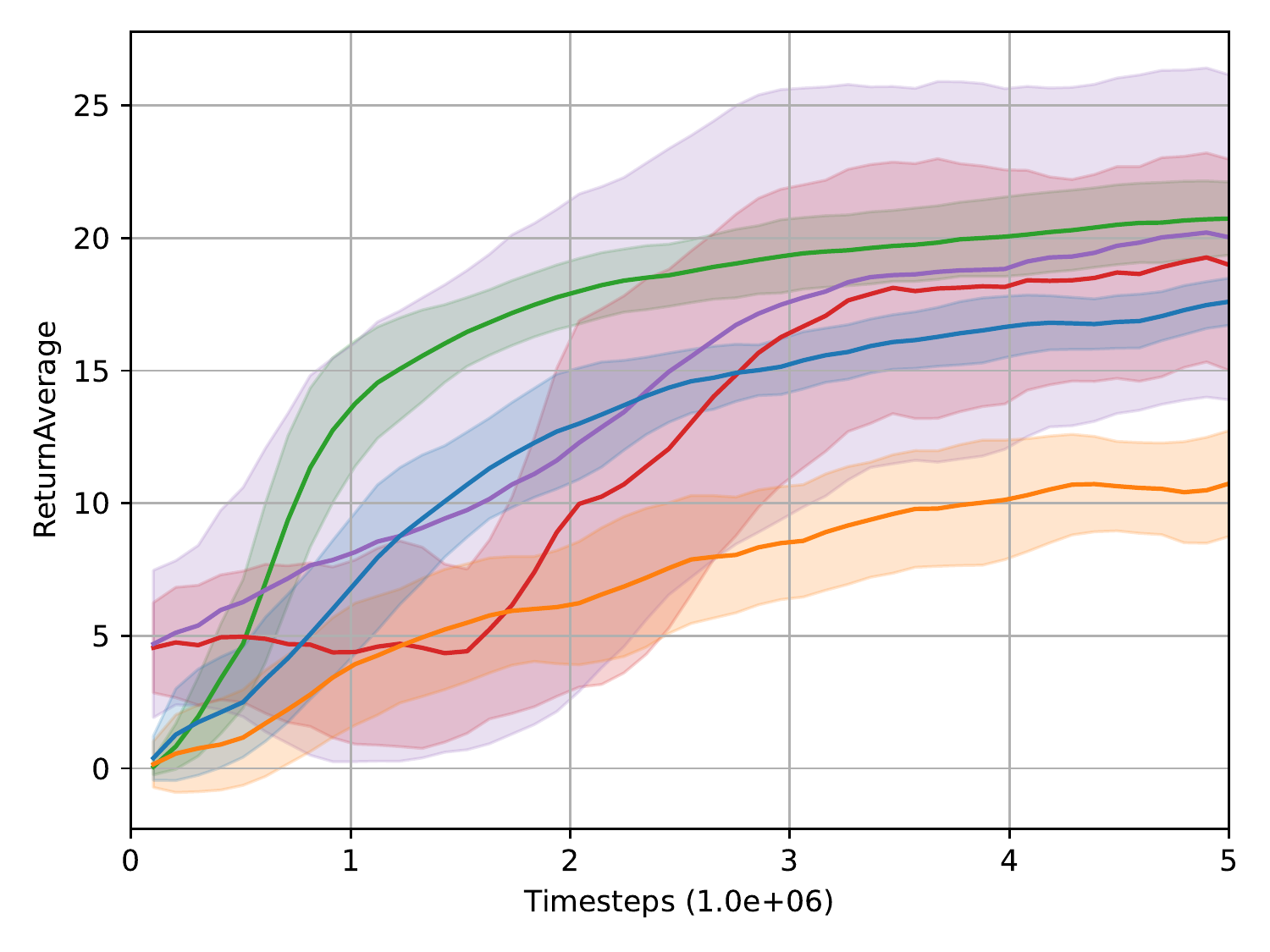}
        \caption{Average Return}
        \label{fig:DynamicEnv_Return}
    \end{subfigure}
    \begin{subfigure}[b]{0.32\columnwidth}
        \centering
        \includegraphics[width=\textwidth]{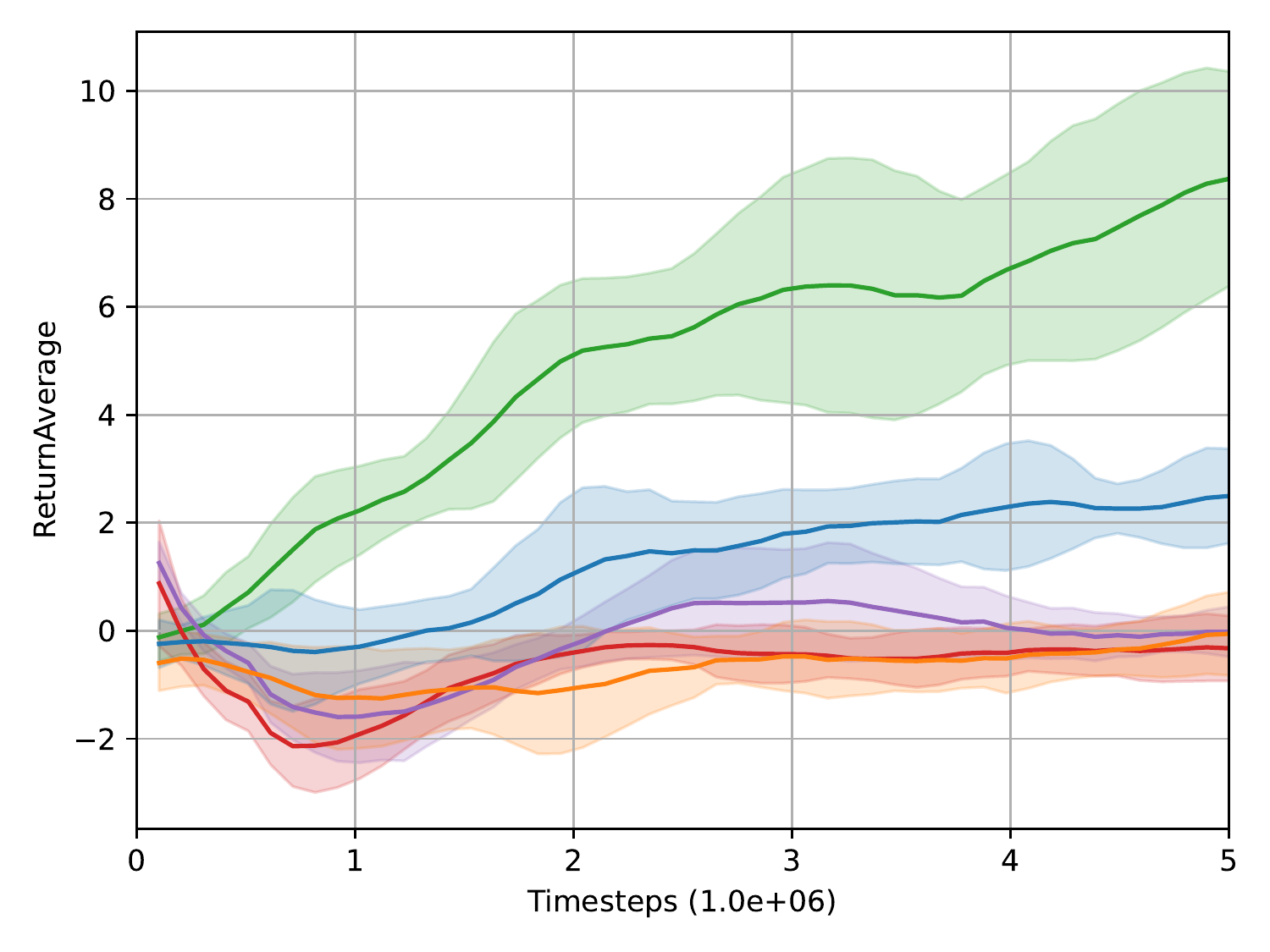}
        \caption{Average Return}
        \label{fig:GremlinEnv_Return}
    \end{subfigure}
    \begin{subfigure}[b]{0.32\columnwidth}
        \centering
        \includegraphics[width=\textwidth]{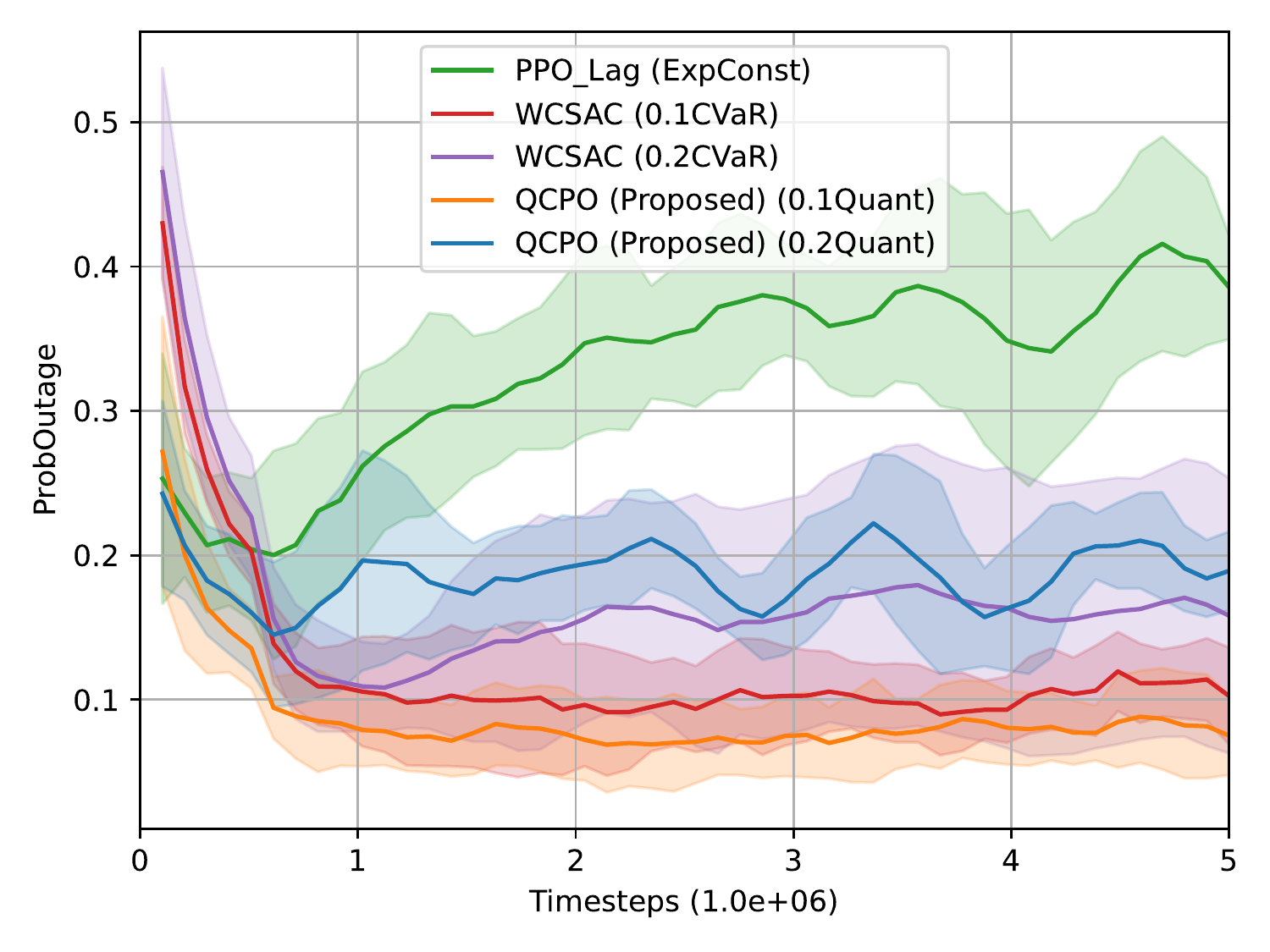}
        \caption{Outage Probability}
        \label{fig:SimpleButtonEnv_ProbOutage}
    \end{subfigure}
    \begin{subfigure}[b]{0.32\columnwidth}
        \centering
        \includegraphics[width=\textwidth]{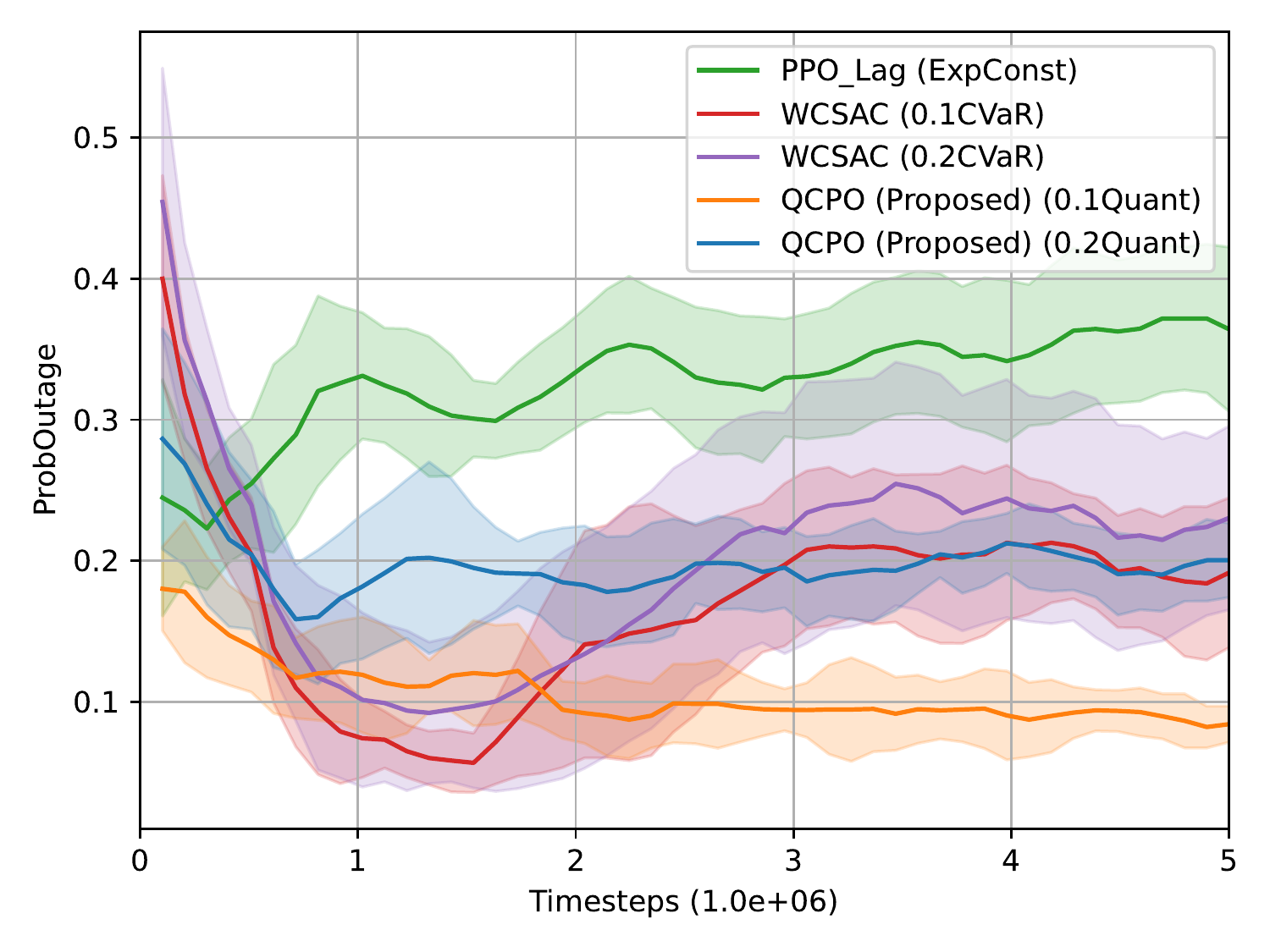}
        \caption{Outage Probability}
        \label{fig:DynamicEnv_ProbOutage}
    \end{subfigure}
    \begin{subfigure}[b]{0.32\columnwidth}
        \centering
        \includegraphics[width=\textwidth]{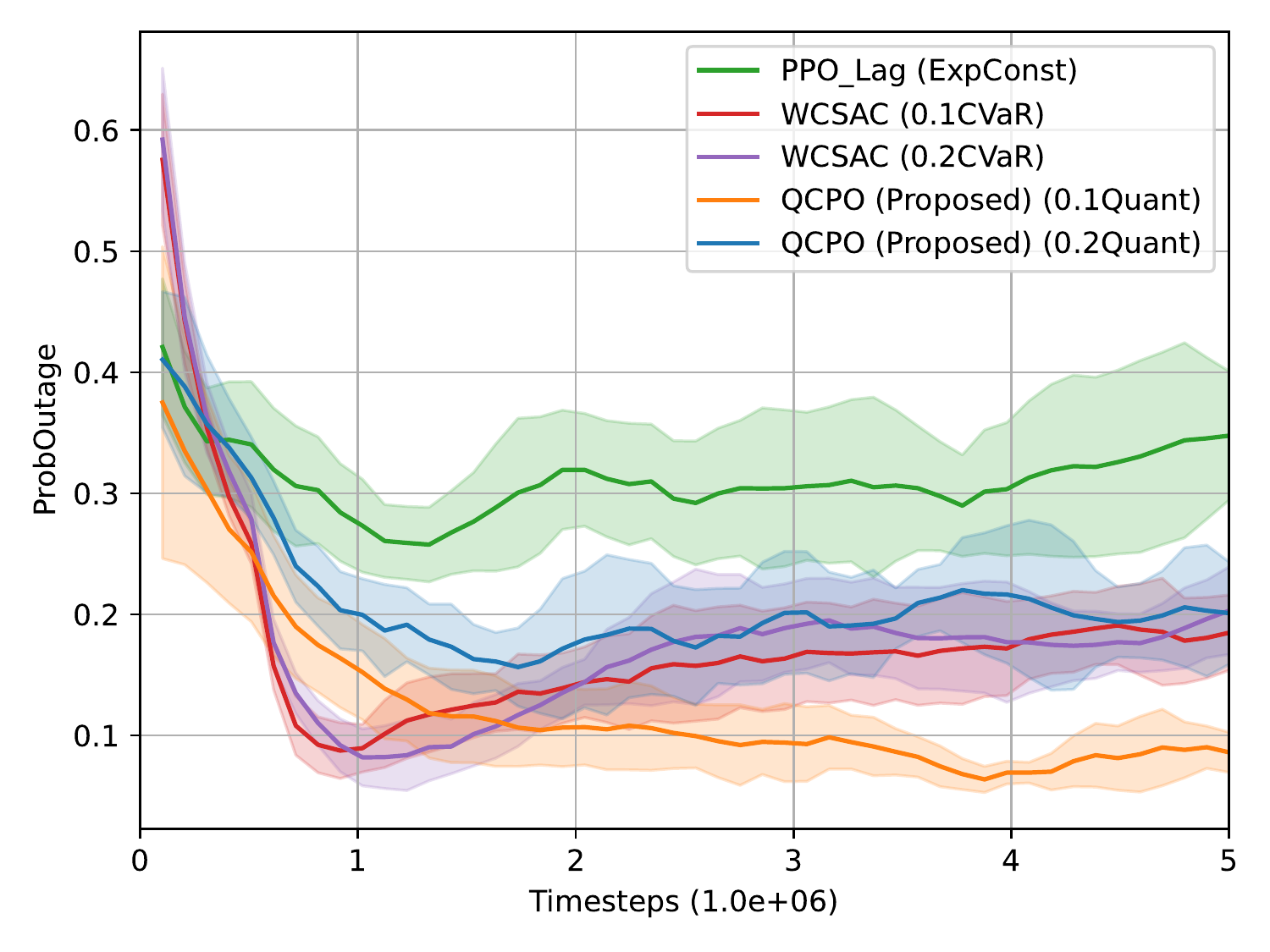}
        \caption{Outage Probability}
        \label{fig:GremlinEnv_ProbOutage}
    \end{subfigure}
    \caption{(left) SimpleButtonEnv, (middle)  DynamicEnv, and (right)  GremlinEnv: (upper row) average return and (lower row) outage probability of the most current 100 episodes.}
    \label{fig:result}
\end{figure}

\section{Conclusion} \label{sec:conclusion}

We have proposed the framework of quantile-constrained RL to constrain the outage probability by adopting a constraint on the quantile, which is equivalent to the outage probability constraint.  We have investigated  issues in applying the policy gradient theorem to the Lagrangian  of the quantile-constrained RL problem and  have converted the quantile into an additive form of costs so that the application of the policy gradient theorem is feasible.  Based on our derivation, we have constructed the QCPO algorithm, which uses distributional RL techniques to learn the $u$-quantile of the cumulative sum cost $X^\pi(s)$, and  Weibull distribution to approximate the tail distribution of $X^\pi(s)$. We also proved the policy improvement condition for QCPO and showed that there exists an approximation loss due to our approximation of the quantile. Empirical results show that QCPO constrains the outage probability well as the desired target value. The meaning of such exact satisfaction of the outage probability is two-fold: First, the constraint on the  outage probability is satisfied to control the probability of unsafe events, and second, the exact satisfaction of the cost constraint enables us to exploit the cost budget fully and obtain a higher return. Empirical results demonstrated the effectiveness of the proposed scheme.

\acksection

This work was supported by Institute of Information \& Communications Technology Planning \& Evaluation (IITP) grant funded by the Korea government (MSIT) (No.2022-0-00469, Development of Core Technologies for Task-oriented Reinforcement Learning for Commercialization of Autonomous Drones, 50\%) and by Institute of Information \& Communications Technology Planning \& Evaluation (IITP) grant funded by the Korea government (MSIT) (No.2022-0-00124, Development of Artificial Intelligence Technology for Self-Improving Competency-Aware Learning Capabilities, 50\%)

\newpage
\bibliography{qcpo_final}
\bibliographystyle{plainnat}

\section*{Checklist}

\begin{enumerate}

\item For all authors...
\begin{enumerate}
  \item Do the main claims made in the abstract and introduction accurately reflect the paper's contributions and scope?
    \answerYes{}
  \item Did you describe the limitations of your work?
    \answerYes{We mentioned theoretical bound and approximation loss in Section \ref{sec:QCRL} and \ref{sec:QCPO}.}
  \item Did you discuss any potential negative societal impacts of your work?
    \answerNA{}
  \item Have you read the ethics review guidelines and ensured that your paper conforms to them?
    \answerYes{I have read the guidelines.}
\end{enumerate}

\item If you are including theoretical results...
\begin{enumerate}
  \item Did you state the full set of assumptions of all theoretical results?
    \answerYes{Please see Appendix \ref{appendix:proof}.}
  \item Did you include complete proofs of all theoretical results?
    \answerYes{Please see Appendix \ref{appendix:proof}.}
\end{enumerate}

\item If you ran experiments...
\begin{enumerate}
  \item Did you include the code, data, and instructions needed to reproduce the main experimental results (either in the supplemental material or as a URL)?
    \answerYes{The implementation code and instructions are uploaded in github. Please see github address in page 7.}
  \item Did you specify all the training details (e.g., data splits, hyperparameters, how they were chosen)?
    \answerYes{Please see Appendix \ref{appendix:implementation_details}.}
  \item Did you report error bars (e.g., with respect to the random seed after running experiments multiple times)?
    \answerYes{Experimental results show mean and standard deviation with 10 random seeds.}
  \item Did you include the total amount of compute and the type of resources used (e.g., type of GPUs, internal cluster, or cloud provider)?
    \answerYes{Please see Appendix \ref{appendix:environments}.}
\end{enumerate}

\item If you are using existing assets (e.g., code, data, models) or curating/releasing new assets...
\begin{enumerate}
  \item If your work uses existing assets, did you cite the creators?
    \answerYes{Please see footnotes in page 9.}
  \item Did you mention the license of the assets?
    \answerYes{Please see footnotes in page 9.}
  \item Did you include any new assets either in the supplemental material or as a URL?
    \answerYes{Please see github address in page 7.}
  \item Did you discuss whether and how consent was obtained from people whose data you're using/curating?
    \answerNA{}
  \item Did you discuss whether the data you are using/curating contains personally identifiable information or offensive content?
    \answerNA{}
\end{enumerate}

\item If you used crowdsourcing or conducted research with human subjects...
\begin{enumerate}
  \item Did you include the full text of instructions given to participants and screenshots, if applicable?
    \answerNA{}
  \item Did you describe any potential participant risks, with links to Institutional Review Board (IRB) approvals, if applicable?
    \answerNA{}
  \item Did you include the estimated hourly wage paid to participants and the total amount spent on participant compensation?
    \answerNA{}
\end{enumerate}

\end{enumerate}

\newpage
\appendix

\section{More Backgrounds} \label{appendix:more_backgounds}

\subsection{Distributional RL} \label{appendix:distributional_rl}

Distributional RL \cite{barth-maron2018distributional, bellemare2017distributional, dabney2018implicit} is an area of RL that considers the distribution of the cumulative return $Z^\pi(s, a) = \sum^\infty_{t=0} \gamma^t r(S_t, A_t)$ for $S_0 = s$, $A_0 = a$, $S_{t+1} \sim M(\cdot | S_t, A_t)$, $A_{t+1} \sim \pi(\cdot | S_t)$, $t=0, 1, \cdots$, instead of the expectation of the cumulative return $Q^\pi(s, a) = \Ebb_{\pi}[Z^\pi(s,a)] = \Ebb_{\pi}\left[ \sum^\infty_{t=0} \gamma^t r(s_t, a_t) \right]$ to optimize a policy $\pi$. In distributional RL, the distribution of the cumulative return $Z^\pi(s, a)$ is computed by the distributional Bellman equation \cite{bellemare2017distributional}, defined as
\begin{equation}
    Z^\pi(s,a) \overset{D}{=} r(s,a) + Z^\pi(S', A')
\end{equation}
for $S' \sim M(\cdot | s, a), A' \sim \pi(\cdot | S')$, where  $\overset{D}{=}$ means that the random variable in the left-hand side (LHS) has the same distribution to that in the right-hand side (RHS). So,  the following holds\cite{ma2021conservative}:
\begin{align}
    F_{Z^\pi(s,a)}(z) &= \Ebb_{s' \sim M, a' \sim \pi} \left[ F_{r(s,a) + \gamma Z^\pi(s',a')}\left( z \right) \right] \nonumber \\
        &= \Ebb_{s' \sim M, a' \sim \pi} \left[ F_{Z^\pi(s',a')}\left( \frac{z - r(s,a)}{\gamma} \right) \right] \label{appendix:eq:relation_cdf}\\
    p_{Z^\pi(s,a)}(z) &= \frac{1}{\gamma} \Ebb_{s' \sim M, a' \sim \pi} \left[ p_{Z^\pi(s',a')}\left( \frac{z - r(s,a)}{\gamma} \right) \right], \label{appendix:eq:relation_pdf}
\end{align}
where $F_X(x)$ and $P_X(x)$ denote the cumulative distribution function (CDF) and PDF of a random variable $X$, respectively, and \eqref{appendix:eq:relation_pdf} is obtained by taking derivative of  \eqref{appendix:eq:relation_cdf}.  To train the distribution of the cumulative return $Z^\pi(s,a)$, the $p$-Wasserstein distance $W_p(X, Y)$ is typically used, which can be  written explicitly as 
\begin{equation}
    W_p(X, Y) = \left( \int^1_0 \left\vert F^{-1}_X(u) - F^{-1}_Y(u) \right\vert^p \right)^{1/p}
\end{equation}
for $p < \infty$, where  $F^{-1}_X(u) = \inf \left\{ x ~|~ F_X(x) \geq u \right\}=:Q_X(u)  $ is the quantile function (inverse CDF) of the random variable $X$.
\citet{dabney2018implicit, dabney2018distributional, mavrin19distributional, kuznetsov2020controlling, yang2019fully} used quantile regression to learn the quantile of the cumulative return $Z^\pi(s,a)$. The quantile regression loss is given by $L_{quant, u}(q) = \Ebb_X \left[ l_{quant, u}(X - q) \right]$,  where 
\begin{equation}
    l_{quant, u}(x) = \left( u - 1_{\{x < 0\}} \right) \cdot x. 
\end{equation}
To smooth the gradient, they used the quantile Huber loss function $L_{Huber, u}(q) = \Ebb_X \left[ l_{Huber, u}(X-q) \right]$ for a given $\kappa > 0$, where
\begin{align}
    l_{Huber, u}(x) &= \left\vert u - 1_{\{x < 0\}} \right\vert ~ \frac{L_{\kappa}(x)}{\kappa}, \\
    L_{\kappa}(x) &= \left\{ \begin{array}{ll} \frac{1}{2} x^2 , \quad & \text{if } \vert x \vert \leq \kappa \\ \kappa \left( \left\vert x \right\vert - \frac{1}{2} \kappa \right) , \quad & \text{otherwise.}\end{array} \right. \nonumber
\end{align}
In this paper, we estimate the quantiles of the cumulative sum cost using the quantile loss, and use them to solve the constrained optimization problem \eqref{problem:quantile}. 

\begin{equation}\tag{QuantCP}
    \begin{array}{ll} \text{Maximize} \quad &\Ebb_{\pi}\left[ \sum^\infty_{t=0} \gamma^t r(s_t, a_t) \right] \\ \text{Subject to} \quad &q^\pi_{1-\epsilon_0}(s_0) \leq d_{th},  \end{array} 
\end{equation}

\subsection{Large Deviation Principle (LDP)} \label{appendix:LDP}

Large deviation principle (LDP) \cite{dembo1998large} is a technique for estimating the limiting behavior of a sequence of distributions. A simple example is the empirical mean $\bar{X}_n = \frac{1}{n} \sum^n_{k=1} X_k$ of i.i.d. random variables $X_i$. We say that a sequence $\{\bar{X}_n\}$ satisfies LDP if the sequence of its log probability distribution $\frac{1}{n} \log{\Pr\left( \bar{X}_n \in \Gamma \right)}$ satisfies the following condition $\frac{1}{n} \log{\Pr\left( \bar{X}_n \in \Gamma \right)} \overset{n \rightarrow \infty}{\longrightarrow} - \inf_{x \in \Gamma } I(x)$ for some function $I(x)$. The function $I(x)$ satisfying such limiting behavior is called the rate function of $\bar{X}_n$. The rate function $I(x)$ is also related to the cumulative distribution function $F_{\bar{X}_n}(x)$ since $1 - F_{\bar{X}_n}(x_0) = \Pr\left(\bar{X}_n \in [x_0, \infty) \right) \approx \exp{\left(-n \inf_{x \in [x_0, \infty)} I(x) \right)}$ for some $x_0 > \Ebb[X]$ and sufficiently large $n$. 

LDP can be applied to finite state Markov chains \cite{dembo1998large}. Let $Y_k \in \mcal{Y} = \left\{ y^1, \ldots y^m \right\}$ be random variables that follows the Markov property: $\Pr(Y_1=y_1, \ldots, Y_n=y_n) = p_0(y_1) \prod^{n}_{i=1} M(y_{i+1} | y_i)$. Then, the sequence of empirical means $Z_n := \frac{1}{n} \sum^n_{k=0} X_k$, where $X_k = f(Y_k)$ for some function $f: \mcal{Y} \rightarrow \Rbb^d$, satisfies LDP and the rate function  is given by $I(z) = \sup_{\lambda \in \Rbb^d} \left\{ \langle \lambda, z \rangle - \log{\rho(\Pi_\lambda)} \right\}$, where $\rho(\Pi)$ is the Perron-Frobenius eigenvalue of a given matrix $\Pi$, and $\Pi_\lambda$ is the matrix whose $(i,j)$-th element is $M(y^j | y^i)\exp{\langle \lambda, f(y^j) \rangle}$. 

In this paper, we consider the tail probability of the distribution of the cumulative sum cost $X^\pi(s_0) = \sum^\infty_{t=0} \gamma^t c(s_t, a_t)$. Finding its analytic rate function is hard. Therefore, we instead approximate the rate function directly as $I_{X^\pi(s)}(x) \approx (x / \beta(s))^{\alpha(s)}$ with learnable parameters $\alpha(s)$ and $\beta(s)$, which results in a Weibull distribution: $1 - F_{X^\pi(s)}(x) = \exp{\left\{ -(x / \beta(s))^{\alpha(s)} \right\}}$. We use this distribution to approximate the tail probability of $p_{X^\pi(s)}(x)$ of $X^\pi(s)$.

\subsection{The Considered  Constrained Problems} \label{appendix:constrained_problem}

In this subsection, we list the  problems for constrained RL. The first constrained problem is a common problem used in many previous constrained RL papers. 
\begin{equation}\tag{ExpCP}
    \begin{array}{ll} \text{Maximize} \quad &V^\pi(s_0) := \Ebb_{\pi}\left[ \sum^\infty_{t=0} \gamma^t r(s_t, a_t) \right] \\ \text{subject to} \quad &C^\pi(s_0) := \Ebb_{\pi}\left[ \sum^\infty_{t=0} \gamma^t c(s_t, a_t) \right] \leq d_{th}, \end{array} 
    \label{appendix:problem:expectation}
\end{equation}
In \eqref{appendix:problem:expectation}, the cost constraint is that the expectation of the sum of costs is less than or equal to a threshold parameter $d_{th}$. Note that the threshold $d_{th}$ is set on the average (i.e., expectation) of the cumulative sum cost to avoid undesired high-cost events  in this formulation. However, solving the problem \eqref{appendix:problem:expectation} may have undesirable outcomes for real environments that typically need  constrained behavior on the event that the cost exceeds the threshold $d_{th}$.

There are two well-known techniques, called Value at Risk (VaR, or Quantile) and Conditional Value at Risk (CVaR), to manage undesirable events in the domain of finance\cite{rockafellar2002conditional}. In the context of RL, the definitions of the quantile and the CVaR for the distribution of the cumulative sum cost for a given $\pi$ are given by $q^\pi_u(s_0) := \inf \{ x ~|~ \Pr(X^\pi(s_0) \leq x) \geq u \}$ and $\cvar^\pi_u(s_0) := \Ebb_{\pi}\left[ X^\pi(s_0) ~\vert~ X^\pi(s_0) \geq q^\pi_u(s_0) \right]$, respectively. Note that the CVaR and the quantile are two  different measures for undesirable events, and the choice between the two depends on what  we desire. For example, an insurance company prefers  the CVaR of undesirable events to determine an insurance premium. On the other hand, a company developing an autonomous driving car system needs the quantile of undesirable events to guarantee the accident probability for safety.

\begin{figure}[ht!]
    \centering
    \includegraphics[width=0.6\textwidth]{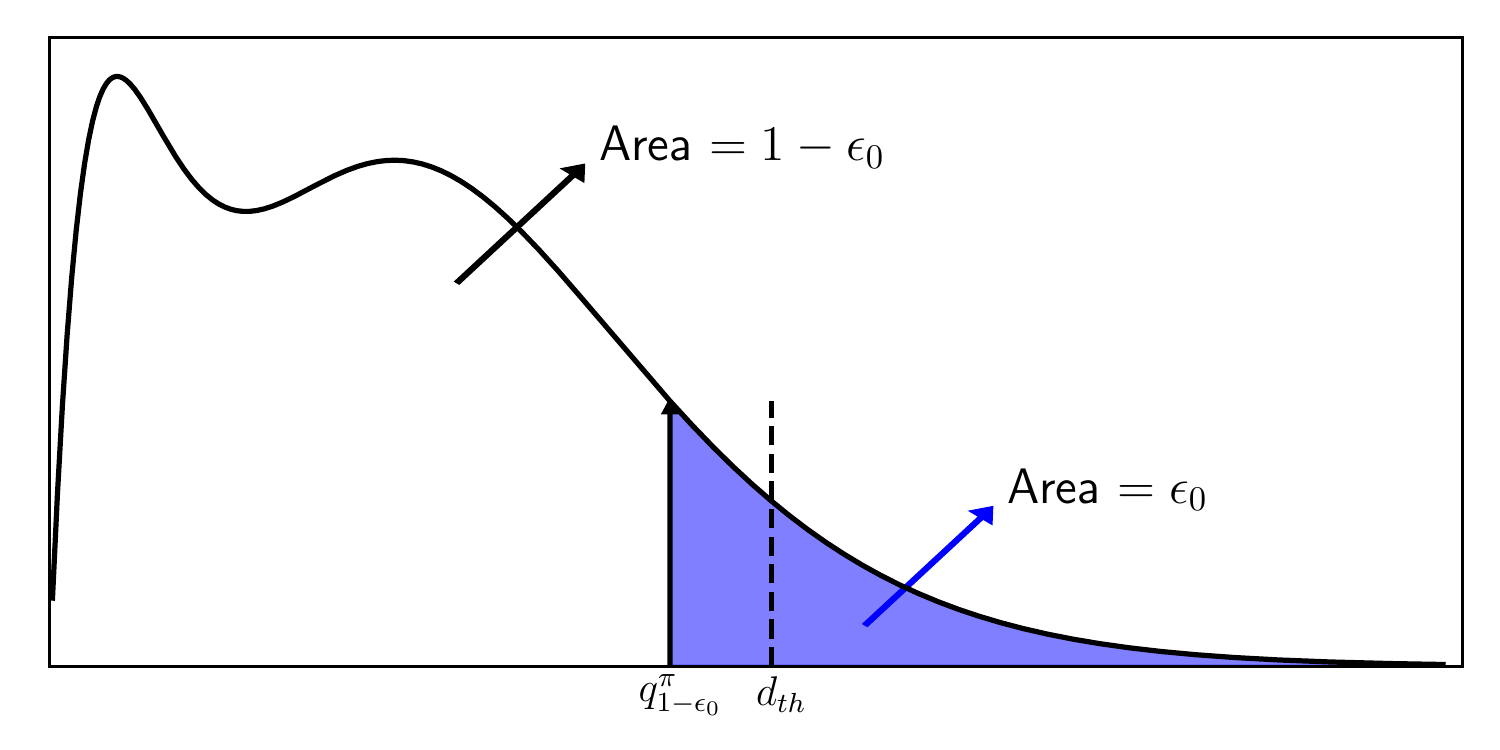}
    \caption{Equivalence between the outage probability constraint and the quantile constraint}
    \label{appendix:fig:relation_between_probcp_quantcp}
\end{figure}

The CVaR constrained problem to constrain undesirable events was previously used in RL \cite{chow2017risk, yang2021wcsac}, and the problem is explicitly formulated as  
\begin{equation}\tag{CVaR-CP}
    \begin{array}{ll} \text{Maximize} \quad &\Ebb_{\pi}\left[ \sum^\infty_{t=0} \gamma^t r(s_t, a_t) \right] \\ \text{Subject to} \quad &\cvar^\pi_{1-\epsilon_0}(s_0) \leq d_{th},  \end{array}
    \label{appendix:problem:cvar}
\end{equation}

In this paper, we focus on constraining the probability of undesirable events that the cost exceeds the threshold $d_{th}$. Thus we can consider a constrained problem with a probabilistic constraint as follows:
\begin{equation}\tag{ProbCP}
    \begin{array}{ll} \text{Maximize} \quad &V^\pi(s_0) = \Ebb_{\pi}\left[ \sum^\infty_{t=0} \gamma^t r(s_t, a_t) \right] \\ \text{Subject to} \quad &\Pr\left[ \sum^\infty_{t=0} \gamma^t c(S_t, A_t) > d_{th} \right] \leq \epsilon_0 \\
    &~~\text{for } S_0 = s_0, A_t \sim \pi(\cdot | S_t), S_{t+1} \sim M(\cdot | S_t, A_t). \end{array} 
    \label{appendix:problem:probabilistic}
\end{equation}
Our approach to this problem is  first to convert the outage probability constraint in \eqref{appendix:problem:probabilistic} into a quantile constraint $q^\pi_{1-\epsilon_0}(s_0) \leq d_{th}$, which is equivalent to the original probabilistic constraint (See Fig. \ref{appendix:fig:relation_between_probcp_quantcp}), and then to solve the equivalent optimization: 
\begin{equation}\tag{QuantCP}
    \begin{array}{ll} \text{Maximize} \quad &\Ebb_{\pi}\left[ \sum^\infty_{t=0} \gamma^t r(s_t, a_t) \right] \\ \text{Subject to} \quad &q^\pi_{1-\epsilon_0}(s_0) \leq d_{th},  \end{array} \label{appendix:problem:quantile}
\end{equation}

Note that the $(1-\epsilon_0)$-quantile denoted as $q^\pi_{1-\epsilon_0}(s)$ is always less or equal to than the $(1-\epsilon_0)$-CVaR denoted as $\cvar^\pi_{1-\epsilon_0}(s)$ for all $s \in \mcal{S}$ because of the definition of the CVaR. Therefore, satisfying the CVaR constraint is a sufficient condition for satisfying the probabilistic constraint, and hence  this problem is a  stricter problem than \eqref{appendix:problem:probabilistic} or \eqref{appendix:problem:quantile}. Therefore, the algorithms proposed to solve \eqref{appendix:problem:cvar} can be used for solving \eqref{appendix:problem:probabilistic}, and this should satisfy the probabilistic constraint in theory.

\newpage
\section{Proofs} \label{appendix:proof}

\setcounter{mythm}{0}
\setcounter{mycor}{0}
\setcounter{mylem}{0}
\setcounter{myassm}{0}

In the following proofs, we used text color so that readers can follow the proof easily.

\subsection{Proof of Theorem \ref{appendix:thm:td_relation_quantile_bound}}

\begin{myassm}[Boundness of quantile difference] \label{appendix:assm:boundness}
    For a given policy $\pi$, the following two quantities are bounded
    \begin{align}
        \left\vert c(s,a) + \gamma q^\pi_u(s') - q^\pi_u(s) \right\vert &\leq \gamma R \\
        \left\vert q^\pi_u(s) - F^{-1}_{X^\pi(s)}\left( \Fqtarg \right) \right\vert &\leq R
    \end{align}
    for all $(s, a, s') \in \Scal \times \Acal \times \Scal$ such that $\pi(a | s) \cdot M(s' | s, a) > 0$.
\end{myassm}
Note that for finite MDPs, which are assumed for many RL proofs, this assumption definitely holds with a finite cost function.

\begin{myassm}[Smoothness of CDF of $X^\pi(s)$] \label{appendix:assm:smoothness}
    For each state $s$, the average slope of $F_{X^\pi(s)}(x)$ between $q^\pi_u(s)$ and $y \in [q^\pi_u(s)-R, q^\pi_u(s)+R]$ is bounded by
    \begin{equation}
        \frac{1}{1 + \epsilon} \cdot p_{X^\pi(s)} \left( q^\pi_u(s) \right) \leq \frac{F_{X^\pi(s)}\left( q^\pi_u(s) \right) - F_{X^\pi(s)}\left( y \right) }{ q^\pi_u(s) - y } \leq \frac{1}{1 - \epsilon} \cdot p_{X^\pi(s)} \left( q^\pi_u(s) \right)
    \end{equation}
    for small $0 < \epsilon < \frac{1}{2}$.
\end{myassm}
This assumption holds when discrete masses are not present in the PDF and the CDF is continuous.

\begin{mythm} \label{appendix:thm:td_relation_quantile_bound}
    Under Assumptions \ref{assm:boundness} and \ref{assm:smoothness}, the $u$-quantile of the random variable $X^\pi(s_t)$ satisfies the following temporal-difference(TD) relation.  For some constant $R$ and small $\epsilon > 0$,
    \begin{equation}
        \left\vert \Ebb_{\pi} \left[ \frac{ p_{X^\pi(s_{t+1})}\left( \frac{q^\pi_u(s_t) - c(s_t, a_t)}{\gamma} \right)}{\gamma p_{X^\pi(s_t)}\left( q^\pi_u(s_t) \right)} \left\{ c(s_t, a_t) + \gamma q^\pi_u(s_{t+1}) - q^\pi_u(s_t)  \right\} \right] \right\vert \leq \frac{\epsilon}{1-\epsilon} R, \label{appendix:eq:td_relation_quantile_bound}
    \end{equation}
    Here, the expectation is for the action $a_t \sim \pi(\cdot | s_t)$ and the next state $s_{t+1} \sim M(\cdot | s_t, a_t)$. ($s_t$ is given.)
\end{mythm}

\begin{proof}
	Note that from \eqref{appendix:eq:relation_cdf}, 
	\begin{equation}
		F_{X^\pi(s)}\left( x  \right) = \Ebb_{\pi}\left[ F_{X^\pi(s')} \left( \frac{x - c(s,a)}{\gamma} \right) \right],
	\end{equation}
	for all $x$. If $x = q^\pi_u(s)$, then this becomes
	\begin{equation}
		u = F_{X^\pi(s)}\left( q^\pi_u(s)  \right) = \Ebb_{\pi}\left[ \Fqtarg \right]. \label{appendix:eq:relation_cdf_quantile}
	\end{equation}
	Using \eqref{appendix:eq:relation_cdf_quantile}, we can obtain
	\begingroup
	\allowdisplaybreaks
	\begin{align}
		&\Ebb_{\pi} \left[ \frac{ p_{X^\pi(s_{t+1})}\left( \frac{q^\pi_u(s_t) - c(s_t, a_t)}{\gamma} \right)}{\gamma p_{X^\pi(s_t)}\left( q^\pi_u(s_t) \right)} \left\{ c(s_t, a_t) + \gamma q^\pi_u(s_{t+1}) - q^\pi_u(s_t)  \right\} \right] \\
		&= \frac{ \Ebb_{\pi} \left[ \pqtargt \left\{ c(s_t, a_t) + \gamma q^\pi_u(s_{t+1}) - q^\pi_u(s_t)  \right\} \right] }{\gamma p_{X^\pi(s_t)}\left( q^\pi_u(s_t) \right)} \\
		&= \frac{1}{\gamma p_{X^\pi(s_t)}\left( q^\pi_u(s_t) \right)} \times \Ebb_{\pi} \Biggl[ \tcb{ \gamma  \underbrace{\left( \Fqtargt - u \right)}_{= 0 \text{ by }\eqref{appendix:eq:relation_cdf_quantile}} } \nonumber \\
		& ~~~~~~~~~~~~~+ \pqtargt \left\{ c(s_t, a_t) + \gamma q^\pi_u(s_{t+1}) - q^\pi_u(s_t)  \right\} \Biggr] \\
		&= \Ebb_{\pi} \left[ \frac{ \tcb{ \left( \Fqtargt - u  \right)} + \pqtargt \left\{ q^\pi_u(s_{t+1}) - \qtargt  \right\} }{p_{X^\pi(s_t)}\left( q^\pi_u(s_t) \right)}  \right] \\
		&= \Ebb_{\pi} \Biggl[   \frac{ \tcb{\left( u - \Fqtargt  \right)}}{p_{X^\pi(s_t)}\left( q^\pi_u(s_t) \right)}    \nonumber \\
		&~~~~~~~ \times  \frac{  \tcb{\left( \Fqtargt - u  \right)} + \pqtargt \left\{ q^\pi_u(s_{t+1}) - \qtargt  \right\} }{ \tcb{\left( u - \Fqtargt  \right)} }  \Biggr] \\
		&= \Ebb_{\pi} \Biggl[   \frac{ \tcb{\left( u - \Fqtargt  \right)}}{p_{X^\pi(s_t)}\left( q^\pi_u(s_t) \right)} \nonumber \\
		&~~~~~~~ \times \left\{ \frac{ \pqtargt \left\{ q^\pi_u(s_{t+1}) - \qtargt  \right\} }{ \tcb{\left( u - \Fqtargt  \right)} } - 1 \right\} \Biggr]
	\end{align}
	\endgroup
	Then, by Cauchy-Schwarz inequality, we can obtain a bound such that
	\begingroup
	\allowdisplaybreaks
	\begin{align}
		&\Ebb_{\pi} \left[ \frac{ p_{X^\pi(s_{t+1})}\left( \frac{q^\pi_u(s_t) - c(s_t, a_t)}{\gamma} \right)}{\gamma p_{X^\pi(s_t)}\left( q^\pi_u(s_t) \right)} \left\{ c(s_t, a_t) + \gamma q^\pi_u(s_{t+1}) - q^\pi_u(s_t)  \right\} \right]^2\\
		&= \Ebb_{\pi} \Biggl[   \frac{  \tcb{\left( u - \Fqtargt  \right)} }{p_{X^\pi(s_t)}\left( q^\pi_u(s_t) \right)}  \nonumber \\
		& ~~~~~\times \left\{ \frac{ \pqtargt \left\{ q^\pi_u(s_{t+1}) - \qtargt  \right\} }{ \tcb{\left( u - \Fqtargt  \right)} } - 1 \right\} \Biggr]^2 \\
		&\leq \underbrace{ \Ebb_{\pi} \left[ \left( \frac{ \tcb{ u - \Fqtargt } }{p_{X^\pi(s_t)}\left( q^\pi_u(s_t) \right)} \right)^2  \right] }_{(a)} \nonumber \\
		& ~~~~~\times \underbrace{ \Ebb_{\pi} \left[ \left( 1 - \frac{ p_{X^\pi(s_{t+1})}\left( \frac{q^\pi_u(s_t) - c(s_t, a_t)}{\gamma} \right) \cdot \left\{ q^\pi_u(s_{t+1}) - \qtargt  \right\} }{ \tcb{ u - \Fqtargt } } \right)^2 \right] }_{(b)}
	\end{align}
	\endgroup
	
	Now we find upper bounds of (a) and (b).
	\begin{itemize}
		\item First, consider an upper bound of (a).
		\begingroup
		\allowdisplaybreaks
		\begin{align}
			&\frac{ \tcb{ u - \Fqtargt } }{p_{X^\pi(s_t)}\left( q^\pi_u(s_t) \right)} \\
			&= \frac{ u -  \tcr{ F_{X^\pi(s_t)}\left( F^{-1}_{X^\pi(s_t)}\left(  \Fqtargt \right) \right) } }{p_{X^\pi(s_t)}\left( q^\pi_u(s_t) \right)} \\
			&= \frac{ u -  \tcr{ F_{X^\pi(s_t)}\left( \bar{q}^\pi_u(s_t, a_t, s_{t+1}) \right) } }{p_{X^\pi(s_t)}\left( q^\pi_u(s_t) \right)} \\
			&= \frac{ F_{X^\pi(s_t)}\left( q^\pi_u(s_t) \right) -  \tcr{ F_{X^\pi(s_t)}\left( \bar{q}^\pi_u(s_t, a_t, s_{t+1}) \right) } }{p_{X^\pi(s_t)}\left( q^\pi_u(s_t) \right)} \\
			&= \frac{ \frac{ F_{X^\pi(s_t)}\left( q^\pi_u(s_t) \right) -  \tcr{ F_{X^\pi(s_t)}\left( \bar{q}^\pi_u(s_t, a_t, s_{t+1}) \right) } }{ q^\pi_u(s_t) - \bar{q}^\pi_u(s_t, a_t, s_{t+1}) } }{ p_{X^\pi(s_t)}\left( q^\pi_u(s_t) \right) } \cdot \left\{q^\pi_u(s_t) - \bar{q}^\pi_u(s_t, a_t, s_{t+1}) \right\}
		\end{align}
		\endgroup
		where $\bar{q}^\pi_u(s_t, a_t, s_{t+1}) = F^{-1}_{X^\pi(s_t)}\left(  \Fqtargt \right)$. Note that 
		\[
		\frac{ F_{X^\pi(s_t)}\left( q^\pi_u(s_t) \right) -  \tcr{ F_{X^\pi(s_t)}\left( \bar{q}^\pi_u(s_t, a_t, s_{t+1}) \right) } }{ q^\pi_u(s_t) - \bar{q}^\pi_u(s_t, a_t, s_{t+1}) }
		\]
		is the average slope of $F_{X^\pi(s_t)}(x)$ between $q^\pi_u(s_t)$ and $\bar{q}^\pi_u(s_t, a_t, s_{t+1})$, and $p_{X^\pi(s_t)}\left( q^\pi_u(s_t) \right)$ is the slope of $F_{X^\pi(s_t)}(x)$ at $x = q^\pi_u(s_t)$. Therefore by Assumption \ref{appendix:assm:boundness} and \ref{appendix:assm:smoothness}, we can obtain an upper bound of (a) as follows:
		\begingroup
		\allowdisplaybreaks
		\begin{align}
			&\Ebb_{\pi} \left[ \left( \frac{ \tcb{ u - \Fqtargt } }{p_{X^\pi(s_t)}\left( q^\pi_u(s_t) \right)} \right)^2  \right] \\
			&= \Ebb_{\pi} \left[ \left( \frac{ \frac{ F_{X^\pi(s_t)}\left( q^\pi_u(s_t) \right) - \tcr{ F_{X^\pi(s_t)}\left( \bar{q}^\pi_u(s_t, a_t, s_{t+1}) \right) } }{ q^\pi_u(s_t) - \bar{q}^\pi_u(s_t, a_t, s_{t+1}) } }{ p_{X^\pi(s_t)}\left( q^\pi_u(s_t) \right) } \cdot \left\{q^\pi_u(s_t) - \bar{q}^\pi_u(s_t, a_t, s_{t+1}) \right\} \right)^2  \right] \\
			&= \Ebb_{\pi} \Biggl[ \Biggl( \frac{ \frac{ F_{X^\pi(s_t)}\left( q^\pi_u(s_t) \right) - \tcr{ F_{X^\pi(s_t)}\left( \bar{q}^\pi_u(s_t, a_t, s_{t+1}) \right) } }{ q^\pi_u(s_t) - \bar{q}^\pi_u(s_t, a_t, s_{t+1}) } }{ p_{X^\pi(s_t)}\left( q^\pi_u(s_t) \right) } \nonumber \\
			&~~~~~~ \times \left\{q^\pi_u(s_t) - F^{-1}_{X^\pi(s_t)}\left(  \Fqtargt \right) \right\} \Biggr)^2  \Biggr] \\
			&\leq \frac{1}{(1 - \epsilon)^2} R^2
		\end{align}
		\endgroup
		\item Next, we consider an upper bound of (b). By Assumption \ref{appendix:assm:boundness} and \ref{appendix:assm:smoothness},
		\begingroup
		\allowdisplaybreaks
		\begin{align}
			&\Ebb_{\pi} \left[ \left( 1 - \frac{ p_{X^\pi(s_{t+1})}\left( \frac{q^\pi_u(s_t) - c(s_t, a_t)}{\gamma} \right) \cdot \left\{ q^\pi_u(s_{t+1}) - \qtargt  \right\} }{ \tcb{ u - \Fqtargt } } \right)^2 \right] \\
			&=\Ebb_{\pi} \left[ \left( 1 - \frac{ p_{X^\pi(s_{t+1})}\left( \frac{q^\pi_u(s_t) - c(s_t, a_t)}{\gamma} \right) }{ \frac{ \tcb{ u - \Fqtargt } }{ \left\{ q^\pi_u(s_{t+1}) - \qtargt  \right\} } } \right)^2 \right] \\
			&\leq \epsilon^2
		\end{align}
		\endgroup
	\end{itemize}
	Therefore by combining two upper bounds, we can conclude the theorem.
	\begingroup
	\allowdisplaybreaks
	\begin{align}
		&\left\vert \Ebb_{\pi} \left[ \frac{ p_{X^\pi(s_{t+1})}\left( \frac{q^\pi_u(s_t) - c(s_t, a_t)}{\gamma} \right)}{\gamma p_{X^\pi(s_t)}\left( q^\pi_u(s_t) \right)} \left\{ c(s_t, a_t) + \gamma q^\pi_u(s_{t+1}) - q^\pi_u(s_t)  \right\} \right] \right\vert \\
		&\leq \left( \underbrace{ \Ebb_{\pi} \left[ \left( \frac{ \tcb{ u - \Fqtargt } }{p_{X^\pi(s_t)}\left( q^\pi_u(s_t) \right)} \right)^2  \right] }_{(a)} \right)^{\frac{1}{2}} \\
		&~~~~\times \left( \underbrace{ \Ebb_{\pi} \left[ \left( 1 - \frac{ p_{X^\pi(s_{t+1})}\left( \frac{q^\pi_u(s_t) - c(s_t, a_t)}{\gamma} \right) \cdot \left\{ q^\pi_u(s_{t+1}) - \qtargt  \right\} }{ \tcb{ u - \Fqtargt } } \right)^2 \right] }_{(b)} \right)^{\frac{1}{2}} \\
		&\leq \frac{\epsilon}{1-\epsilon} R
	\end{align}
	\endgroup
\end{proof}

\begin{mycor} \label{appendix:cor:td_relation_quantile_bound}
    Under Assumptions \ref{appendix:assm:boundness} and \ref{appendix:assm:smoothness}, the $u$-quantile $q^\pi_u(s_t)$ of the random variable $X^\pi(s_t)$ is bounded as
    \begin{equation} \label{appendix:eq:lem1main}
        \biggl\vert q^\pi_u(s_t) - \Ebb_{\pi} \biggl[ \mu^{\pi}_u \left( s_t, a_t, s_{t+1} \right) \bigl\{ c(s_t, a_t) + \gamma q^\pi_u(s_{t+1}) \bigr\} \biggr] \biggr\vert \leq \frac{\epsilon}{1-\epsilon} R. 
    \end{equation}
    where 
    \begin{equation}
        \mu^{\pi}_u \left( s_t, a_t, s_{t+1} \right) := \frac{ p_{X^\pi(s_{t+1})}\left( \frac{q^\pi_u(s_t) - c(s_t, a_t)}{\gamma} \right)}{\gamma p_{X^\pi(s_t)}\left( q^\pi_u(s_t) \right)} \label{appendix:eq:def_mu}
    \end{equation}
\end{mycor}
\begin{proof}
    Note that the term $q_u^\pi(s_t)$ can go outside  the expectation in \eqref{appendix:eq:td_relation_quantile_bound} since the expectation is over $(a_t,s_{t+1})$. From eq. \eqref{appendix:eq:relation_pdf} in Appendix \ref{appendix:distributional_rl}, the expectation of the numerator of $\mu^{\pi}_u \left( s_t, a_t, s_{t+1} \right)$ is the same as the the denominator of the weight, i.e.,  $\Ebb_{\pi}\left[ p_{X^\pi(s_{t+1})}\left( \frac{q^\pi_u(s_t) - c(s_t, a_t)}{\gamma} \right) \right] = \gamma p_{X^\pi(s_t)}\left( q^\pi_u(s_t) \right)$ and this leads to $\mathbb{E}_\pi [\mu_u^\pi(s_t,a_t,s_{t+1})] = 1$. So, we have the claim.
\end{proof}

\subsection{Proof of Lemma \ref{appendix:lem:expectation_form_quantile}}

\begin{mylem} \label{appendix:lem:expectation_form_quantile}
    Suppose that the state transition dynamics are deterministic, i.e., $s_{t+1} = h(s_t, a_t)$. Then, under Assumptions \ref{appendix:assm:boundness} and \ref{appendix:assm:smoothness}, the $u$-quantile $q^\pi_u(s_0)$ of the random variable $X^\pi(s_0)$ is expressed as
    \begin{equation} \label{appendix:eq:lemma2main}
         \left\vert q^\pi_u(s_0) - \Ebb_{\tilde{\pi}_u} \left[ \sum^\infty_{t=0} \gamma^t c(s_t, a_t) \right] \right\vert \leq \frac{\epsilon R}{(1-\epsilon)(1-\gamma)},
    \end{equation}
    where
    \begin{align}
        \tilde{\pi}_u(a | s) &= \pi(a | s) \cdot \frac{ p_{X^\pi(h(s, a))}\left( \frac{q^\pi_u(s) - c(s, a)}{\gamma} \right)}{\gamma p_{X^\pi(s)}\left( q^\pi_u(s) \right)}  \propto \pi(a | s) \cdot p_{X^\pi(h(s, a))}\left( \frac{q^\pi_u(s) - c(s, a)}{\gamma} \right). \label{appendix:eq:tpipimu}
    \end{align}
    
\end{mylem}

\begin{proof}
Remind that $\mu^\pi_u \left( s, a, s' \right)$ is defined in \eqref{appendix:eq:def_mu} as
\[
    \mu^\pi_u \left( s, a, s' \right) := \frac{ p_{X^\pi(s')}\left( \frac{q^\pi_u(s) - c(s, a)}{\gamma} \right)}{\gamma p_{X^\pi(s)}\left( q^\pi_u(s) \right)}
\]
Consider $\Ebb_{\pi} \left[\mu^\pi_u \left( s, a, s' \right) \left\{ c(s, a) + \gamma q^\pi_u(s') \right\} \right]$.
\begingroup
\allowdisplaybreaks
\begin{align}
    &\Ebb_{\pi} \left[ \mu^\pi_u \left( s, a, s' \right) \left\{ c(s, a) + \gamma q^\pi_u(s') \right\} \right] \\
    &= \sum_{a} \sum_{s'} \pi(a | s) \cdot M(s' | s, a) \cdot \mu^\pi_u \left( s, a, s' \right) \left\{ c(s, a) + \gamma q^\pi_u(s') \right\}  \\
    &= \sum_{a} \sum_{s'} \tilde{\pi}_u (a | s) \cdot \tilde{M}_u (s' | s, a) \left\{ c(s, a) + \gamma q^\pi_u(s') \right\} 
\end{align}
\endgroup
for some distorted policy $\tilde{\pi}_u$ and some distorted state transition dynamics $\tilde{M}_u$. This is because
\begin{equation}
    \Ebb_{\pi}\left[ \mu^\pi_u \left( s, a, s' \right) \right] = \sum_a \sum_{s'} \pi(a|s) \cdot M(s' | s, a) \cdot \frac{ p_{X^\pi(s')}\left( \frac{q^\pi_u(s) - c(s, a)}{\gamma} \right)}{\gamma p_{X^\pi(s)}\left( q^\pi_u(s) \right)} = 1. \label{appendix:eq:expectation_is_1}
\end{equation}
where the last equation holds from \eqref{appendix:eq:relation_pdf}. Now, under the assumption that the state transition dynamics is deterministic $s' = h(s, a)$, i.e., $p(s' | s, a) = \delta_{h(s, a)}(s')$,  the distorted transition dynamics are the same as the original transition dynamics and the only difference is the distorted policy:
\begingroup
\allowdisplaybreaks
\begin{align}
    \tilde{M}_u(s' | s, a) &= \frac{\pi(a|s) M(s'|a, s) \mu^\pi_u \left( s, a, s' \right) }{\sum_{\tilde{s}} \pi(a|s) M(\tilde{s}|a, s) \mu^\pi_u \left( s, a, \tilde{s} \right) } \\
    &= \frac{\delta_{h(s, a)}(s') \pi(a|s) \mu^\pi_u \left( s, a, s' \right) }{\sum_{\tilde{s}} \delta_{h(s, a)}(\tilde{s}) \pi(a|s) \mu^\pi_u \left( s, a, \tilde{s} \right) } \\
    &= \frac{\delta_{h(s, a)}(s') \mu^\pi_u \left( s, a, h(s,a) \right) }{\mu^\pi_u \left( s, a, h(s,a) \right)}\\
    &= \delta_{h(s, a)}(s') \\
    &= M(s'|s, a) \\
    \tilde{\pi}_u(a | s) &= \pi(a | s) \frac{\sum_{s'} M(s' |a, s) \mu^\pi_u \left( s, a, s' \right) }{\sum_{\tilde{a}, s'} \pi(\tilde{a}|s) M(s'|\tilde{a}, s) \mu^\pi_u \left( s, \tilde{a}, s' \right) } \\
    &= \pi(a | s) \frac{ \mu^\pi_u \left( s, a, h(s, a) \right) }{\sum_{\tilde{a}} \pi(\tilde{a}|s) \mu^\pi_u \left( s, \tilde{a}, h(s, \tilde{a}) \right) } \\
    &\overset{(a)}{=} \pi(a | s) \mu^\pi_u \left( s, a, h(s, a) \right) \\ 
    &= \pi(a | s) \frac{ p_{X^\pi(h(s, a))}\left( \frac{q^\pi_u(s) - c(s, a)}{\gamma} \right)}{\gamma p_{X^\pi(s)}\left( q^\pi_u(s) \right)} \\
    &\propto \pi(a | s) \cdot p_{X^\pi(h(s, a))}\left( \frac{q^\pi_u(s) - c(s, a)}{\gamma} \right).
\end{align}
\endgroup
Here the equality (a) holds from \eqref{appendix:eq:expectation_is_1}. Thus, from Corollary \ref{appendix:cor:td_relation_quantile_bound}, we can obtain the following approximation:
\begin{equation}
    \Ebb_{a_t \sim \tilde{\pi}_u} \left[ c(s_t, a_t) + \gamma q^\pi_u(s_{t+1}) \right] - \frac{\epsilon R}{1-\epsilon} \leq q^\pi_u(s_t) \leq \Ebb_{a_t \sim \tilde{\pi}_u} \left[ c(s_t, a_t) + \gamma q^\pi_u(s_{t+1}) \right] + \frac{\epsilon R}{1-\epsilon}. \label{appendix:eq:td_quantile_distorted_policy}
\end{equation}
Therefore, we  obtain
\begingroup
\allowdisplaybreaks
\begin{align}
    q^\pi_u(s_0) &\leq \Ebb_{a_0 \sim \tilde{\pi}_u} \left[ c(s_0, a_0) + \gamma q^\pi_u(s_1) \right] + \frac{\epsilon R}{1-\epsilon} \\
        &\leq \Ebb_{a_0, a_1 \sim \tilde{\pi}_u} \left[ c(s_0, a_0) + \gamma c(s_1, a_1) + \gamma^2 q^\pi_u(s_2) \right] + \frac{\epsilon R}{1-\epsilon} \left( 1 + \gamma \right) \\
        &\leq \cdots \\
        &\leq \Ebb_{\tilde{\pi}_u} \left[ \sum^\infty_{t=0} \gamma^t c(s_t, a_t) \right] + \frac{\epsilon R}{(1-\epsilon)(1-\gamma)}, \\
    q^\pi_u(s_0) &\geq \Ebb_{a_0 \sim \tilde{\pi}_u} \left[ c(s_0, a_0) + \gamma q^\pi_u(s_1) \right] - \frac{\epsilon R}{1-\epsilon} \\
        &\geq \Ebb_{a_0, a_1 \sim \tilde{\pi}_u} \left[ c(s_0, a_0) + \gamma c(s_1, a_1) + \gamma^2 q^\pi_u(s_2) \right] - \frac{\epsilon R}{1-\epsilon} \left( 1 + \gamma \right) \\
        &\geq \cdots \\
        &\geq \Ebb_{\tilde{\pi}_u} \left[ \sum^\infty_{t=0} \gamma^t c(s_t, a_t) \right] - \frac{\epsilon R}{(1-\epsilon)(1-\gamma)}.
\end{align}
\endgroup

\end{proof}

\subsection{Proof of Theorem \ref{appendix:thm:expectation_form_quantile_policy_dependent_cost}}
\begin{mythm} \label{appendix:thm:expectation_form_quantile_policy_dependent_cost}
    Under  deterministic dynamics $s_{t+1} = h(s_t, a_t)$ and Assumptions \ref{assm:boundness} and \ref{assm:smoothness}, $q_u^\pi(s)$ can be expressed as 
    \begin{equation}
        \left\vert q^{\pi}_u (s_0) - \Ebb_{\pi} \left[ \sum^\infty_{t=0} \gamma^t \left\{ c(s_t, a_t) + \tilde{c}^\pi_u(s_t, a_t) \right\} \right] \right\vert \leq \frac{\epsilon R}{(1-\epsilon)(1-\gamma)}, \label{appendix:eq:quantile_approximation_policy_dependent_cost}
    \end{equation}
where    
\begin{equation}
    \tilde{c}^\pi_u(s,a) = \left( \frac{ p_{X^\pi(s')}\left( \frac{q^\pi_u(s) - c(s, a)}{\gamma} \right)}{\gamma p_{X^\pi(s)}\left( q^\pi_u(s) \right)} - 1 \right) \cdot \left[ c(s, a) + \gamma q^\pi_u(h(s, a)) \right]. \nonumber
\end{equation}
\end{mythm}

\begin{proof}
    From \eqref{appendix:eq:td_quantile_distorted_policy}, we have
    \begingroup
    \allowdisplaybreaks
    \begin{align}
        q^\pi_u(s_t) &\leq \Ebb_{a_t \sim \tilde{\pi}_u} \left[ c(s_t, a_t) + \gamma q^\pi_u(s_{t+1}) \right] + \frac{\epsilon R}{1-\epsilon} \\
        q^\pi_u(s_t) &\geq \Ebb_{a_t \sim \tilde{\pi}_u} \left[ c(s_t, a_t) + \gamma q^\pi_u(s_{t+1}) \right] - \frac{\epsilon R}{1-\epsilon}
    \end{align}
    \endgroup
    The expectation $\Ebb_{a_t \sim \tilde{\pi}_u} \left[ c(s_t, a_t) + \gamma q^\pi_u(s_{t+1}) \right]$ can be rewritten as
    \begingroup
    \allowdisplaybreaks
    \begin{align}
        \Ebb_{a_t \sim \tilde{\pi}_u} \left[ c(s_t, a_t) + \gamma q^\pi_u(s_{t+1}) \right] &= \Ebb_{a_t \sim \pi} \left[ \frac{\tilde{\pi}_u(a_t | s_t)}{\pi(a_t | s_t)} \cdot \left\{ c(s_t, a_t) + \gamma q^\pi_u({h(s_t, a_t)}) \right\} \right] \\
        &= \Ebb_{a_t \sim \pi} \left[  c(s_t, a_t) + \tilde{c}^\pi_u(s_t, a_t) + \gamma q^\pi_u({h(s_t, a_t)}) \right] \\
        &= \Ebb_{a_t \sim \pi} \left[  c(s_t, a_t) + \tilde{c}^\pi_u(s_t, a_t) + \gamma q^\pi_u(s_{t+1}) \right] \label{appendix:eq:td_quantile_policy_dependent_additional_cost}
    \end{align}
    \endgroup
    where
    \begingroup
    \allowdisplaybreaks
    \begin{align}
        \tilde{c}^\pi_u(s,a) &:= \left( \frac{\tilde{\pi}_u(a | s)}{\pi(a | s)} - 1 \right) \cdot \left\{ c(s, a) + \gamma q^\pi_u({h(s, a)}) \right\} \\
        &= \left( \frac{ p_{X^\pi(h(s,a))}\left( \frac{q^\pi_u(s) - c(s, a)}{\gamma} \right)}{\gamma p_{X^\pi(s)}\left( q^\pi_u(s) \right)} - 1 \right) \cdot \left\{ c(s, a) + \gamma q^\pi_u({h(s, a)}) \right\}.
    \end{align}
    \endgroup
    Then, using \eqref{appendix:eq:td_quantile_policy_dependent_additional_cost}, we  obtain 
    \begingroup
    \allowdisplaybreaks
    \begin{align}
        q^\pi_u(s_0) &\leq \Ebb_{a_0 \sim \tilde{\pi}_u} \left[ c(s_0, a_0) + \gamma q^\pi_u(s_1) \right] + \frac{\epsilon R}{1-\epsilon} \\
        &= \Ebb_{a_0 \sim \pi} \left[  c(s_0, a_0) + \tilde{c}^\pi_u(s_0, a_0) + \gamma q^\pi_u(s_1) \right] + \frac{\epsilon R}{1-\epsilon}\\
        &\leq \Ebb_{a_0, a_1 \sim \pi} \left[  \left\{ c(s_0, a_0) + \tilde{c}^\pi_u(s_0, a_0) \right\} + \gamma \left\{ c(s_1, a_1) + \tilde{c}^\pi_u(s_1, a_1) \right\} + \gamma^2 q^\pi_u(s_2) \right] \\
        &~~~~~ + \frac{\epsilon R}{1-\epsilon} \left( 1 + \gamma \right) \\
        &\leq \cdots \\
        &\leq \Ebb_{\pi} \left[ \sum^\infty_{t=0} \gamma^t \left\{ c(s_t, a_t) + \tilde{c}^\pi_u(s_t, a_t) \right\} \right] + \frac{\epsilon R}{(1-\epsilon)(1-\gamma)} \\
        q^\pi_u(s_0) &\geq \Ebb_{a_0 \sim \tilde{\pi}_u} \left[ c(s_0, a_0) + \gamma q^\pi_u(s_1) \right] - \frac{\epsilon R}{1-\epsilon} \\
        &= \Ebb_{a_0 \sim \pi} \left[  c(s_0, a_0) + \tilde{c}^\pi_u(s_0, a_0) + \gamma q^\pi_u(s_1) \right] - \frac{\epsilon R}{1-\epsilon}\\
        &\geq \Ebb_{a_0, a_1 \sim \pi} \left[  \left\{ c(s_0, a_0) + \tilde{c}^\pi_u(s_0, a_0) \right\} + \gamma \left\{ c(s_1, a_1) + \tilde{c}^\pi_u(s_1, a_1) \right\} + \gamma^2 q^\pi_u(s_2) \right] \\
        &~~~~~ - \frac{\epsilon R}{1-\epsilon} \left( 1 + \gamma \right) \\
        &\geq \cdots \\
        &\geq \Ebb_{\pi} \left[ \sum^\infty_{t=0} \gamma^t \left\{ c(s_t, a_t) + \tilde{c}^\pi_u(s_t, a_t) \right\} \right] - \frac{\epsilon R}{(1-\epsilon)(1-\gamma)} \\
    \end{align}
    \endgroup
\end{proof}

\subsection{Proof of Theorem \ref{appendix:thm:expectation_form_quantile_policy_independent_cost}}
\begin{myassm}[Lipschitz continuity of $\tilde{c}^\pi_u(s,a)$ over $\pi$]  \label{appendix:assm:additional_cost}
    For any given fixed $u \in (0, 1)$ and any policies $\pi$ and $\pi'$, there exists a coefficient $C_u$ such that
    \begin{equation}
        \left\vert \tilde{c}^{\pi'}_u(s, a) - \tilde{c}^{\pi}_u(s, a) \right\vert \leq C_u \cdot \max_{s'} \KL \left( \pi'(\cdot | s') ~\big\Vert~ \pi(\cdot | s') \right)
    \end{equation}
    for all $s \in \Scal$, $a \in \Acal$.
\end{myassm}
Basically, Assumption \ref{appendix:assm:additional_cost} is that the function $\tilde{c}_u^\pi$ as a function of $\pi$ is continuous, which is expected to be satisfied if there is no abrupt change in the associated distributions.

\begin{mythm} \label{appendix:thm:expectation_form_quantile_policy_independent_cost}
    Under  deterministic dynamics $s_{t+1} = h(s_t, a_t)$ and Assumptions \ref{appendix:assm:boundness}, \ref{appendix:assm:smoothness}, and \ref{appendix:assm:additional_cost}, the $u$-quantile $q^\pi_u(s_0)$ is expressed as  the expectation of the sum of actual cost and a $\pi$-independent additional cost  $\tilde{c}^{\pi'}_u(s, a)$ for $\pi'$ satisfying $\max_s \KL(\pi'(\cdot |s) ~\Vert~ \pi(\cdot |s)) \leq \delta$:
    \begin{equation}
        \left\vert q^{\pi}_u (s_0) - \Ebb_{\pi} \left[ \sum^\infty_{t=0} \gamma^t \left\{ c(s_t, a_t) + \tilde{c}^{\pi'}_u(s_t, a_t) \right\} \right] \right\vert \leq \frac{\epsilon R}{(1-\epsilon)(1-\gamma)} + \frac{C_u}{1-\gamma} \delta. 
    \end{equation}
\end{mythm}
\begin{proof}
    From Assumption \ref{appendix:assm:additional_cost}, the additional cost $\tilde{c}^\pi_u(s,a)$ is bounded as follows
    \begingroup
    \allowdisplaybreaks
    \begin{align}
        \tilde{c}^\pi_u(s,a) &\leq \tilde{c}^{\pi'}_u(s, a) + C_u \cdot \max_s \KL(\pi'(\cdot |s) ~\Vert~ \pi(\cdot |s)) \\
        &\leq \tilde{c}^{\pi'}_u(s, a) + C_u \cdot \delta \\
        \tilde{c}^\pi_u(s,a) &\geq \tilde{c}^{\pi'}_u(s, a) - C_u \cdot \max_s \KL(\pi'(\cdot |s) ~\Vert~ \pi(\cdot |s)) \\
        &\geq \tilde{c}^{\pi'}_u(s, a) - C_u \cdot \delta
    \end{align}
    \endgroup
    for $\pi'$ satisfying $\max_s \KL(\pi'(\cdot |s) ~\Vert~ \pi(\cdot |s)) \leq \delta$. Thus,  from \eqref{appendix:eq:quantile_approximation_policy_dependent_cost}, we can obtain the following bounds
    \begingroup
    \allowdisplaybreaks
    \begin{align}
        q^\pi_u(s_0) &\leq \Ebb_{\pi} \left[ \sum^\infty_{t=0} \gamma^t \left\{ c(s_t, a_t) + \tilde{c}^\pi_u(s_t, a_t) \right\} \right] + \frac{\epsilon R}{(1-\epsilon)(1-\gamma)} \\
        &\leq \Ebb_{\pi} \left[ \sum^\infty_{t=0} \gamma^t \left\{ c(s_t, a_t) + \tilde{c}^{\pi'}_u(s_t, a_t) + C_u \cdot \delta \right\} \right] + \frac{\epsilon R}{(1-\epsilon)(1-\gamma)} \\
        &= \Ebb_{\pi} \left[ \sum^\infty_{t=0} \gamma^t \left\{ c(s_t, a_t) + \tilde{c}^{\pi'}_u(s_t, a_t) \right\} \right] + \frac{\epsilon R}{(1-\epsilon)(1-\gamma)} + \frac{C_u }{1 - \gamma} \delta \\
        q^\pi_u(s_0) &\geq \Ebb_{\pi} \left[ \sum^\infty_{t=0} \gamma^t \left\{ c(s_t, a_t) + \tilde{c}^\pi_u(s_t, a_t) \right\} \right] - \frac{\epsilon R}{(1-\epsilon)(1-\gamma)} \\
        &\geq \Ebb_{\pi} \left[ \sum^\infty_{t=0} \gamma^t \left\{ c(s_t, a_t) + \tilde{c}^{\pi'}_u(s_t, a_t) - C_u \cdot \delta \right\} \right] - \frac{\epsilon R}{(1-\epsilon)(1-\gamma)} \\
        &= \Ebb_{\pi} \left[ \sum^\infty_{t=0} \gamma^t \left\{ c(s_t, a_t) + \tilde{c}^{\pi'}_u(s_t, a_t) \right\} \right] - \frac{\epsilon R}{(1-\epsilon)(1-\gamma)} - \frac{C_u }{1 - \gamma} \delta
    \end{align}
    \endgroup
\end{proof}

\subsection{Proof of Policy Improvement Condition} \label{appendix:proof_improvement_theorem}

\begin{mylem}[Telescoping Lemma for $u$-quantile] \label{appendix:lem:telescoping_lemma_quantile}
    Under deterministic dynamics $s_{t+1} = h(s_t, a_t)$ and Assumption \ref{appendix:assm:boundness} and \ref{appendix:assm:smoothness}, the following holds for any two policies $\pi$ and $\pi'$:
    \begin{align}
        &\left\vert  q^{\pi}_u(s_0) - \left\{ q^{\pi'}_u(s_0) + \Ebb_{\pi} \left[ \sum^\infty_{t=0} \gamma^t \left( c(s_t, a_t) + \tilde{c}^{\pi}_u(s_t, a_t) + \gamma q^{\pi'}_u(s_{t+1}) - q^{\pi'}_u(s_t)  \right) \right] \right\} \right\vert \\
        &\leq \frac{\epsilon R}{(1-\epsilon)(1-\gamma)}     
    \end{align}
\end{mylem}

\begin{proof}
    With Assumption \ref{appendix:assm:boundness} and \ref{appendix:assm:smoothness}, we have the following inequality by Theorem \ref{appendix:thm:expectation_form_quantile_policy_dependent_cost}:
    \begingroup
    \allowdisplaybreaks
    \begin{align}
        q^{\pi}_u (s_0) &\leq \Ebb_{\pi} \left[ \sum^\infty_{t=0} \gamma^t \left\{ c(s_t, a_t) + \tilde{c}^{\pi}_u(s_t, a_t) \right\} \right] + \frac{\epsilon R}{(1-\epsilon)(1-\gamma)} \\
        q^{\pi}_u (s_0) &\geq \Ebb_{\pi} \left[ \sum^\infty_{t=0} \gamma^t \left\{ c(s_t, a_t) + \tilde{c}^{\pi}_u(s_t, a_t) \right\} \right]  -  \frac{\epsilon R}{(1-\epsilon)(1-\gamma)}
    \end{align}
    \endgroup
    Then note that
    \begin{equation}
        q^{\pi'}_u (s_0) = - \Ebb_{\pi} \left[ \sum^\infty_{t=0} \gamma^t \left\{ \gamma q^{\pi'}_u(s_{t+1}) - q^{\pi'}_u(s_{t}) \right\} \right].
    \end{equation}
    Therefore,
    \begingroup
    \allowdisplaybreaks
    \begin{align}
        &q^{\pi}_u (s_0) - q^{\pi'}_u (s_0) \\
        &\leq \Ebb_{\pi} \left[ \sum^\infty_{t=0} \gamma^t \left\{ c(s_t, a_t) + \tilde{c}^{\pi}_u(s_t, a_t) \right\} \right] + \Ebb_{\pi} \left[ \sum^\infty_{t=0} \gamma^t \left\{ \gamma q^{\pi'}_u (s_{t+1}) - q^{\pi'}_u (s_t) \right\} \right] \\
        &~~~~~ + \frac{\epsilon R}{(1-\epsilon)(1-\gamma)} \\
        &= \Ebb_{\pi} \left[ \sum^\infty_{t=0} \gamma^t \left\{ c(s_t, a_t) + \tilde{c}^{\pi}_u(s_t, a_t) + \gamma q^{\pi'}_u (s_{t+1}) - q^{\pi'}_u (s_t) \right\} \right] + \frac{\epsilon R}{(1-\epsilon)(1-\gamma)} \\
        &q^{\pi}_u (s_0) - q^{\pi'}_u (s_0) \\
        &\geq \Ebb_{\pi} \left[ \sum^\infty_{t=0} \gamma^t \left\{ c(s_t, a_t) + \tilde{c}^{\pi}_u(s_t, a_t) \right\} \right] + \Ebb_{\pi} \left[ \sum^\infty_{t=0} \gamma^t \left\{ \gamma q^{\pi'}_u (s_{t+1}) - q^{\pi'}_u (s_t) \right\} \right] \\
        &~~~~~ - \frac{\epsilon R}{(1-\epsilon)(1-\gamma)} \\
        &= \Ebb_{\pi} \left[ \sum^\infty_{t=0} \gamma^t \left\{ c(s_t, a_t) + \tilde{c}^{\pi}_u(s_t, a_t) + \gamma q^{\pi'}_u (s_{t+1}) - q^{\pi'}_u (s_t) \right\} \right] - \frac{\epsilon R}{(1-\epsilon)(1-\gamma)}
    \end{align}
    \endgroup
\end{proof}

Next, we can obtain the following corollary.
\begin{mycor} \label{appendix:cor:telescoping_lemma_quantile}
    Under deterministic dynamics $s_{t+1} = h(s_t, a_t)$ and Assumptions \ref{appendix:assm:boundness}, \ref{appendix:assm:smoothness}, and \ref{appendix:assm:additional_cost}, the following holds for any two policies $\pi$ and $\pi'$ :
    \begin{align}
        &\left\vert  q^{\pi}_u(s_0) - \left\{ q^{\pi'}_u(s_0) + \Ebb_{\pi} \left[ \sum^\infty_{t=0} \gamma^t \left( c(s_t, a_t) +  \tilde{c}^{\pi'}_u(s_t, a_t)  + \gamma q^{\pi'}_u(s_{t+1}) - q^{\pi'}_u(s_t)  \right) \right] \right\} \right\vert \\
        &\leq \frac{\epsilon R}{(1-\epsilon)(1-\gamma)} + \frac{C_u}{1-\gamma} \max_s \KL(\pi'(\cdot |s) ~\Vert~ \pi(\cdot |s))
    \end{align}
\end{mycor}
\begin{proof}
    Let denote $\delta = \max_s \KL(\pi'(\cdot |s) ~\Vert~ \pi(\cdot |s))$ for simplicity in this proof.
    If Assumption \ref{appendix:assm:additional_cost} holds, this can be rewritten as
    \begingroup
    \allowdisplaybreaks
    \begin{align}
    &q^{\pi}_u (s_0) - q^{\pi'}_u (s_0) \\
    &\leq \Ebb_{\pi} \left[ \sum^\infty_{t=0} \gamma^t \left\{ c(s_t, a_t) + \tilde{c}^{\pi}_u(s_t, a_t) + \gamma q^{\pi'}_u (s_{t+1}) - q^{\pi'}_u (s_t) \right\} \right] + \frac{\epsilon R}{(1-\epsilon)(1-\gamma)} \\
    &\leq \Ebb_{\pi} \left[ \sum^\infty_{t=0} \gamma^t \left\{ c(s_t, a_t) + \tilde{c}^{\pi'}_u(s_t, a_t) + C_u \cdot \delta + \gamma q^{\pi'}_u (s_{t+1}) - q^{\pi'}_u (s_t) \right\} \right] + \frac{\epsilon R}{(1-\epsilon)(1-\gamma)} \\
    &= \Ebb_{\pi} \left[ \sum^\infty_{t=0} \gamma^t \left\{ c(s_t, a_t) + \tilde{c}^{\pi'}_u(s_t, a_t) + \gamma q^{\pi'}_u (s_{t+1}) - q^{\pi'}_u (s_t) \right\} \right] + \frac{\epsilon R}{(1-\epsilon)(1-\gamma)} + \frac{C_u}{1-\gamma} \delta\\
    &q^{\pi}_u (s_0) - q^{\pi'}_u (s_0) \\
    &\geq \Ebb_{\pi} \left[ \sum^\infty_{t=0} \gamma^t \left\{ c(s_t, a_t) + \tilde{c}^{\pi}_u(s_t, a_t) + \gamma q^{\pi'}_u (s_{t+1}) - q^{\pi'}_u (s_t) \right\} \right] - \frac{\epsilon R}{(1-\epsilon)(1-\gamma)} \\
    &\geq \Ebb_{\pi} \left[ \sum^\infty_{t=0} \gamma^t \left\{ c(s_t, a_t) + \tilde{c}^{\pi'}_u(s_t, a_t) - C_u \cdot \delta + \gamma q^{\pi'}_u (s_{t+1}) - q^{\pi'}_u (s_t) \right\} \right] - \frac{\epsilon R}{(1-\epsilon)(1-\gamma)} \\
    &= \Ebb_{\pi} \left[ \sum^\infty_{t=0} \gamma^t \left\{ c(s_t, a_t) + \tilde{c}^{\pi'}_u(s_t, a_t) + \gamma q^{\pi'}_u (s_{t+1}) - q^{\pi'}_u (s_t) \right\} \right] - \frac{\epsilon R}{(1-\epsilon)(1-\gamma)} - \frac{C_u}{1-\gamma} \delta
    \end{align}
    \endgroup
\end{proof}

To prove improvement theorem, we need a definition of $\alpha$-coupled policy and several lemmas similar to \cite{schulman2015trust}.
\begin{mydef}[From \cite{schulman2015trust}] 
The two policies $\pi$ and $\pi'$ are $\alpha$-coupled if $\Pr\left( a \neq a' \right) \leq \alpha$, $(a, a') \sim (\pi(a | s), \pi'(a' | s))$ for all $s$.
\end{mydef}

For the $u$-quantile, we define an advantage function $A^\pi_u(s,a)$ using the additional cost function $\tilde{c}^\pi_u(s,a)$ as 
\begin{equation}
    A^\pi_u(s,a) := c(s,a) + \tilde{c}^\pi_u(s,a) + \gamma \Ebb_{s'} \left[ q^\pi_u(s') \right] - q^\pi_u(s) \label{appendix:eq:def_advantage_quantile}
\end{equation}
\begin{mylem}[Similar to Lemma 2 in \cite{schulman2015trust}] \label{appendix:lem:trpo_lem2}
    Under deterministic dynamics $s_{t+1} = h(s_t, a_t)$ and Assumptions \ref{appendix:assm:boundness} and \ref{appendix:assm:smoothness}, $\alpha$-coupled policies $\pi$ and $\pi'$ satisfy the following inequality
    \begin{equation}
        \left\vert \Ebb_{\pi}\left[ A^{\pi'}_u(s,a) \right] \right\vert \leq 2 \alpha \max_{s,a} \left\vert A^{\pi'}_u(s,a) \right\vert +  \frac{\epsilon R}{(1-\epsilon)}
    \end{equation}
    for all $s$.
\end{mylem}

\begin{proof} (Similar to the proof of Lemma 2 in \cite{schulman2015trust}) 
    First we note that the following holds by \eqref{appendix:eq:td_quantile_policy_dependent_additional_cost} and Theorem \ref{appendix:thm:td_relation_quantile_bound}:
    \begingroup
    \allowdisplaybreaks
    \begin{align}
        \left\vert \Ebb_{a \sim \pi'} \left[ A^{\pi'}_u(s, a) \right] \right\vert &=\left\vert \Ebb_{a \sim \pi'} \left[  c(s, a) + \tilde{c}^{\pi'}_u(s, a) + \gamma q^{\pi'}_u(s') - q^{\pi'}_u(s) \right] \right\vert \\
        &= \left\vert \Ebb_{a \sim \pi'} \left[ \frac{ p_{X^{\pi'}(s')}\left( \frac{q^{\pi'}_u(s) - c(s, a)}{\gamma} \right)}{\gamma p_{X^{\pi'}(s)}\left( q^{\pi'}_u(s) \right)} \left\{ c(s, a) + \gamma q^{\pi'}_u(s') - q^{\pi'}_u(s) \right\} \right] \right\vert \\
        &\leq \frac{\epsilon R}{(1-\epsilon)} \label{appendix:eq:lem3_1}
    \end{align}
    \endgroup
    Therefore,
    \begingroup
    \allowdisplaybreaks
    \begin{align}
        \left\vert \Ebb_{\pi}\left[ A^{\pi'}_u(s,a) \right] \right\vert &\overset{(a)}{\leq} \left\vert \Ebb_{a \sim \pi} \left[ A^{\pi'}_u(s, a) \right] - \Ebb_{a' \sim \pi'} \left[ A^{\pi'}_u(s, a') \right] \right\vert + \left\vert \Ebb_{\pi'}\left[ A^{\pi'}_u(s,a) \right] \right\vert \\
        &\overset{(b)}{\leq} \left\vert \Ebb_{a \sim \pi} \left[ A^{\pi'}_u(s, a) \right] - \Ebb_{a' \sim \pi'} \left[ A^{\pi'}_u(s, a') \right] \right\vert + \frac{\epsilon R}{(1-\epsilon)} \\
        &= \left\vert \Ebb_{(a, a') \sim (\pi, \pi')} \left[ A^{\pi'}_u(s, a) -  A^{\pi'}_u(s, a') \right] \right\vert + \frac{\epsilon R}{(1-\epsilon)} \\
        &= \biggl\vert \Pr\left(a = a'\right) \Ebb_{(a, a') \sim (\pi, \pi') |_{a = a'}} \left[ A^{\pi'}_u(s,a) - A^{\pi'}_u(s,a') \right] \\
        &~~~~~ +  \Pr\left(a \neq a'\right) \Ebb_{(a, a') \sim (\pi, \pi') |_{a \neq a'}} \left[  A^{\pi'}_u(s,a) - A^{\pi'}_u(s,a') \right] \biggr\vert + \frac{\epsilon R}{(1-\epsilon)} \\
        &= \Pr\left(a \neq a'\right) \left\vert \Ebb_{(a, a') \sim (\pi, \pi') |_{a \neq a'}} \left[  A^{\pi'}_u(s,a) - A^{\pi'}_u(s,a') \right] \right\vert + \frac{\epsilon R}{(1-\epsilon)} \\
        &\overset{(c)}{\leq} 2 \alpha \max_{s,a} \left\vert A^{\pi'}_u(s,a) \right\vert +  \frac{\epsilon R}{(1-\epsilon)}
    \end{align}
    \endgroup
    where (a) holds by the triangular inequality, (b) holds by \eqref{appendix:eq:lem3_1}, and (c) holds since $\pi$ and $\pi'$ are $\alpha$-coupled policies.
\end{proof}

\begin{mylem}[Similar to Lemma 3 in \cite{schulman2015trust}] \label{appendix:lem:trpo_lem3}
    Under deterministic dynamics $s_{t+1} = h(s_t,a_t)$ and Assumptions \ref{appendix:assm:boundness} and \ref{appendix:assm:smoothness}, the following holds for $\alpha$-coupled policies $\pi$ and $\pi'$
    \begingroup
    \allowdisplaybreaks
    \begin{align}
        &\left\vert \Ebb_{ \tcr{ s_t \sim \pi } }\left[ \Ebb_{ a \sim \pi}\left[ A^{\pi'}_u(s_t ,a) \right] \right] - \Ebb_{ \tcb{ s_t \sim \pi' } }\left[ \Ebb_{a \sim \pi}\left[ A^{\pi'}_u(s_t, a) \right] \right] \right\vert \\
        &\leq \left( 1 - (1-\alpha)^t \right) \left\{ 4 \alpha \max_{s,a} \left\vert A^{\pi'}_u(s,a) \right\vert +  \frac{2 \epsilon R}{(1-\epsilon)} \right\}
    \end{align}
    \endgroup
\end{mylem}

\begin{proof} (Similar to the proof of Lemma 3 in \cite{schulman2015trust})
    For $\alpha$-coupled policies $\pi$ and $\pi'$, first we consider trajectories drawn from each policy, i.e., $\tau=(s_0, a_0, s_1, a_1, \ldots) \sim \pi$ and $\tau'=(s_0, a_0', s_1', a_1', \ldots) \sim \pi'$. We consider the timestep $t$ and observe the advantage of $\pi'$ over $\pi$. Let define $n_t$ as the number of times that mismatched actions occurs, $a_i \neq a_i'$ for $i < t$. Then
    \begingroup
    \allowdisplaybreaks
    \begin{align}
        &\Ebb_{ \tcr{ s_t \sim \pi } }\left[ \Ebb_{a \sim \pi(\cdot | s_t)}\left[ A^{\pi'}_u(s_t, a) \right] \right] \\
        &= P(n_t = 0) \cdot \Ebb_{ \tcr{ s_t \sim \pi } | n_t = 0}\left[ \Ebb_{a \sim \pi(\cdot | s_t)}\left[ A^{\pi'}_u(s_t, a) \right] \right] \nonumber \\
        &~~~~~ + P(n_t > 0) \cdot \Ebb_{ \tcr{ s_t \sim \pi } | n_t > 0}\left[ \Ebb_{a \sim \pi(\cdot | s_t)}\left[ A^{\pi'}_u(s_t, a) \right] \right] \label{appendix:eq:advantage_diff_pip} \\
        &\Ebb_{ \tcb{ s_t \sim \pi' } }\left[ \Ebb_{a \sim \pi(\cdot | s_t)}\left[ A^{\pi'}_u(s_t, a) \right] \right] \\
        &= P(n_t = 0) \cdot \Ebb_{ \tcb{ s_t \sim \pi' } | n_t = 0}\left[ \Ebb_{a \sim \pi(\cdot | s_t)}\left[ A^{\pi'}_u(s_t, a) \right] \right] \nonumber \\
        &~~~~~ + P(n_t > 0) \cdot \Ebb_{ \tcb{ s_t \sim \pi' } | n_t > 0}\left[ \Ebb_{a \sim \pi(\cdot | s_t)}\left[ A^{\pi'}_u(s_t, a) \right] \right] \label{appendix:eq:advantage_diff_pi}
    \end{align}
    \endgroup
    For the case $n_t = 0$, 
    \begin{equation}
        \Ebb_{\tcr{ s_t \sim \pi } | n_t = 0}\left[ \Ebb_{a \sim \pi(\cdot | s_t)}\left[ A^{\pi'}_u(s_t, a) \right] \right] = \Ebb_{ \tcb{ s_t \sim \pi' } | n_t = 0}\left[ \Ebb_{a \sim \pi(\cdot | s_t)}\left[ A^{\pi'}_u(s_t, a) \right] \right]
    \end{equation}
    Thus by subtracting \eqref{appendix:eq:advantage_diff_pi} and \eqref{appendix:eq:advantage_diff_pip}, we can obtain
    \begingroup
    \allowdisplaybreaks
    \begin{align}
        &\Ebb_{ \tcr{ s_t \sim \pi } }\left[ \Ebb_{a \sim \pi(\cdot | s_t)}\left[ A^{\pi'}_u(s_t, a) \right] \right] - \Ebb_{ \tcb{ s_t \sim \pi' } }\left[ \Ebb_{a \sim \pi(\cdot | s_t)}\left[ A^{\pi'}_u(s_t, a) \right] \right]  \\
        &= P(n_t > 0) \cdot \left( \Ebb_{ \tcr{ s_t \sim \pi } | n_t > 0}\left[ \Ebb_{a \sim \pi(\cdot | s_t)}\left[ A^{\pi'}_u(s_t, a) \right] \right] - \Ebb_{ \tcb{ s_t \sim \pi'} | n_t > 0}\left[ \Ebb_{a \sim \pi(\cdot | s_t)}\left[ A^{\pi'}_u(s_t, a) \right] \right] \right) \label{appendix:eq:trpo_lem3_1}
    \end{align}
    \endgroup
    From the definition of $\alpha$-coupled policy, we get
    \begin{equation}
        P(n_t = 0) \geq (1-\alpha)^t, \qquad P(n_t > 0) \leq 1 - (1-\alpha)^t \label{appendix:eq:trpo_lem3_2}
    \end{equation}
    Then note that
    \begingroup
    \allowdisplaybreaks
    \begin{align}
        &\left\vert \Ebb_{ \tcr{ s_t \sim \pi } | n_t > 0}\left[ \Ebb_{a \sim \pi(\cdot | s_t)}\left[ A^{\pi'}_u(s_t, a) \right] \right] - \Ebb_{ \tcb{ s_t \sim \pi' } | n_t > 0}\left[ \Ebb_{a \sim \pi(\cdot | s_t)}\left[ A^{\pi'}_u(s_t, a) \right] \right] \right\vert \\
        &\overset{(a)}{\leq} \left\vert \Ebb_{ \tcr{ s_t \sim \pi } | n_t > 0}\left[ \Ebb_{a \sim \pi}\left[ A^{\pi'}_u(s_t ,a) \right] \right] \right\vert + \left\vert \Ebb_{ \tcb{ s_t \sim \pi' } | n_t > 0}\left[ \Ebb_{a \sim \pi}\left[ A^{\pi'}_u(s_t, a) \right] \right] \right\vert \\
        &\leq 2 \max_s \left\vert \Ebb_{a \sim \pi}\left[ A^{\pi'}_u(s ,a) \right]  \right\vert \\
        &\overset{(b)}{\leq} 4 \alpha \max_{s,a} \left\vert A^{\pi'}_u(s,a) \right\vert +  \frac{2 \epsilon R}{(1-\epsilon)} \label{appendix:eq:trpo_lem3_3}
    \end{align}
    \endgroup
    where (a) holds by the triangular inequality, and (b) holds by Lemma \ref{appendix:lem:trpo_lem2}. Therefore using \eqref{appendix:eq:trpo_lem3_1}, \eqref{appendix:eq:trpo_lem3_2}, and \eqref{appendix:eq:trpo_lem3_3}, we can conclude
    \begingroup
    \allowdisplaybreaks
    \begin{align}
        &\Ebb_{ \tcr{ s_t \sim \pi } }\left[ \Ebb_{a \sim \pi(\cdot | s_t)}\left[ A^{\pi'}_u(s_t, a) \right] \right] - \Ebb_{ \tcb{ s_t \sim \pi' } }\left[ \Ebb_{a \sim \pi(\cdot | s_t)}\left[ A^{\pi'}_u(s_t, a) \right] \right] \\
        &\leq \left( 1 - (1-\alpha)^t \right) \left\{ 4 \alpha \max_{s,a} \left\vert A^{\pi'}_u(s,a) \right\vert +  \frac{2 \epsilon R}{(1-\epsilon)} \right\}
    \end{align}
    \endgroup
\end{proof}

Now we define $L^{\pi'}_u(\pi)$ as
\begingroup
\allowdisplaybreaks
\begin{align}
    L^{\pi'}_u(\pi) &:= q^{\pi'}_u(s_0) + \Ebb_{\pi'}\left[ \sum^{\infty}_{t=0} \gamma^t \Ebb_{a \sim \pi}\left[ A^{\pi'}_u(s_t, a) \right] \right] \label{appendix:eq:Lpi_pip} \\
    &= q^{\pi'}_u(s_0) + \Ebb_{\pi'}\left[ \sum^{\infty}_{t=0} \gamma^t \Ebb_{a \sim \pi}\left[ c(s_t, a) + \tilde{c}^{\pi'}_u(s_t, a) + \gamma  q^{\pi'}_u(s_{t+1}) - q^{\pi'}_u(s_t) \right] \right]  
\end{align}
\endgroup
Then note that
\begingroup
\allowdisplaybreaks
\begin{align}
    L^{\pi'}_u(\pi') &= q^{\pi'}_u(s_0) + \Ebb_{\pi'}\left[ \sum^{\infty}_{t=0} \gamma^t \left\{ c(s_t, a_t) + \tilde{c}^{\pi'}_u(s_t, a_t) + \gamma  q^{\pi'}_u(s_{t+1}) - q^{\pi'}_u(s_t) \right\} \right] \\
    &= \Ebb_{\pi'}\left[ \sum^{\infty}_{t=0} \gamma^t \left\{ c(s_t, a_t) + \tilde{c}^{\pi'}_u(s_t, a_t) \right\} \right]
\end{align}
\endgroup
Then from Theorem \ref{appendix:thm:expectation_form_quantile_policy_dependent_cost}, 
\begingroup
\allowdisplaybreaks
\begin{align}
    \left\vert q^{\pi'}_u (s_0) - L^{\pi'}_u(\pi') \right\vert &= \left\vert q^{\pi'}_u (s_0) - \Ebb_{\pi'} \left[ \sum^\infty_{t=0} \gamma^t \left\{ c(s_t, a_t) + \tilde{c}^{\pi'}_u(s_t, a_t) \right\} \right] \right\vert \\ 
    &\leq \frac{\epsilon R}{(1-\epsilon)(1-\gamma)}.
\end{align}
\endgroup
Therefore, we get
\begin{equation}
    q^{\pi'}_u (s_0) \geq L^{\pi'}_u(\pi') - \frac{\epsilon R}{(1-\epsilon)(1-\gamma)} \label{appendix:eq:qpi_Lpi}
\end{equation}

\begin{myprop} \label{appendix:prop:qpip_Lpipip_bound}
    Under deterministic dynamics $s_{t+1} = h(s_t, a_t)$ and Assumptions \ref{appendix:assm:boundness}, \ref{appendix:assm:smoothness}, and \ref{appendix:assm:additional_cost}, the following holds
    \begin{equation}
        q^{\pi}_u (s_0) \leq L^{\pi'}_u(\pi) + C_1 \max_s \KL(\pi'(\cdot |s) ~\Vert~ \pi(\cdot |s)) + C_2 \frac{\epsilon}{1-\epsilon}
    \end{equation}
    where 
    \begin{equation}
        C_1 = \left( \frac{4 \gamma \max_{s,a} \left\vert A^{\pi'}_u(s,a) \right\vert + \gamma R }{(1-\gamma)^2} + \frac{C_u}{1-\gamma} \right), \qquad C_2 =  \frac{R}{(1-\gamma)^2}
    \end{equation}
\end{myprop}

\begin{proof}
    Let define $B = \max_{s,a} \left\vert A^{\pi'}_u(s, a) \right\vert$. 
    Remind that the definition of the advantage for the $u$-quantile \eqref{appendix:eq:def_advantage_quantile} $A^{\pi'}_u(s,a) := c(s,a) + \tilde{c}^{\pi'}_u(s,a) + \gamma \Ebb_{s'} \left[ q^{\pi'}_u(s') \right] - q^{\pi'}_u(s)$,
    Corollary \ref{appendix:cor:telescoping_lemma_quantile}
    \begingroup
    \allowdisplaybreaks
    \begin{align}
        &\left\vert  q^{\pi}_u(s_0) -  \left\{ q^{\pi'}_u(s_0) + \Ebb_{\pi} \left[ \sum^\infty_{t=0} \gamma^t \underbrace{\left( c(s_t, a_t) +  \tilde{c}^{\pi'}_u(s_t, a_t) + \gamma q^{\pi'}_u(s_{t+1}) - q^{\pi'}_u(s_t)  \right)}_{= A^{\pi'}_u(s_t, a_t)} \right] \right\}  \right\vert \nonumber \\
        &\leq \frac{\epsilon R}{(1-\epsilon)(1-\gamma)} + \frac{C_u}{1-\gamma} \max_s \KL(\pi'(\cdot |s) ~\Vert~ \pi(\cdot |s)) \label{appendix:eq:cor2_repeat}
    \end{align}
    \endgroup
    and the definition of $L^{\pi'}_u(\pi)$ in \eqref{appendix:eq:Lpi_pip}
    \begin{align}
        L^{\pi'}_u(\pi) &:= q^{\pi'}_u(s_0) + \Ebb_{\pi'}\left[ \sum^{\infty}_{t=0} \gamma^t \Ebb_{a \sim \pi}\left[ A^{\pi'}_u(s_t, a) \right] \right]
    \end{align}
    Then we can obtain
    \begingroup
    \allowdisplaybreaks
    \begin{align}
        &\left\vert q^{\pi}_u(s_0) - L^{\pi'}_u(\pi) \right\vert \\
        &\overset{(a)}{\leq} \left\vert  q^{\pi}_u(s_0) -  \left\{ q^{\pi'}_u(s_0) + \Ebb_{\pi} \left[ \sum^\infty_{t=0} \gamma^t A^{\pi'}_u(s_t, a_t) \right] \right\} \right\vert \\
        &~~~~~ + \left\vert   \left\{ q^{\pi'}_u(s_0) + \Ebb_{\pi} \left[ \sum^\infty_{t=0} \gamma^t A^{\pi'}_u(s_t, a_t) \right] \right\}  - L^{\pi'}_u(\pi) \right\vert \\
        &\overset{(b)}{\leq} \frac{\epsilon R}{(1-\epsilon)(1-\gamma)} + \frac{C_u}{1-\gamma} \max_s \KL(\pi'(\cdot |s) ~\Vert~ \pi(\cdot |s)) \\
        &~~~~~ + \left\vert   \left\{ q^{\pi'}_u(s_0) + \Ebb_{\pi} \left[ \sum^\infty_{t=0} \gamma^t A^{\pi'}_u(s_t, a_t) \right] \right\}  - L^{\pi'}_u(\pi) \right\vert \\
        &= \left\vert   \left\{ q^{\pi'}_u(s_0) + \Ebb_{\pi} \left[ \sum^\infty_{t=0} \gamma^t A^{\pi'}_u(s_t, a_t) \right] \right\}  - \left\{  q^{\pi'}_u(s_0) + \Ebb_{\pi'}\left[ \sum^{\infty}_{t=0} \gamma^t \Ebb_{a \sim \pi}\left[ A^{\pi'}_u(s_t, a) \right] \right] \right\} \right\vert \\
        &~~~~~ + \frac{\epsilon R}{(1-\epsilon)(1-\gamma)} + \frac{C_u}{1-\gamma} \max_s \KL(\pi'(\cdot |s) ~\Vert~ \pi(\cdot |s)) \\
        &= \underbrace{\left\vert  \Ebb_{\pi} \left[ \sum^\infty_{t=0} \gamma^t \Ebb_{a \sim \pi}\left[ A^{\pi'}_u(s_t, a) \right] \right] - \Ebb_{\pi'}\left[ \sum^{\infty}_{t=0} \gamma^t \Ebb_{a \sim \pi}\left[ A^{\pi'}_u(s_t, a) \right] \right]  \right\vert}_{(c)} \\
        &~~~~~ + \frac{\epsilon R}{(1-\epsilon)(1-\gamma)} + \frac{C_u}{1-\gamma} \max_s \KL(\pi'(\cdot |s) ~\Vert~ \pi(\cdot |s)) \label{appendix:eq:prop1_1}
    \end{align}
    \endgroup
    where (a) holds by the triangular inequality, (b) holds from \eqref{appendix:eq:cor2_repeat}. The term (c) can be written as
    \begingroup
    \allowdisplaybreaks
    \begin{align}
        &\left\vert  \Ebb_{\pi} \left[ \sum^\infty_{t=0} \gamma^t \Ebb_{a \sim \pi}\left[ A^{\pi'}_u(s_t, a) \right] \right] - \Ebb_{\pi'}\left[ \sum^{\infty}_{t=0} \gamma^t \Ebb_{a \sim \pi}\left[ A^{\pi'}_u(s_t, a) \right] \right]  \right\vert \\
        &\overset{(a)}{\leq} \sum^{\infty}_{t=0} \gamma^t \biggl\vert \Ebb_{ \tcr{ s_t \sim \pi } }\left[ \Ebb_{a \sim \pi}\left[ A^{\pi'}_u(s, a) \right] \right] - \Ebb_{ \tcb{ s_t \sim \pi' } }\left[ \Ebb_{ a \sim \pi}\left[ A^{\pi'}_u(s, a) \right] \right] \biggr\vert \\
        &\overset{(b)}{\leq} \sum^{\infty}_{t=0} \gamma^t \left( 1 - (1-\alpha)^t \right) \left\{ 4 \alpha \max_{s,a} \left\vert A^{\pi'}_u(s,a) \right\vert +  \frac{2 \epsilon R}{(1-\epsilon)} \right\} \\
        &= \left( 4 \alpha \max_{s,a} \left\vert A^{\pi'}_u(s,a) \right\vert +  \frac{2 \epsilon R}{(1-\epsilon)} \right) \left( \frac{1}{1-\gamma} - \frac{1}{1 - \gamma (1-\alpha)} \right) \\
        &= \left( 4 \alpha \max_{s,a} \left\vert A^{\pi'}_u(s,a) \right\vert +  \frac{2 \epsilon R}{(1-\epsilon)} \right) \frac{\alpha \gamma}{(1-\gamma) (1 - \gamma (1 - \alpha))} \\
        &\overset{(c)}{\leq}  \left( 4 \alpha \max_{s,a} \left\vert A^{\pi'}_u(s,a) \right\vert +  \frac{2 \epsilon R}{(1-\epsilon)} \right) \frac{\alpha \gamma}{(1-\gamma)^2} \label{appendix:eq:prop1_2}
    \end{align}
    \endgroup
    where (a) holds by the triangular inequality, (b) holds by Lemma \ref{appendix:lem:trpo_lem3}, and (c) holds by $\alpha < 1$ ($\alpha$ is for $\alpha$-coupled policy). Therefore by putting \eqref{appendix:eq:prop1_2} into term (c) in \eqref{appendix:eq:prop1_1}, we can obtain
    \begingroup
    \allowdisplaybreaks
    \begin{align}
        &\left\vert q^{\pi}_u(s_0) - L^{\pi'}_u(\pi) \right\vert \\
        &=  \tcb{ \underbrace{\left\vert  \Ebb_{\pi} \left[ \sum^\infty_{t=0} \gamma^t \Ebb_{a \sim \pi}\left[ A^{\pi'}_u(s_t, a) \right] \right] - \Ebb_{\pi'}\left[ \sum^{\infty}_{t=0} \gamma^t \Ebb_{a \sim \pi}\left[ A^{\pi'}_u(s_t, a) \right] \right]  \right\vert}_{(c)} } \\
        &~~~~~ + \frac{\epsilon R}{(1-\epsilon)(1-\gamma)} + \frac{C_u}{1-\gamma} \max_s \KL(\pi'(\cdot |s) ~\Vert~ \pi(\cdot |s)) \\
        &\leq  \tcb{ \left( 4 \alpha \max_{s,a} \left\vert A^{\pi'}_u(s,a) \right\vert +  \frac{2 \epsilon R}{(1-\epsilon)} \right)  \frac{\alpha \gamma}{(1-\gamma)^2} } + \frac{\epsilon R}{(1-\epsilon)(1-\gamma)} + \frac{C_u}{1-\gamma} \max_s \KL(\pi'(\cdot |s) ~\Vert~ \pi(\cdot |s)) \\
        &= \left( \frac{4 \gamma \max_{s,a} \left\vert A^{\pi'}_u(s,a) \right\vert }{(1-\gamma)^2} \right) \alpha^2 + \tcr{ \left( 2 \alpha \frac{\epsilon}{1-\epsilon} \right) } \cdot \frac{\gamma R}{ (1-\gamma)^2} \\
        &~~~~~ + \left( \frac{ (1-\gamma) R}{(1-\epsilon) (1-\gamma)^2} \right) \epsilon + \frac{C_u}{1-\gamma} \max_s \KL(\pi'(\cdot |s) ~\Vert~ \pi(\cdot |s)) \\
        &\overset{(a)}{\leq} \left( \frac{4 \gamma \max_{s,a} \left\vert A^{\pi'}_u(s,a) \right\vert }{(1-\gamma)^2} \right) \alpha^2 + \tcr{ \left\{ \alpha^2 + \left( \frac{\epsilon}{1-\epsilon}\right)^2 \right\} } \cdot \frac{\gamma R}{ (1-\gamma)^2} \\
        &~~~~~ + \left( \frac{ (1-\gamma) R}{(1-\epsilon) (1-\gamma)^2} \right) \epsilon + \frac{C_u}{1-\gamma} \max_s \KL(\pi'(\cdot |s) ~\Vert~ \pi(\cdot |s)) \\
        &= \left( \frac{4 \gamma \max_{s,a} \left\vert A^{\pi'}_u(s,a) \right\vert + \gamma R }{(1-\gamma)^2} \right) \alpha^2 + \frac{C_u}{1-\gamma} \max_s \KL(\pi'(\cdot |s) ~\Vert~ \pi(\cdot |s)) \\
        &~~~~~ + \left( \frac{ (1-\gamma) R}{(1-\gamma)^2} \right) \left( \frac{\epsilon}{1-\epsilon}\right) + \frac{\gamma R}{ (1-\gamma)^2} \left( \frac{\epsilon}{1-\epsilon}\right)^2 \\
        &\overset{(b)}{\leq} \left( \frac{4 \gamma \max_{s,a} \left\vert A^{\pi'}_u(s,a) \right\vert + \gamma R }{(1-\gamma)^2} \right) \alpha^2 + \frac{C_u}{1-\gamma} \max_s \KL(\pi'(\cdot |s) ~\Vert~ \pi(\cdot |s)) \\
        &~~~~~ + \left( \frac{ (1-\gamma) R}{(1-\gamma)^2} \right) \left( \frac{\epsilon}{1-\epsilon}\right) + \frac{\gamma R}{ (1-\gamma)^2} \left( \frac{\epsilon}{1-\epsilon}\right) \\
        &= \frac{R}{(1-\gamma)^2} \left( \frac{\epsilon}{1-\epsilon} \right) + \left( \frac{4 \gamma \max_{s,a} \left\vert A^{\pi'}_u(s,a) \right\vert + \gamma R }{(1-\gamma)^2} \right) \alpha^2 + \frac{C_u}{1-\gamma} \max_s \KL(\pi'(\cdot |s) ~\Vert~ \pi(\cdot |s)) \label{appendix:eq:prop1_3}
    \end{align}
    \endgroup
    where (a) holds by the inequality of arithmetic and geometric means, and (b) holds from the definition of $0 < \epsilon < \frac{1}{2}$ in Assumption \ref{appendix:assm:smoothness}. Like \cite{schulman2015trust}, if we take $\alpha$ as the maximum of the total variation of two policies $\pi$ and $\pi'$, i.e., $\alpha = \max_s D_{TV}(\pi'(\cdot | s) || \pi(\cdot | s))$, then these policies are $\alpha$-coupled. Since $D_{TV}(\pi'(\cdot | s) || \pi(\cdot | s))^2 \leq \KL(\pi'(\cdot | s) || \pi(\cdot | s))$, eq. \eqref{appendix:eq:prop1_3} becomes
    \begin{align}
        \left\vert q^{\pi}_u(s_0) - L^{\pi'}_u(\pi) \right\vert &\leq C_1 \max_s \KL(\pi'(\cdot |s) ~\Vert~ \pi(\cdot |s)) + C_2 \frac{\epsilon}{1-\epsilon}
    \end{align}
    where
    \begin{equation}
        C_1 = \left( \frac{4 \gamma \max_{s,a} \left\vert A^{\pi'}_u(s,a) \right\vert + \gamma R }{(1-\gamma)^2} + \frac{C_u}{1-\gamma} \right), \qquad C_2 =  \frac{R}{(1-\gamma)^2}
    \end{equation}
\end{proof}

Together with Proposition \ref{appendix:prop:qpip_Lpipip_bound} above, and Theorem 1 in \cite{schulman2015trust}, we can obtain the Theorem \ref{appendix:thm:policy_improvement_condition} for policy improvement condition. 

\begin{mythm} \label{appendix:thm:policy_improvement_condition}
    Let $\pi_{new} := \pi_{\theta_{new}}$ be the solution of the problem of maximizing 
    \begin{align}
        L^{\pi_{old}}(\pi_\theta) - \tilde{C}_1 \max_s \KL(\pi_{old}(\cdot |s) ~\Vert~ \pi_{\theta}(\cdot |s)),
    \end{align}
    where 
    \begin{align}
        L^{\pi_{old}}(\pi_\theta) &= L^{\pi_{old}}_r(\pi_\theta) - \lambda L^{\pi_{old}}_{1-\epsilon_0}(\pi_\theta)\\
        &= \left( V^{\pi_{old}}(s_0) - \lambda q^{\pi_{old}}_{1-\epsilon_0}(s_0) \right) + \Ebb_{\pi_{old}}\left[ \sum^\infty_{t=0} \gamma^t  \Ebb_{a \sim \pi_{\theta}}\left[  A^{\pi_{old}}_r(s_t, a) - \lambda A^{\pi_{old}}_{1-\epsilon_0}(s_t, a) \right] \right] ,
    \end{align}
    and some constant $\tilde{C}_1 > 0$. Then, under deterministic dynamics $s_{t+1} = h(s_t, a_t)$ and Assumptions \ref{appendix:assm:boundness}, \ref{appendix:assm:smoothness}, and \ref{appendix:assm:additional_cost}, the following inequality holds:
    \begin{align}
        &L_{quant}(\pi_{new}, \lambda) - L_{quant}(\pi_{old}, \lambda) \\
        &\geq L^{\pi_{old}}(\pi_{new}) - L^{\pi_{old}}(\pi_{old}) - \tilde{C}_1 \KL_{max}(\pi_{old} || \pi_{new}) - \underbrace{\tilde{C}_2 \frac{\epsilon}{1-\epsilon}}_{\text{approximation loss}} 
    \end{align}
    for a given Lagrange multiplier $\lambda > 0$, some constant $\tilde{C}_2$ and small $\epsilon > 0$.
\end{mythm}

Remind that Theorem 1 of \cite{schulman2015trust} with our notation:
\begin{mythm}[Theorem 1 of \cite{schulman2015trust}] \label{appendix:thm:trpo_thm1}
    \begin{align}
        V^{\pi}(s_0) &\geq \underbrace{V^{\pi'}(s_0) + \Ebb_{\pi'}\left[ \sum^\infty_{t=0} \gamma^t  \Ebb_{a \sim \pi}\left[ A^{\pi'}_r(s_t, a) \right] \right]}_{=: L^{\pi'}_r(\pi)} - C_3 \max_s \KL(\pi'(\cdot |s) ~\Vert~ \pi(\cdot |s)) \\
        &= L^{\pi'}_r \left( \pi \right) - C_3 \max_s \KL(\pi'(\cdot |s) ~\Vert~ \pi(\cdot |s))    
    \end{align}
    where 
    \begin{align}
        C_3 &= \frac{4 \gamma \max_{s, a} \left\vert A^{\pi'}_r(s,a) \right\vert}{(1-\gamma)^2} \\
        A^{\pi'}_r(s,a) &:= r(s,a) + \gamma \Ebb_{s'}\left[ V^{\pi'}(s') \right] - V^{\pi'}(s)
    \end{align}
\end{mythm}
We omit the proof of Theorem \ref{appendix:thm:trpo_thm1}. Please see \cite{schulman2015trust} for the proof. Note that Theorem \ref{appendix:thm:trpo_thm1} holds for any two policies $\pi$ and $\pi'$ as we can see in the appendix of the original paper \cite{schulman2015trust}.

Finally now we prove Theorem \ref{appendix:thm:policy_improvement_condition}.
\begin{proof}[Proof of Theorem \ref{appendix:thm:policy_improvement_condition}]
    From Proposition \ref{appendix:prop:qpip_Lpipip_bound}, we have
    \begin{equation}
        q^{\pi}_u (s_0) \leq \underbrace{q^{\pi'}_u(s_0) + \Ebb_{\pi'}\left[ \sum^{\infty}_{t=0} \gamma^t \Ebb_{a \sim \pi}\left[ A^{\pi'}_u(s_t, a) \right] \right]}_{L^{\pi'}_u(\pi)} + C_1 \max_s \KL(\pi'(\cdot |s) ~\Vert~ \pi(\cdot |s)) + C_2 \frac{\epsilon}{1-\epsilon} \label{appendix:eq:qpip_Lpip_pi}
    \end{equation}
    for 
    \begingroup
    \allowdisplaybreaks
    \begin{align}
        A^{\pi'}_u(s,a) &:= c(s,a) + \tilde{c}^{\pi'}_u(s,a) + \gamma \Ebb_{s'} \left[ q^{\pi'}_u(s') \right] - q^{\pi'}_u(s) \\
        C_1 &= \left( \frac{4 \gamma \max_{s,a} \left\vert A^{\pi'}_u(s,a) \right\vert + \gamma R }{(1-\gamma)^2} + \frac{C_u}{1-\gamma} \right) \\
        C_2 &= \frac{R}{(1-\gamma)^2},
    \end{align}
    \endgroup
    and from Theorem 1 in \cite{schulman2015trust} (or Theorem \ref{appendix:thm:trpo_thm1} in this appendix), we have
    \begingroup
    \allowdisplaybreaks
    \begin{align}
        V^{\pi}(s_0) &\geq \underbrace{V^{\pi'}(s_0) + \Ebb_{\pi'}\left[ \sum^\infty_{t=0} \gamma^t  \Ebb_{a \sim \pi}\left[ A^{\pi'}_r(s_t, a) \right] \right]}_{=: L^{\pi'}_r(\pi)} - C_3 \max_s \KL(\pi'(\cdot |s) ~\Vert~ \pi(\cdot |s)) \label{appendix:eq:Vpip_Lpip_pi}
    \end{align}
    \endgroup
    where 
    \begingroup
    \allowdisplaybreaks
    \begin{align}
        A^{\pi'}_r(s,a) &:= r(s,a) + \gamma \Ebb_{s'}\left[ V^{\pi'}(s') \right] - V^{\pi'}(s) \\
        C_3 &= \frac{4 \gamma \max_{s, a} \left\vert A^{\pi'}_r(s,a) \right\vert}{(1-\gamma)^2} .
    \end{align}
    \endgroup
    For a given $\lambda > 0$, by subtracting $\lambda \times $ \eqref{appendix:eq:qpip_Lpip_pi} from \eqref{appendix:eq:Vpip_Lpip_pi}, then we have
    \begingroup
    \allowdisplaybreaks
    \begin{align}
        &V^{\pi}(s_0) - \lambda q^{\pi}_u (s_0) \nonumber \\
        &\geq \tcb{ L^{\pi'}_r(\pi) } - \lambda \tcr{ L^{\pi'}_u(\pi) } - ( \lambda C_1 + C_3) \max_s \KL(\pi'(\cdot |s) ~\Vert~ \pi(\cdot |s)) - \lambda C_2 \frac{\epsilon}{1-\epsilon} \\
        &= \tcb{ V^{\pi'}(s_0) + \Ebb_{\pi'}\left[ \sum^\infty_{t=0} \gamma^t  \Ebb_{a \sim \pi}\left[ A^{\pi'}_r(s_t, a) \right] \right] } - \lambda \tcr{ \left\{ q^{\pi'}_u(s_0) + \Ebb_{\pi'}\left[ \sum^{\infty}_{t=0} \gamma^t \Ebb_{a \sim \pi}\left[ A^{\pi'}_u(s_t, a) \right] \right] \right\} } \nonumber \\
        &~~~~~ -  (\lambda C_1 + C_3) \max_s \KL(\pi'(\cdot |s) ~\Vert~ \pi(\cdot |s)) - \lambda C_2 \frac{\epsilon}{1-\epsilon} \\
        &= \underbrace{\left( \tcb{ V^{\pi'}(s_0) } - \lambda \tcr{ q^{\pi'}_u(s_0) } \right) + \Ebb_{\pi'}\left[ \sum^\infty_{t=0} \gamma^t  \Ebb_{a \sim \pi}\left[  \tcb{ A^{\pi'}_r(s_t, a) } - \lambda \tcr{ A^{\pi'}_u(s_t, a) } \right] \right]}_{=: L^{\pi'}(\pi)} \\
        &~~~~~ -  (\lambda C_1 + C_3) \max_s \KL(\pi'(\cdot |s) ~\Vert~ \pi(\cdot |s)) - \lambda C_2 \frac{\epsilon}{1-\epsilon} \\
        &= L^{\pi'}(\pi) - (\lambda C_1 + C_3) \max_s \KL(\pi'(\cdot |s) ~\Vert~ \pi(\cdot |s)) - \lambda C_2 \frac{\epsilon}{1-\epsilon}
    \end{align}
    \endgroup
    Therefore now we have 
    \begingroup
    \allowdisplaybreaks
    \begin{align}
        L_{quant}(\pi, \lambda) &= V^{\pi}(s_0) - \lambda \left( q^{\pi}_u (s_0) - d_{th} \right) \\
        &\geq L^{\pi'}(\pi) + \lambda \cdot d_{th} - (\lambda C_1 + C_3) \max_s \KL(\pi'(\cdot |s) ~\Vert~ \pi(\cdot |s)) - \lambda C_2 \frac{\epsilon}{1-\epsilon} \label{appendix:eq:thm4_1} 
    \end{align}
    \endgroup
    Note that
    \begingroup
    \allowdisplaybreaks
    \begin{align}
        L^{\pi'}_r(\pi') &= V^{\pi'}(s_0) + \Ebb_{\pi'}\left[ \sum^\infty_{t=0} \gamma^t  \Ebb_{a' \sim \pi'}\left[ A^{\pi'}_r(s_t, a') \right] \right] \\
        &= V^{\pi'}(s_0) \\
        L^{\pi'}_u(\pi') &= q^{\pi'}_u(s_0) + \Ebb_{\pi'}\left[ \sum^{\infty}_{t=0} \gamma^t \Ebb_{a' \sim \pi'}\left[ A^{\pi'}_u(s_t, a') \right] \right]\\
        &\overset{(a)}{\leq} q^{\pi'}_u(s_0) + \frac{\epsilon R}{(1-\epsilon)(1-\gamma)}
    \end{align}
    \endgroup
    where (a) holds by \eqref{appendix:eq:qpi_Lpi}. Therefore,
    \begingroup
    \allowdisplaybreaks
    \begin{align}
        L^{\pi'}(\pi') &= L^{\pi'}_r(\pi') - \lambda L^{\pi'}_u(\pi') \\
        &\geq V^{\pi'}(s_0) - \lambda \left( q^{\pi'}_u(s_0) + \frac{\epsilon R}{(1-\epsilon)(1-\gamma)} \right) \\
        &= V^{\pi'}(s_0) - \lambda \left( q^{\pi'}_u(s_0) - d_{th}  \right) - \lambda \left(d_{th} + \frac{\epsilon R}{(1-\epsilon)(1-\gamma)} \right) \\
        &= L_{quant}(\pi', \lambda) - \lambda \cdot d_{th} - \lambda \frac{\epsilon R}{(1-\epsilon)(1-\gamma)}. 
    \end{align}
    \endgroup
    By rearranging this, we get
    \begin{equation}
        - L_{quant}(\pi', \lambda) \geq - L^{\pi'}(\pi') - \lambda \cdot d_{th} - \lambda \frac{\epsilon R}{(1-\epsilon)(1-\gamma)}. \label{appendix:eq:thm4_2}
    \end{equation}
    Therefore by adding \eqref{appendix:eq:thm4_2} and \eqref{appendix:eq:thm4_1}, we can conclude
    \begingroup
    \allowdisplaybreaks
    \begin{align}
        L_{quant}(\pi, \lambda) - L_{quant}(\pi', \lambda) &\geq L^{\pi'}(\pi) - L^{\pi'}(\pi') - (\lambda C_1 + C_3) \max_s \KL(\pi'(\cdot |s) ~\Vert~ \pi(\cdot |s)) \\
        &~~~~~ - \lambda C_2 \frac{\epsilon}{1-\epsilon} - \lambda \frac{\epsilon R}{(1-\epsilon)(1-\gamma)} \\
        &= L^{\pi'}(\pi) - L^{\pi'}(\pi') - \tilde{C}_1 \max_s \KL(\pi'(\cdot |s) ~\Vert~ \pi(\cdot |s)) - \tilde{C}_2 \frac{\epsilon}{1-\epsilon}
    \end{align}
    \endgroup
    where
    \begingroup
    \allowdisplaybreaks
    \begin{align}
        \tilde{C}_1 &= \lambda C_1 + C_3 \\
        &= \lambda \left( \frac{4 \gamma \max_{s,a} \left\vert A^{\pi'}_u(s,a) \right\vert + \gamma R }{(1-\gamma)^2} + \frac{C_u}{1-\gamma} \right) + \frac{4 \gamma \max_{s, a} \left\vert A^{\pi'}_r(s,a) \right\vert}{(1-\gamma)^2} \\
        &=  \frac{4 \gamma \left( \max_{s, a} \left\vert A^{\pi'}_r(s,a) \right\vert + \lambda \max_{s,a} \left\vert A^{\pi'}_u(s,a) \right\vert \right) }{(1-\gamma)^2} + \lambda \frac{\gamma R}{(1-\gamma)^2} + \lambda \frac{C_u}{1-\gamma}\\
        \tilde{C}_2 &= \lambda C_2 + \lambda \frac{R}{1-\gamma} \\
        &= \lambda \left( \frac{R}{(1-\gamma)^2} + \frac{R}{1-\gamma} \right) \\
        &= \lambda \frac{(2-\gamma) R}{(1-\gamma)^2}
    \end{align}
    \endgroup
\end{proof}

\newpage
\section{Detailed Explanation of The Environments}\label{appendix:environments}

\begin{figure}[b]
    \centering
    \begin{subfigure}[b]{0.45\columnwidth}
        \centering
        \includegraphics[width=\textwidth]{figures/SimpleButtonEnv.png}
        \caption{SimpleButtonEnv}
        \label{appendix:fig:SimpleButtonEnv}
    \end{subfigure}
    \begin{subfigure}[b]{0.45\columnwidth}
        \centering
        \includegraphics[width=\textwidth]{figures/DynamicEnv.png}
        \caption{DynamicEnv}
        \label{appendix:fig:DynamicEnv}
    \end{subfigure}
    \begin{subfigure}[b]{0.45\columnwidth}
        \centering
        \includegraphics[width=\textwidth]{figures/GremlinEnv.png}
        \caption{GremlinEnv}
        \label{appendix:fig:GremlinEnv}
    \end{subfigure}
    \begin{subfigure}[b]{0.45\columnwidth}
        \centering
        \includegraphics[width=\textwidth]{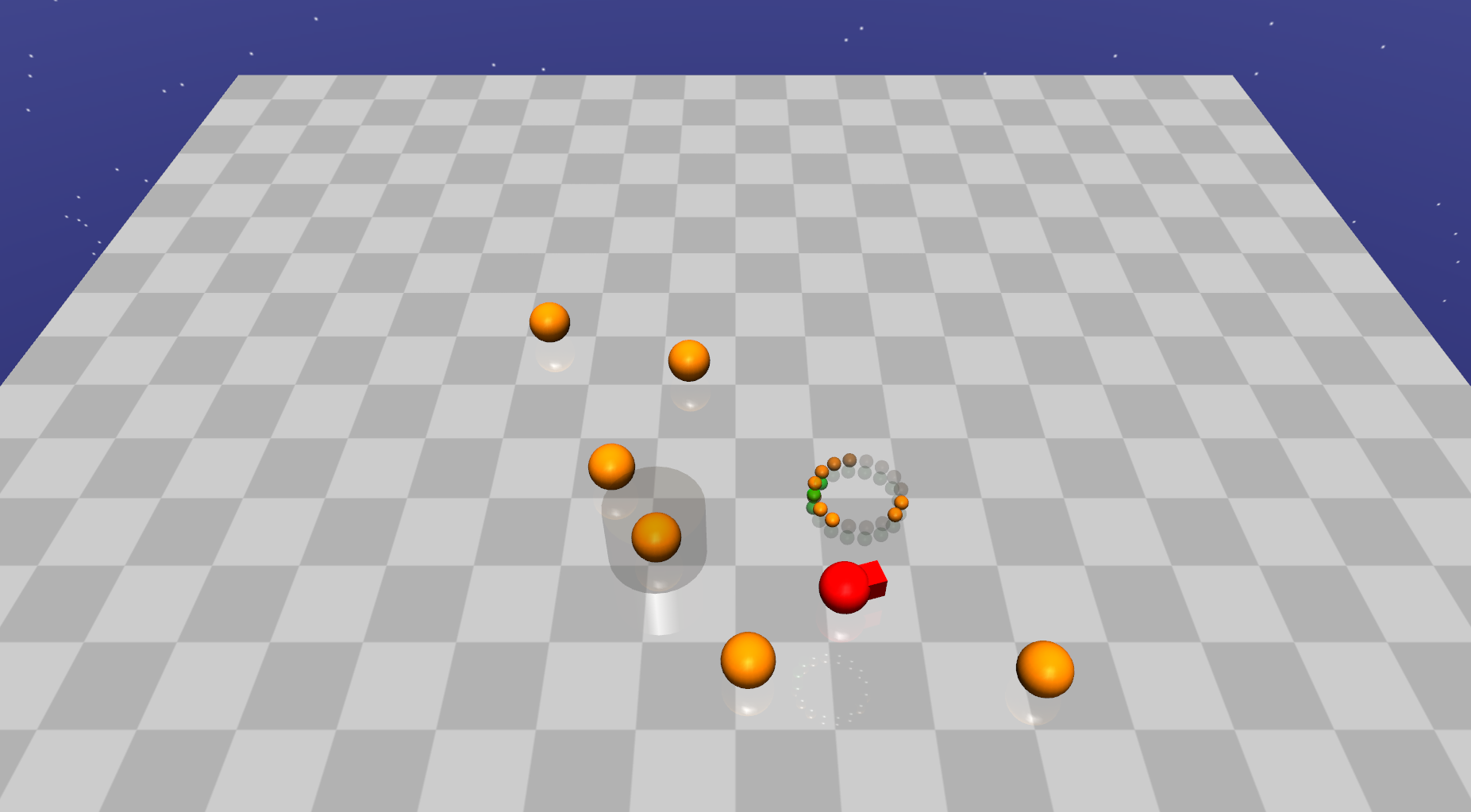}
        \caption{DynamicButtonEnv}
        \label{appendix:fig:DynamicButtonEnv}
    \end{subfigure}
    \caption{The considered environments}
    \label{appendix:fig:environments}
\end{figure}

The considered environments are SimpleButtonEnv, DynamicEnv \cite{yang2021wcsac}, GremlinEnv, and DynamicButtonEnv, which are based on Safety Gym \cite{ray2019benchmarking}, MuJoCo \cite{todorov2012mujoco}, and OpenAI Gym \cite{openaigym}. The experiments are performed on a server with Intel(R) Xeon(R) Gold 6240R CPU @2.40GHz, and each experiment takes 8 $\sim$ 10 hours. The environments are illustrated in Fig. \ref{appendix:fig:environments}. The goal of these environments is for a robot (red sphere) to reach a goal (the orange sphere wrapped by a grey translucent pillar for SimpleButtonEnv and DynamicButtonEnv, and the green pillar for DynamicEnv and GremlinEnv), while avoiding hazards (blue circles) or the non-goal button (the orange sphere). Once the robot reaches the current goal, the environments generate the next goal deterministically (SimpleButtonEnv) or randomly (DynamicEnv, GremlinEnv, DynamicButtonEnv). When the robot performs an action at time step $t$, it receives a reward $\left\{ \left\Vert p_{t+1} - p_{\text{goal}} \right\Vert_2 - \left\Vert p_{t} - p_{\text{goal}} \right\Vert_2  \right\} + 1_{\text{goal reached}}$, where $p_{t}$ is the position ($x$, $y$) of the robot at time step $t$ and $p_{\text{goal}}$ is the current goal position at time step $t$. It also receives a cost $+1$ if the robot touches non-goal objects (a hazard or the non-goal button), and $0$ otherwise. Hence, for the robot, it receives a higher return when the robot touches more goals in a maximum timesteps $T=1000$, and causes a higher sum of costs when the robot touches the other objects more often.

{\bfseries SimpleButtonEnv:} This environment consists of a robot (the red sphere), three hazards (blue pillars), a goal button (the orange sphere wrapped by a grey translucent pillar), and a non-goal button (the orange sphere). When it starts a new episode, it locates the robot randomly in in a restricted region $[x_{min}, x_{max}, y_{min}, y_{max}] = [-1.5, 1.5, -1.5, 1.5]$ and the other objects in a fixed position. When the robot reaches the current goal, it sets the next goal as the non-goal button. Thus, the objective of this environment is to touch two buttons many times iteratively in a fixed maximum timesteps.  

{\bfseries DynamicEnv:} This environment consists of a robot (the red sphere), three hazards (blue pillars), and a goal (the green pillar). When it starts a new episode, it locates these objects randomly in a restricted region $[x_{min}, x_{max}, y_{min}, y_{max}] = [-1.5, 1.5, -1.5, 1.5]$. When the robot reaches the current goal, the next goal is generated at a random position. 

{\bfseries GremlinEnv:} This environment consists of a robot (the red sphere), five hazards (blue pillars), three gremlins (purple moving cubes), and a goal (the green pillar). This is similar to DynamicEnv except the gremlins and higher complexity of the task. Each gremlin goes around in a circle, and when the agent touches a gremiln, it receives a cost.  When it starts a new episode, it locates these objects randomly in a restricted region $[x_{min}, x_{max}, y_{min}, y_{max}] = [-2, 2, -2, 2]$. When the robot reaches the current goal,  the next goal is generated at a random position.  

{\bfseries DynamicButtonEnv:} This environment consists of a robot (the red sphere), and goal button (the orange sphere wrapped by a grey translucent pillar), and five non-goal buttons (the orange sphere). When it starts a new episode, it locates these objects randomly in a restricted region $[x_{min}, x_{max}, y_{min}, y_{max}] = [-1.5, 1.5, -1.5, 1.5]$. When the robot reaches the current goal, it sets the next goal randomly among non-goal buttons. This environment is similar to DynamicEnv but the hazards are the non-goal buttons.

{\bfseries Observation Space:} The observation in these environments is  sensor values (accelerometer, velocimeter, gyro, and magnetometer) plus lidar values which measure the distance between the robot and the other objects. There are $16$ lidar sensors for each object (a goal, hazards, buttons, gremlins) and these are located around the robot. Each lidar sensor for an object measures the distance between the robot and the object located in its corresponding direction. Gathering all these sensor values, the environment gives these values to the agent as an observation at the current time. The dimensions of the observation spaces are $44$ (DynamicEnv, DynamicButtonEnv) and $60$ (SimpleButtonEnv and GremlinEnv).

\newpage
\section{More Results}\label{appendix:more_results}

\subsection{QCPO with Various Target Outage Probability $\epsilon_0$}

\begin{figure}[ht!]
    \centering
    \begin{subfigure}[b]{0.35\textwidth}
        \centering
        \includegraphics[width=\textwidth]{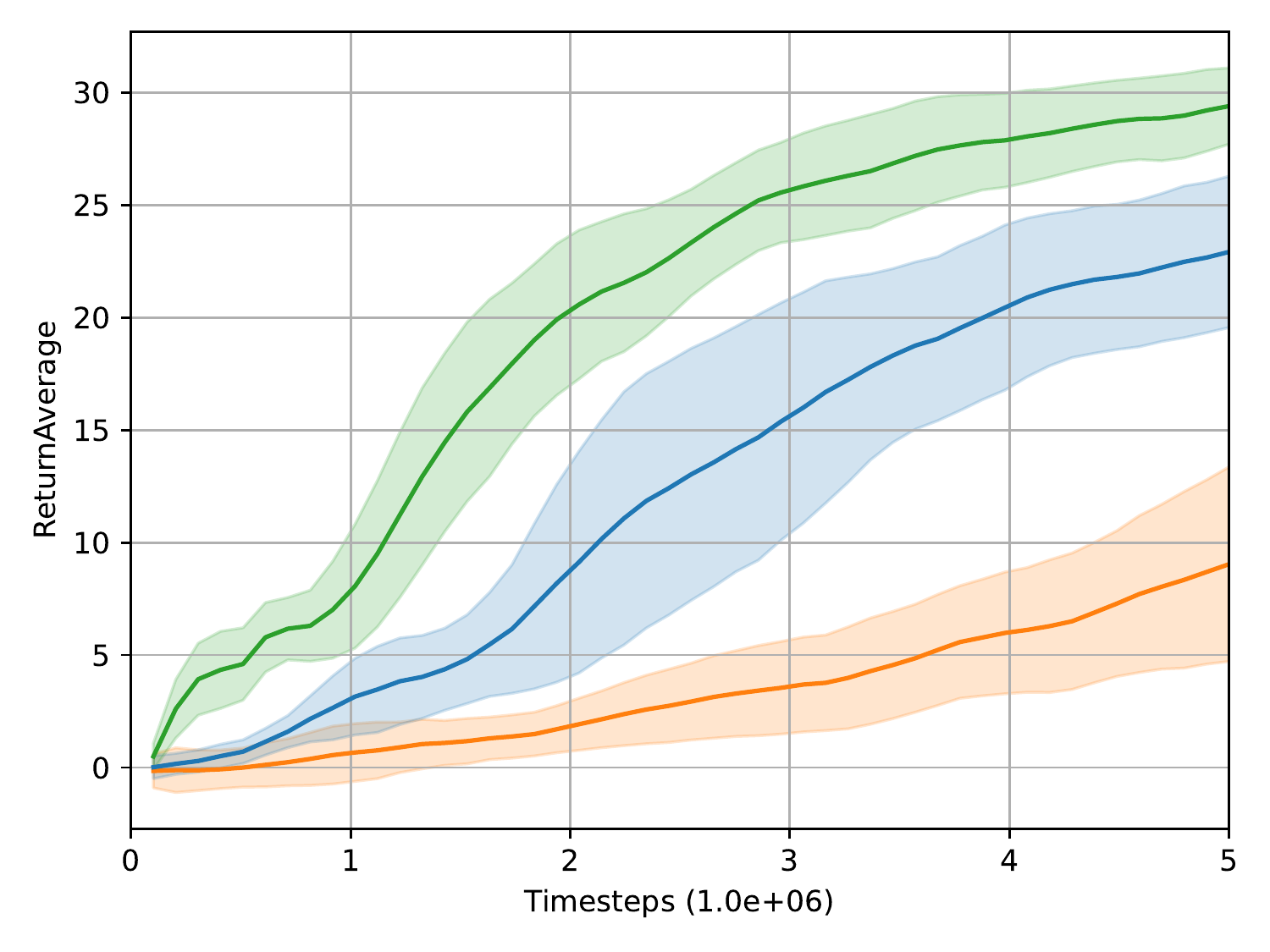}
        \caption{Average Return}
        \label{appendix:fig:SimpleButtonEnv_Return_05_02_01}
    \end{subfigure}
    \begin{subfigure}[b]{0.35\textwidth}
        \centering
        \includegraphics[width=\textwidth]{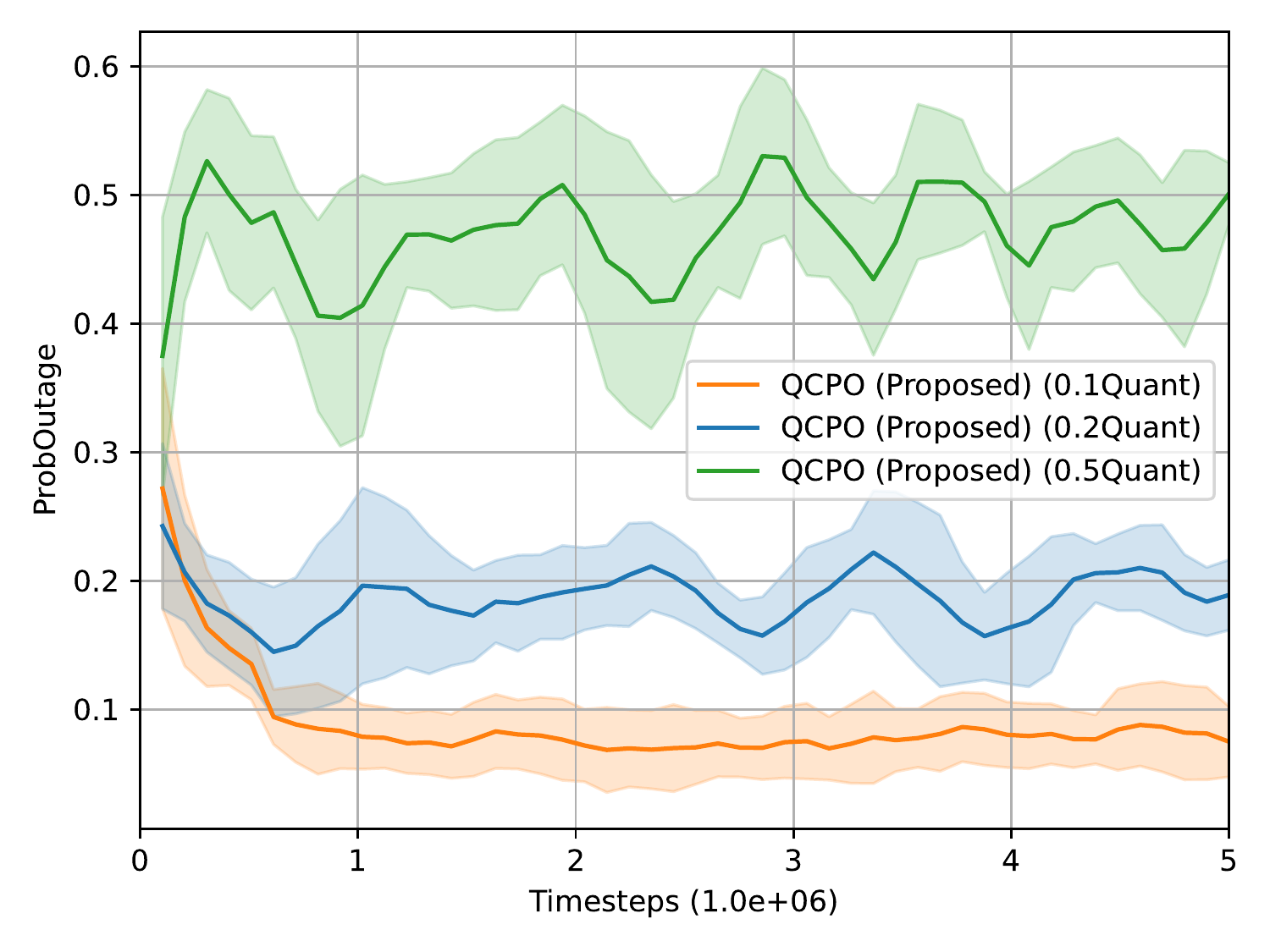}
        \caption{Outage Probability}
        \label{appendix:fig:SimpleButtonEnv_ProbOutage_05_02_01}
    \end{subfigure}
    
    \begin{subfigure}[b]{0.35\textwidth}
        \centering
        \includegraphics[width=\textwidth]{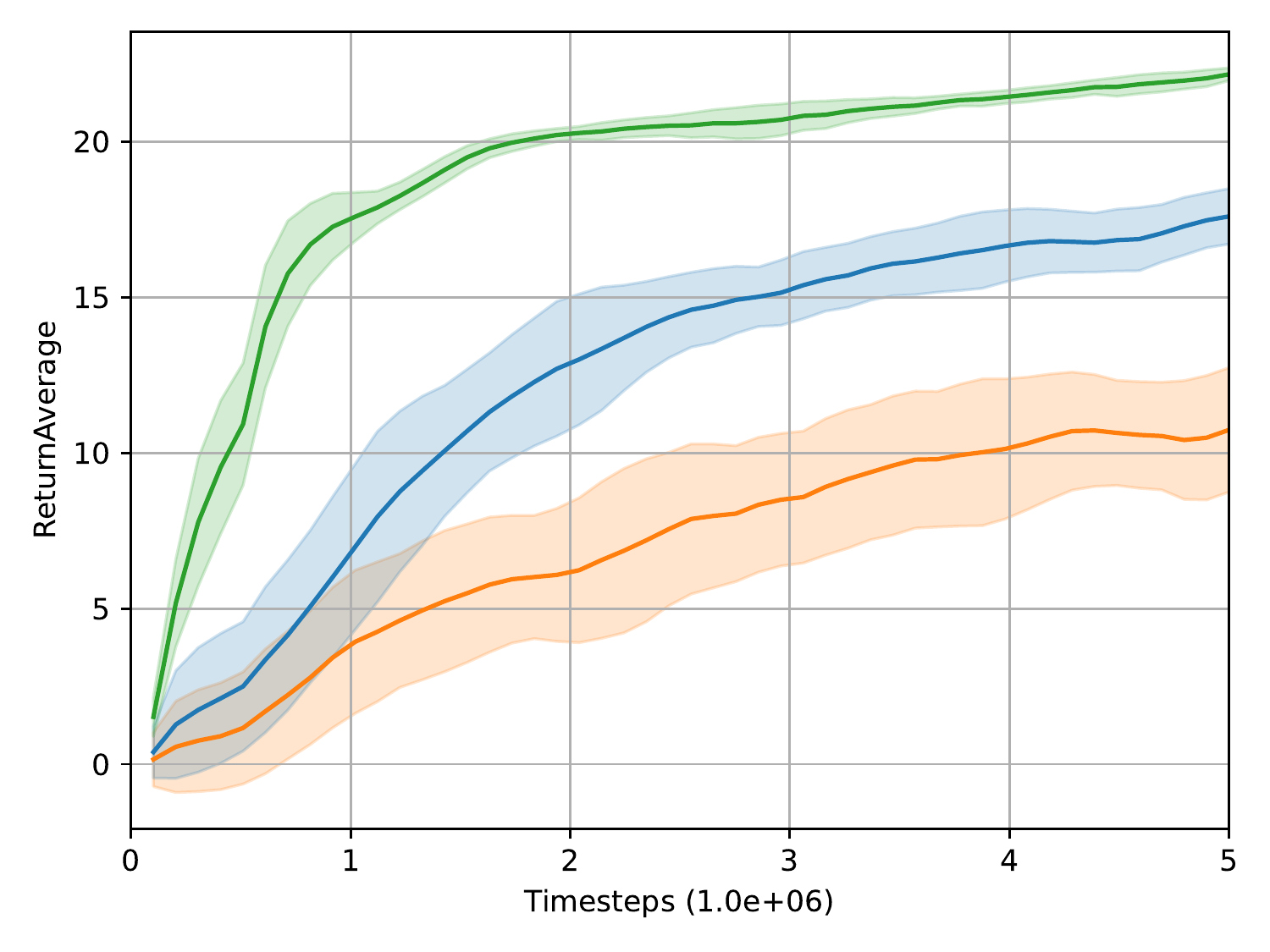}
        \caption{Average Return}
        \label{appendix:fig:DynamicEnv_Return_05_02_01}
    \end{subfigure}
    \begin{subfigure}[b]{0.35\textwidth}
        \centering
        \includegraphics[width=\textwidth]{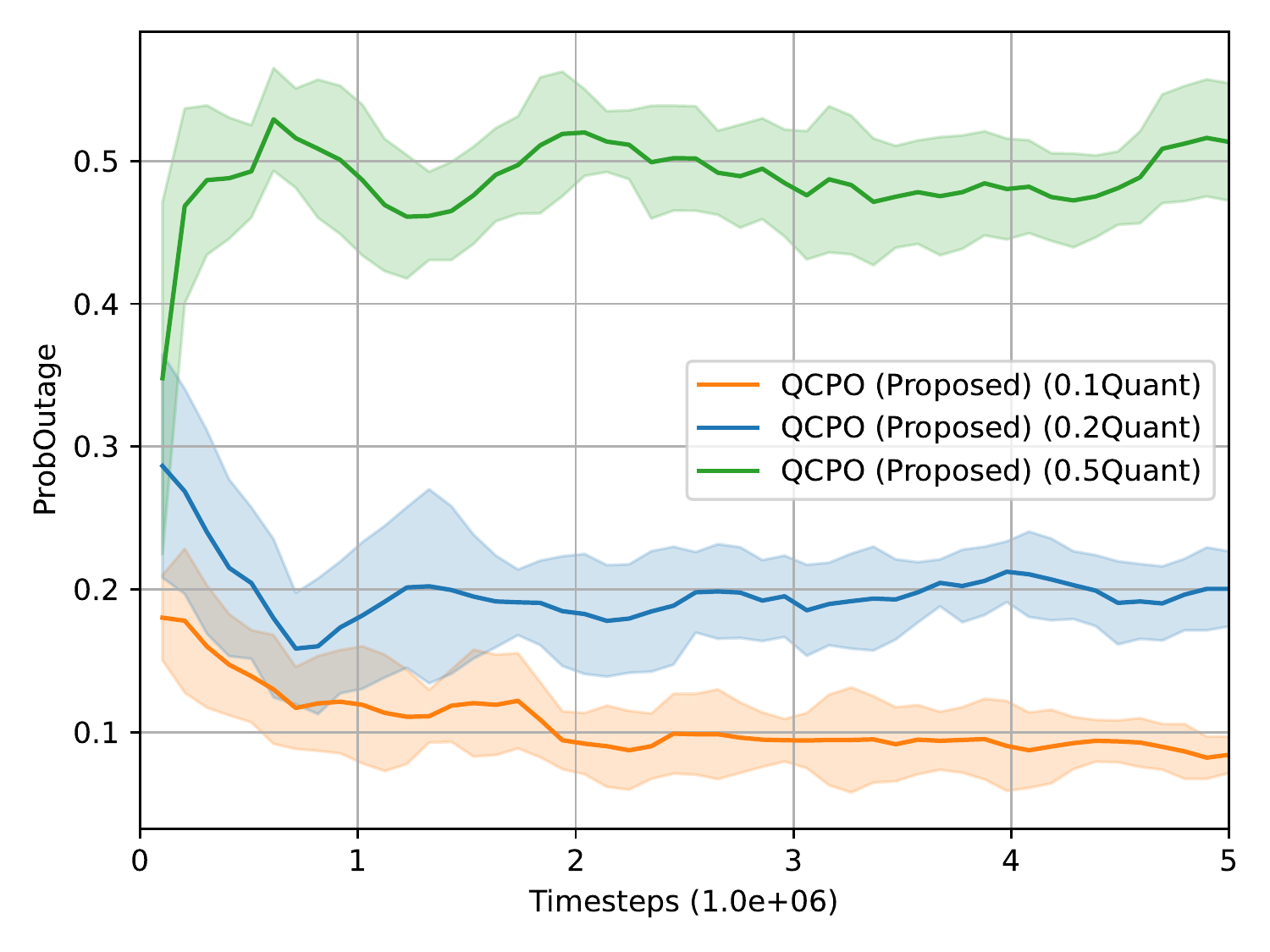}
        \caption{Outage Probability}
        \label{appendix:fig:DynamicEnv_ProbOutage_05_02_01}
    \end{subfigure}
    \caption{Results of QCPO with $\epsilon_0 = 0.5$ (green), $0.2$ (blue) and $0.1$ (orange) on SimpleButtonEnv (1st row), DynamicEnv (2nd row): (left)  average return of the most current 100 episodes and (right) outage probability of the most current 100 episodes.}
    \label{appendix:fig:result_qcpo_epsilon_05_02_01}
\end{figure}

\begin{figure}[ht!]
    \centering
    \begin{subfigure}[b]{0.35\textwidth}
        \centering
        \includegraphics[width=\textwidth]{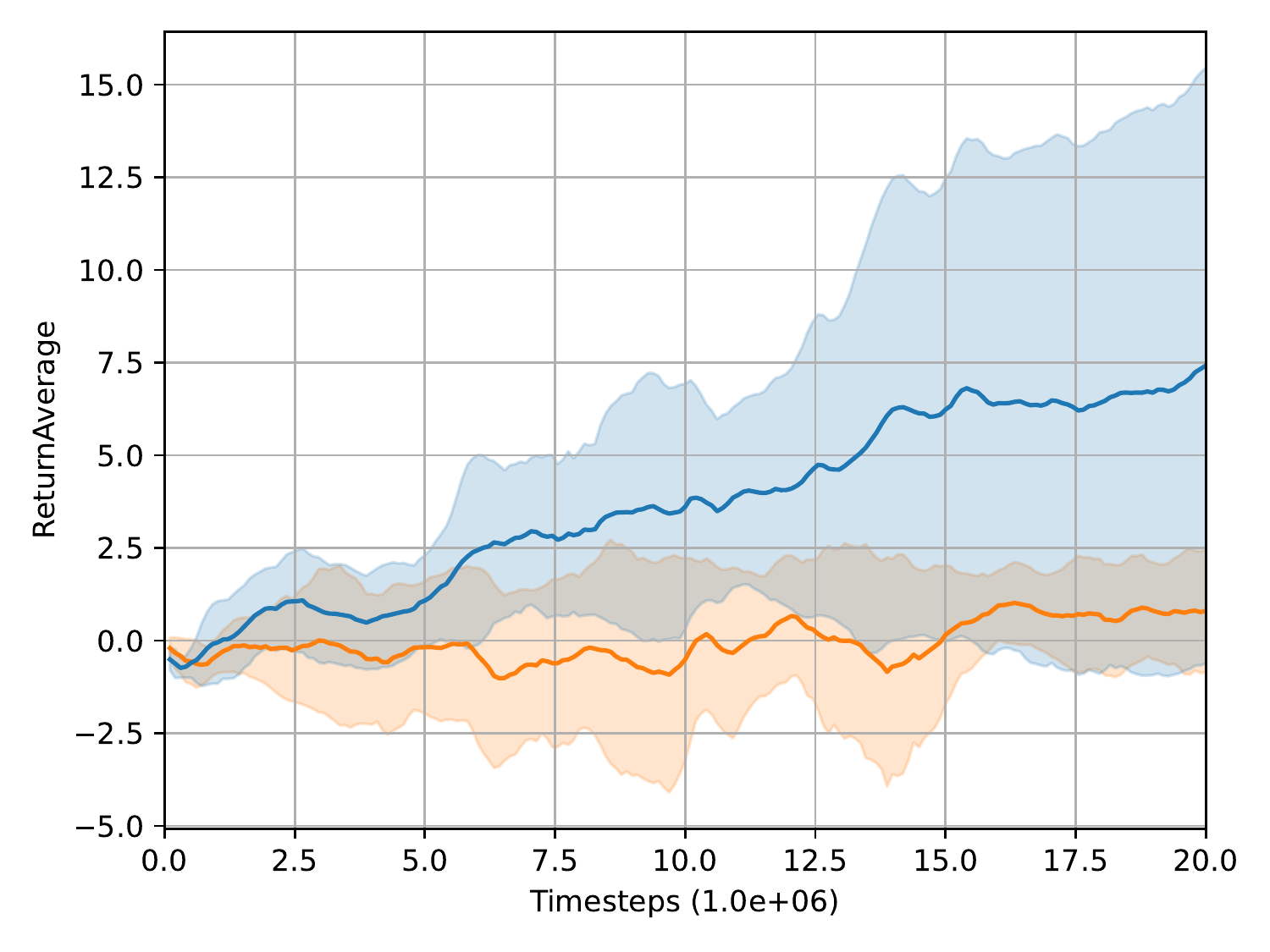}
        \caption{Average Return}
        \label{appendix:fig:SimpleButtonEnv_Return_005_002}
    \end{subfigure}
    \begin{subfigure}[b]{0.35\textwidth}
        \centering
        \includegraphics[width=\textwidth]{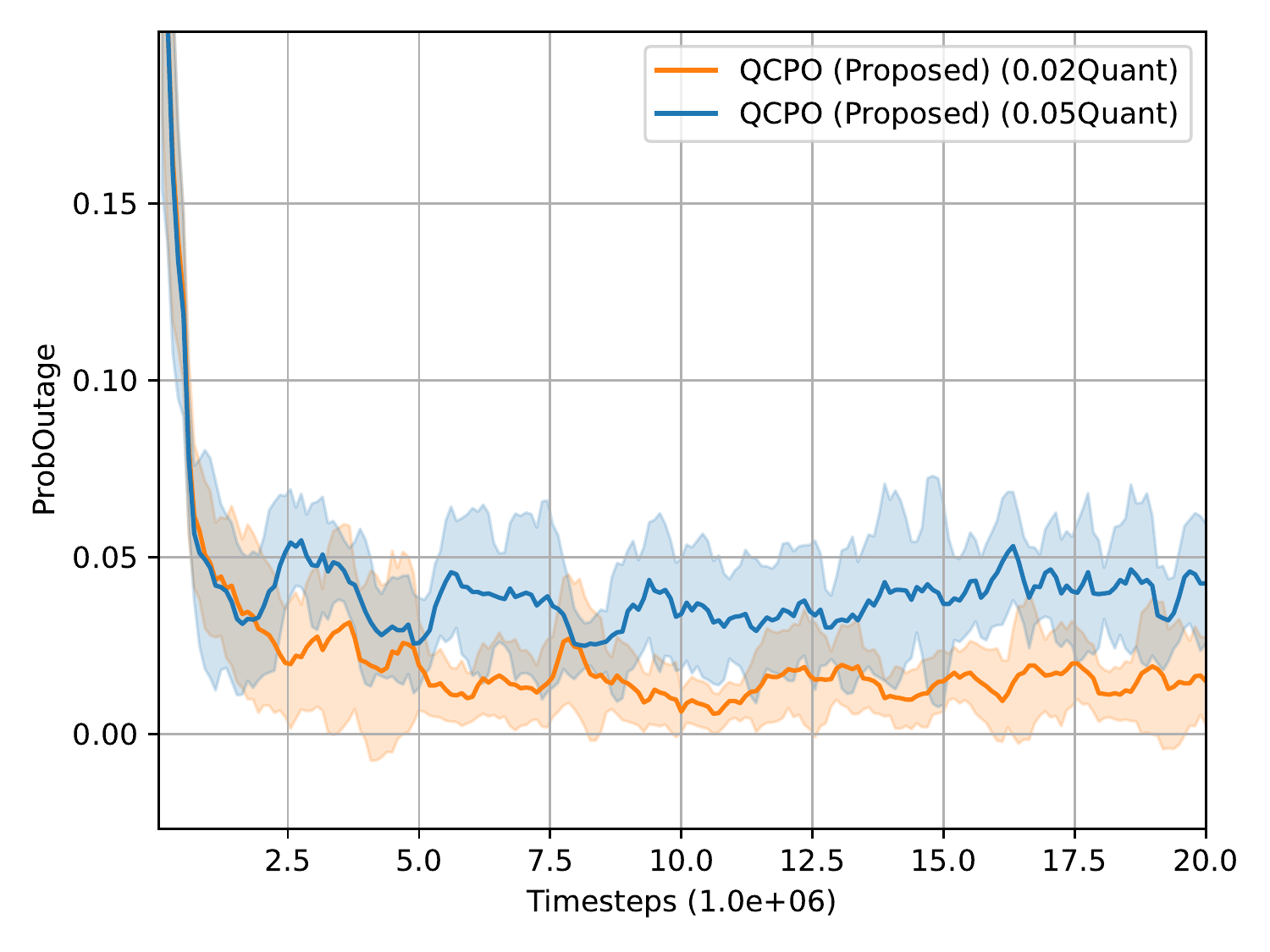}
        \caption{Outage Probability}
        \label{appendix:fig:SimpleButtonEnv_ProbOutage_005_002}
    \end{subfigure}
    
    \begin{subfigure}[b]{0.35\textwidth}
        \centering
        \includegraphics[width=\textwidth]{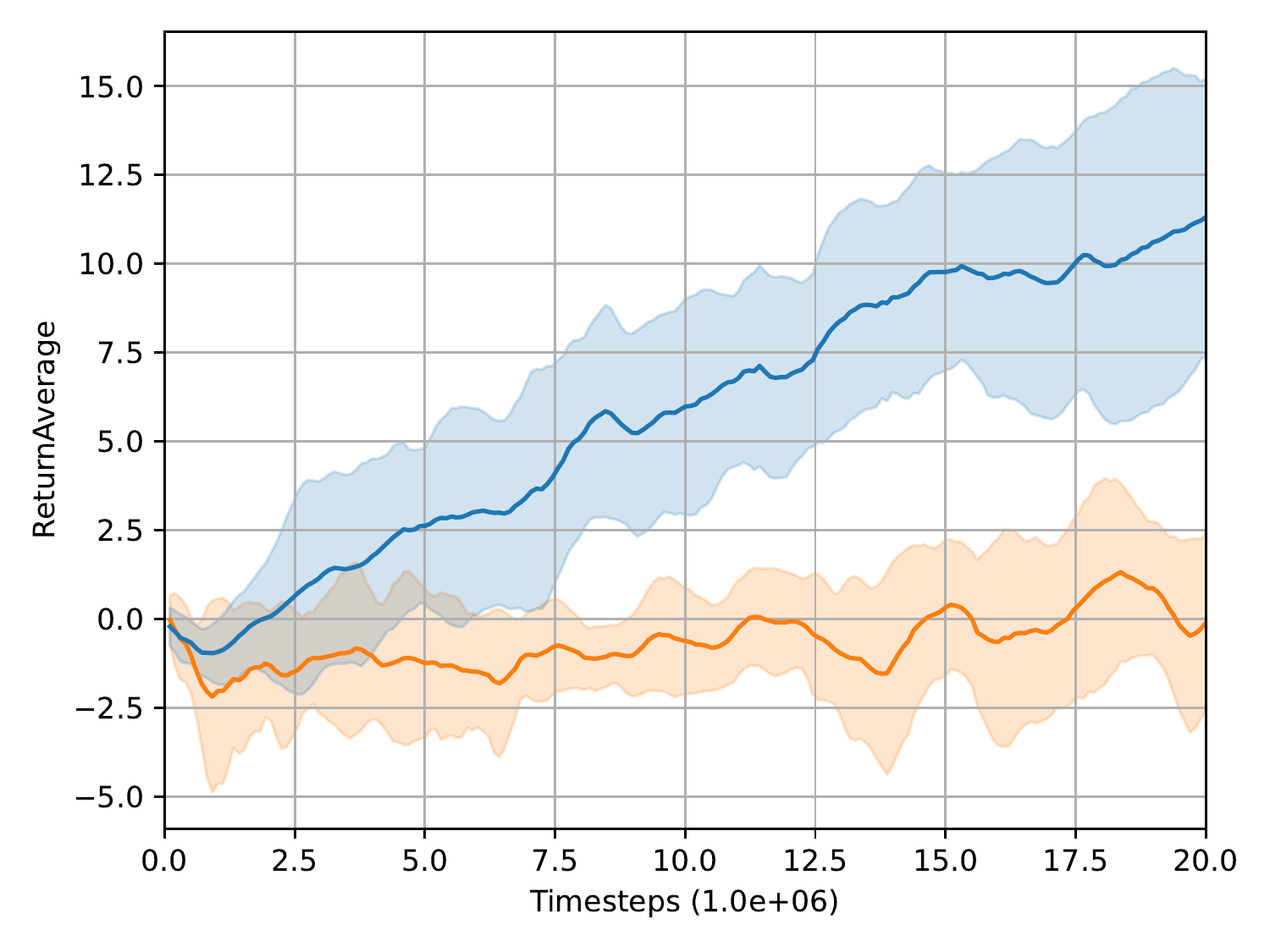}
        \caption{Average Return}
        \label{appendix:fig:DynamicEnv_Return_005_002}
    \end{subfigure}
    \begin{subfigure}[b]{0.35\textwidth}
        \centering
        \includegraphics[width=\textwidth]{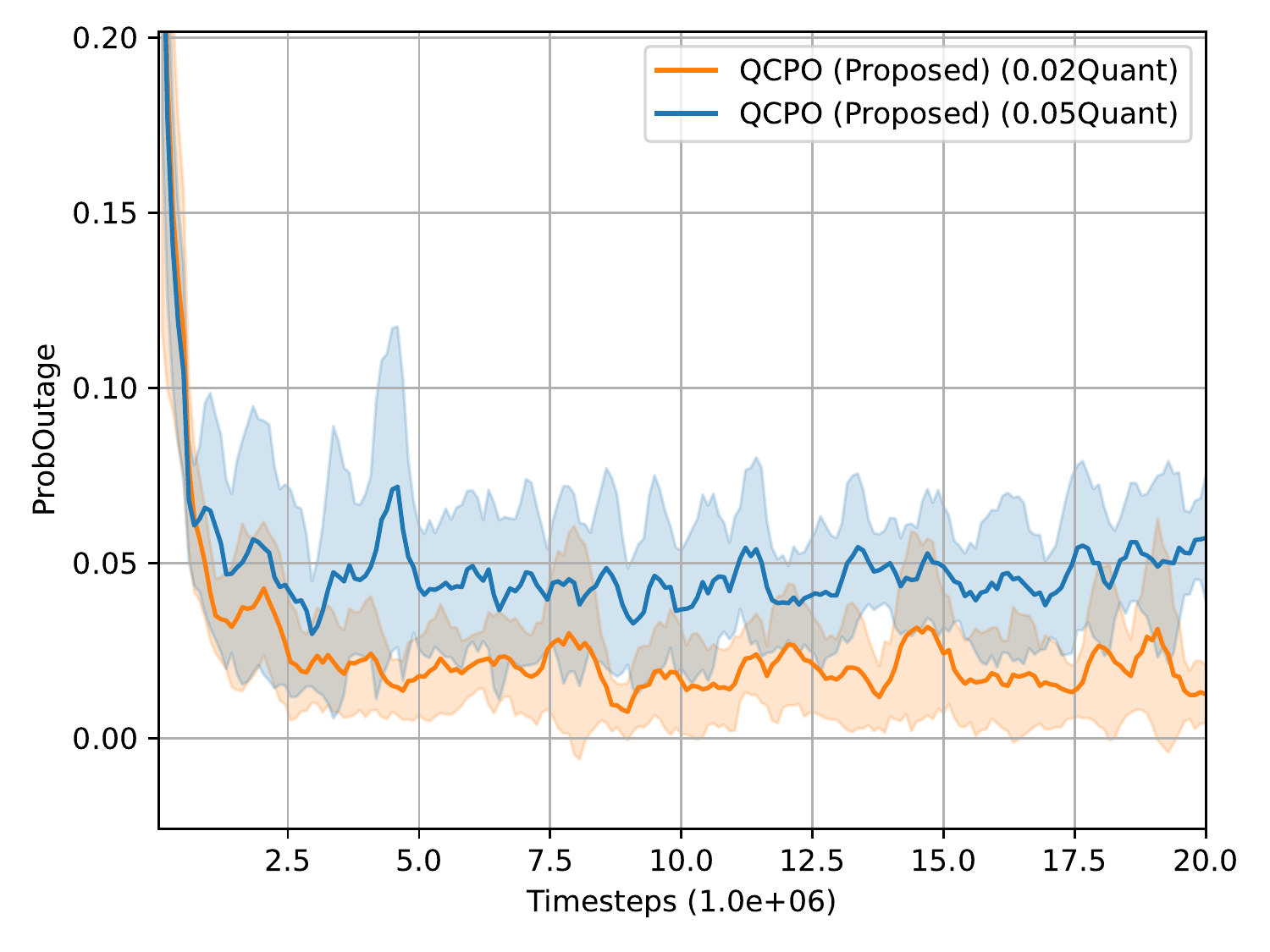}
        \caption{Outage Probability}
        \label{appendix:fig:DynamicEnv_ProbOutage_005_002}
    \end{subfigure}
    \caption{ Results of QCPO with $\epsilon_0 = 0.05$ (blue) and $0.02$ (orange) on SimpleButtonEnv (1st row) and DynamicEnv (2nd row): (left)  average return of the most current 100 episodes and (right) outage probability of the most current 100 episodes. }
    \label{appendix:fig:result_qcpo_epsilon_005_002}
\end{figure}

In this subsection, we provide results of QCPO with various target outage probabilities $\epsilon_0 = 0.5, 0.2, 0.1, 0.05$ and $0.02$. Fig. \ref{appendix:fig:result_qcpo_epsilon_05_02_01} shows the average return and the outage probability of QCPO with $\epsilon_0 = 0.5, 0.2$ and $0.1$. It is seen that QCPO satisfies the outage probability constraint after some initial time and then tries to increase the return while satisfying the outage probability constraint.
Fig. \ref{appendix:fig:result_qcpo_epsilon_005_002}  shows the average return and the outage probability of QCPO with  $\epsilon_0 = 0.05$ and $0.02$. In Fig. \ref{appendix:fig:result_qcpo_epsilon_005_002}, it is again seen that QCPO satisfies the outage probability constraint after some initial time and then tries to increase the return while satisfying the outage probability constraint.
However, it seems that more initial time steps are required than in the case of $\epsilon_0 = 0.5, 0.2$ and $0.1$ to satisfy the target outage probability.

\subsection{WCSAC with Weibull distribution approximation}

In this subsection, we provide results of QCPO, WCSAC\cite{yang2021wcsac}, and WCSAC with Weibull distribution approximation. 
In Fig. \ref{appendix:fig:WCSAC_Weibull_DynamicEnv_02_01}, it is seen that WCSAC with Weibull distribution approximation satisfies the outage probability constraint, while the original WCSAC with Gaussian distribution approximation does not. These results can imply that Weibull distribution approximation can estimate the true underlying distribution of the cumulative sum cost better than Gaussian distribution, and this is due to the limited capability of Gaussian distribution to capture the decay rate of the tail probability.

\begin{figure}[ht!]
    \centering
    \begin{subfigure}[b]{0.35\textwidth}
        \centering
        \includegraphics[width=\textwidth]{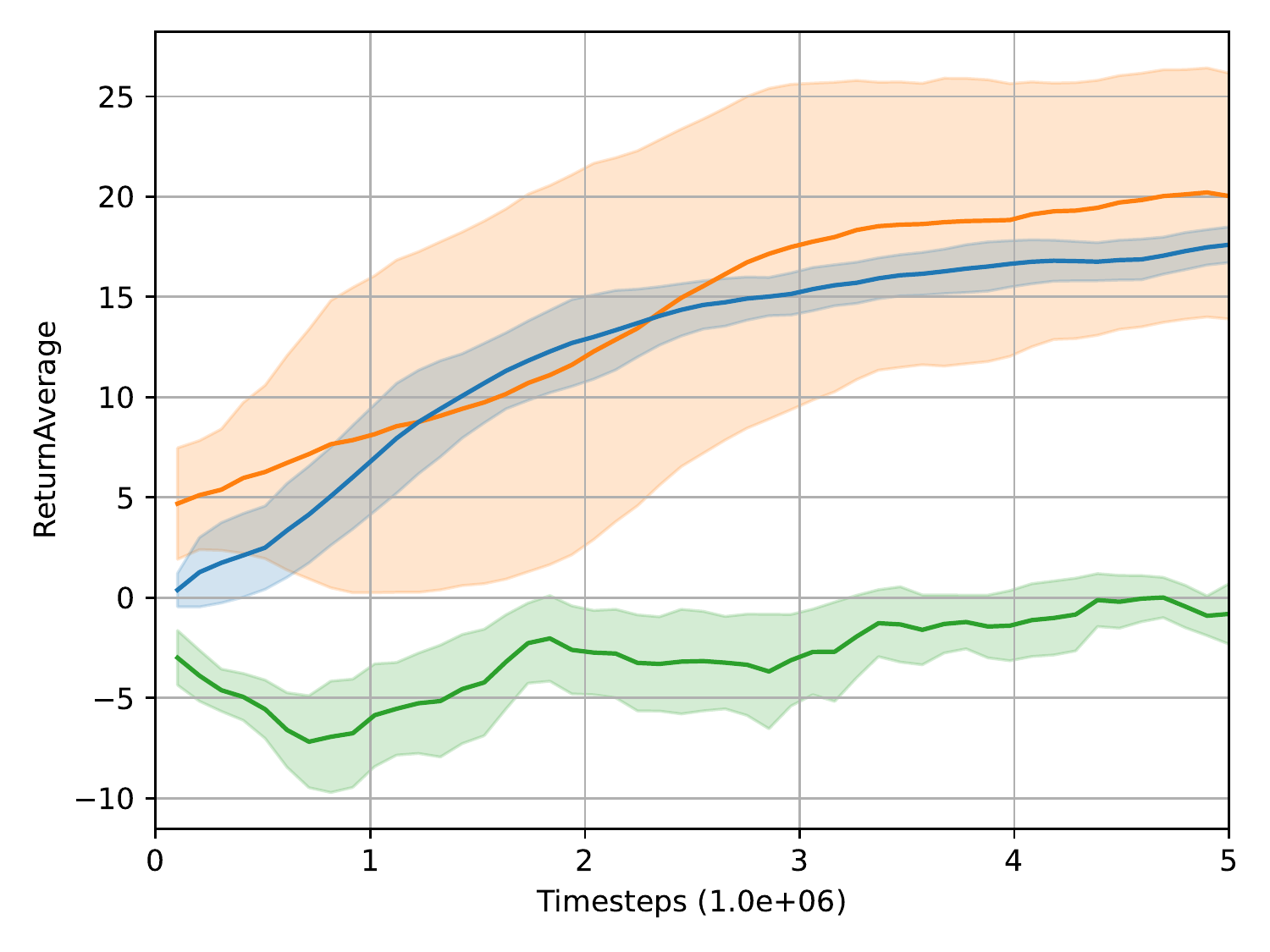}
        \caption{Average Return}
        \label{appendix:fig:WCSAC_Weibull_DynamicEnv_Return_02}
    \end{subfigure}
    \begin{subfigure}[b]{0.35\textwidth}
        \centering
        \includegraphics[width=\textwidth]{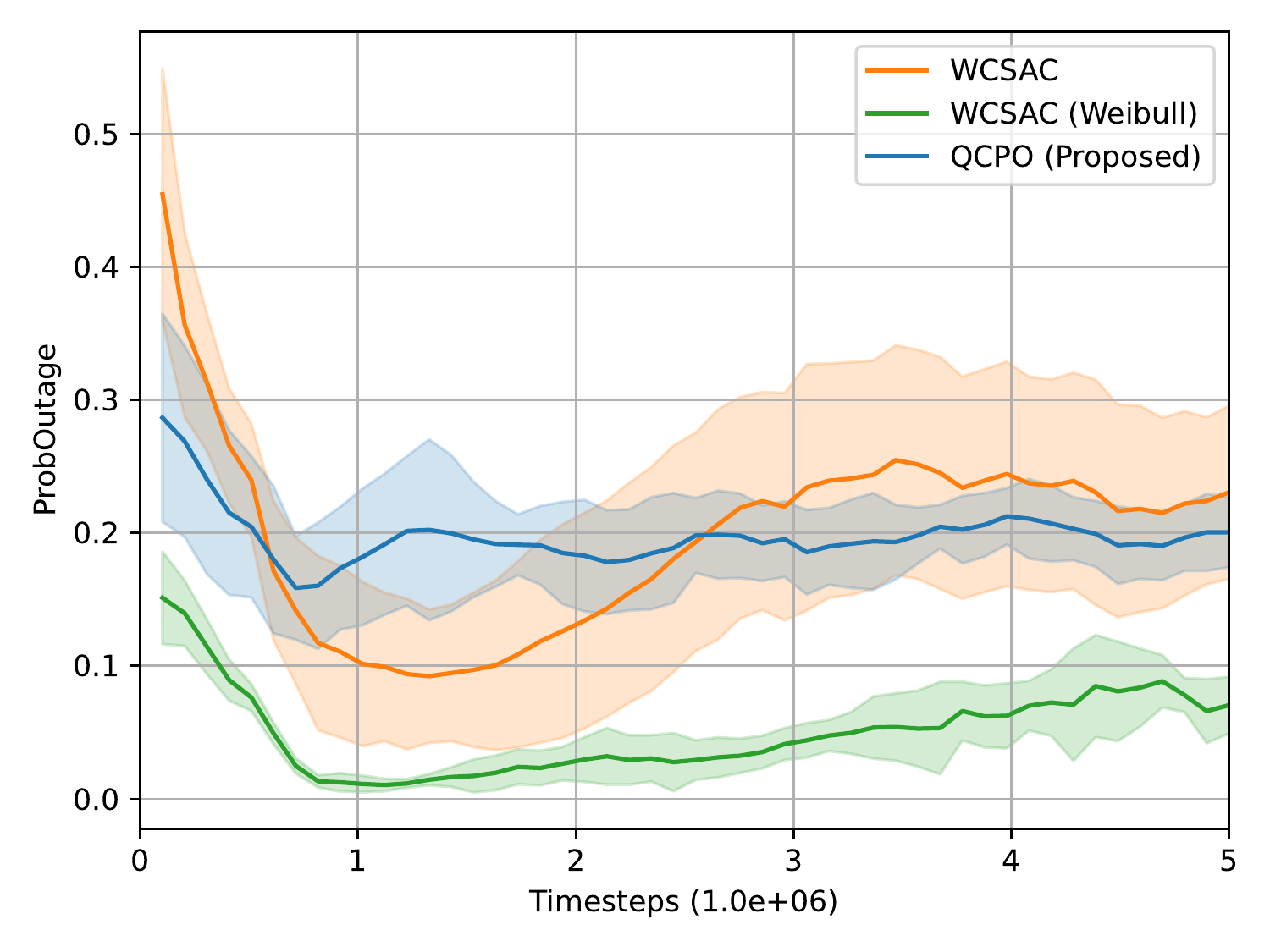}
        \caption{Outage Probability}
        \label{appendix:fig:WCSAC_Weibull_DynamicEnv_ProbOutage_02}
    \end{subfigure}
    
    \begin{subfigure}[b]{0.35\textwidth}
        \centering
        \includegraphics[width=\textwidth]{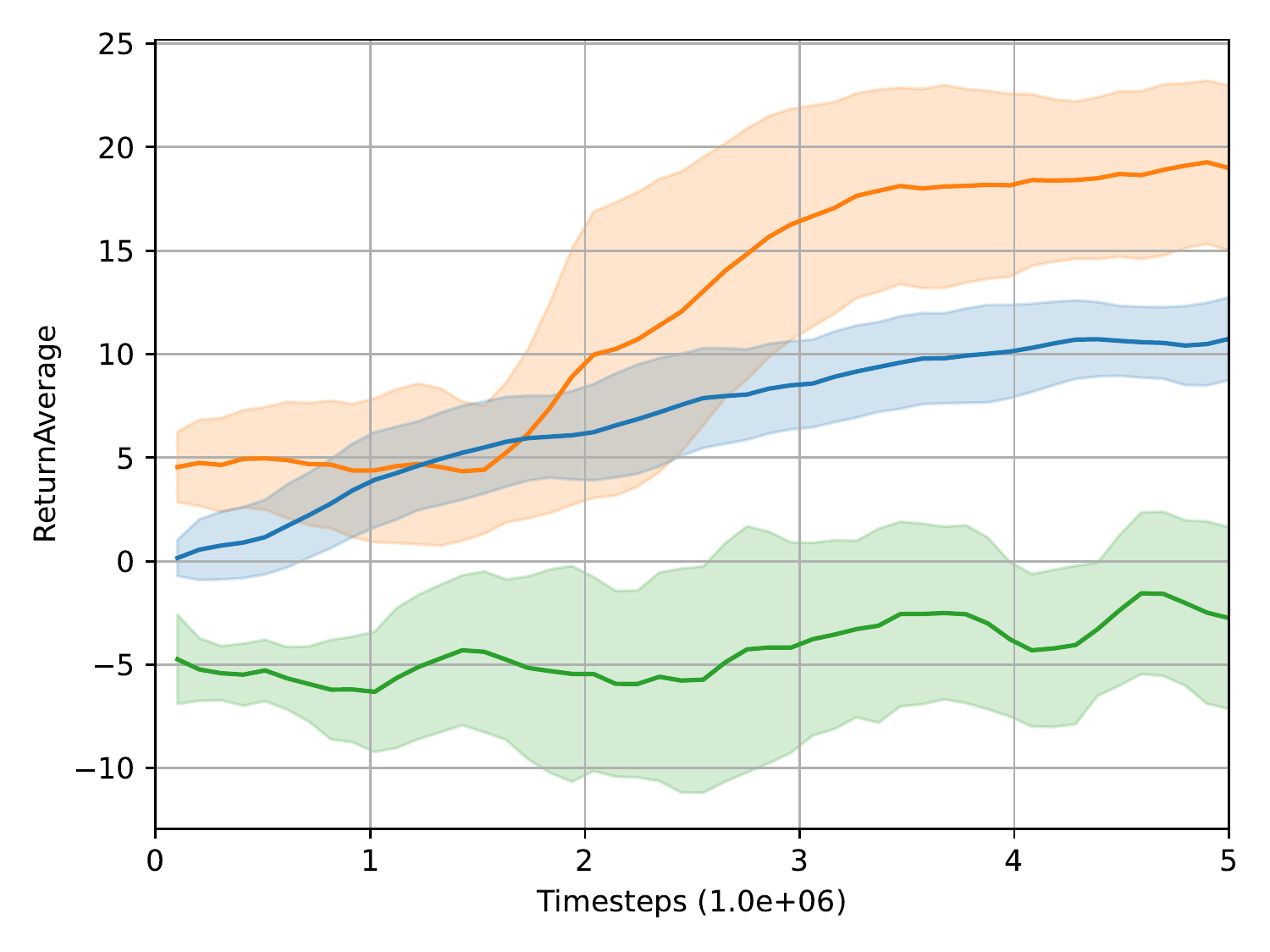}
        \caption{Average Return}
        \label{appendix:fig:WCSAC_Weibull_DynamicEnv_Return_01}
    \end{subfigure}
    \begin{subfigure}[b]{0.35\textwidth}
        \centering
        \includegraphics[width=\textwidth]{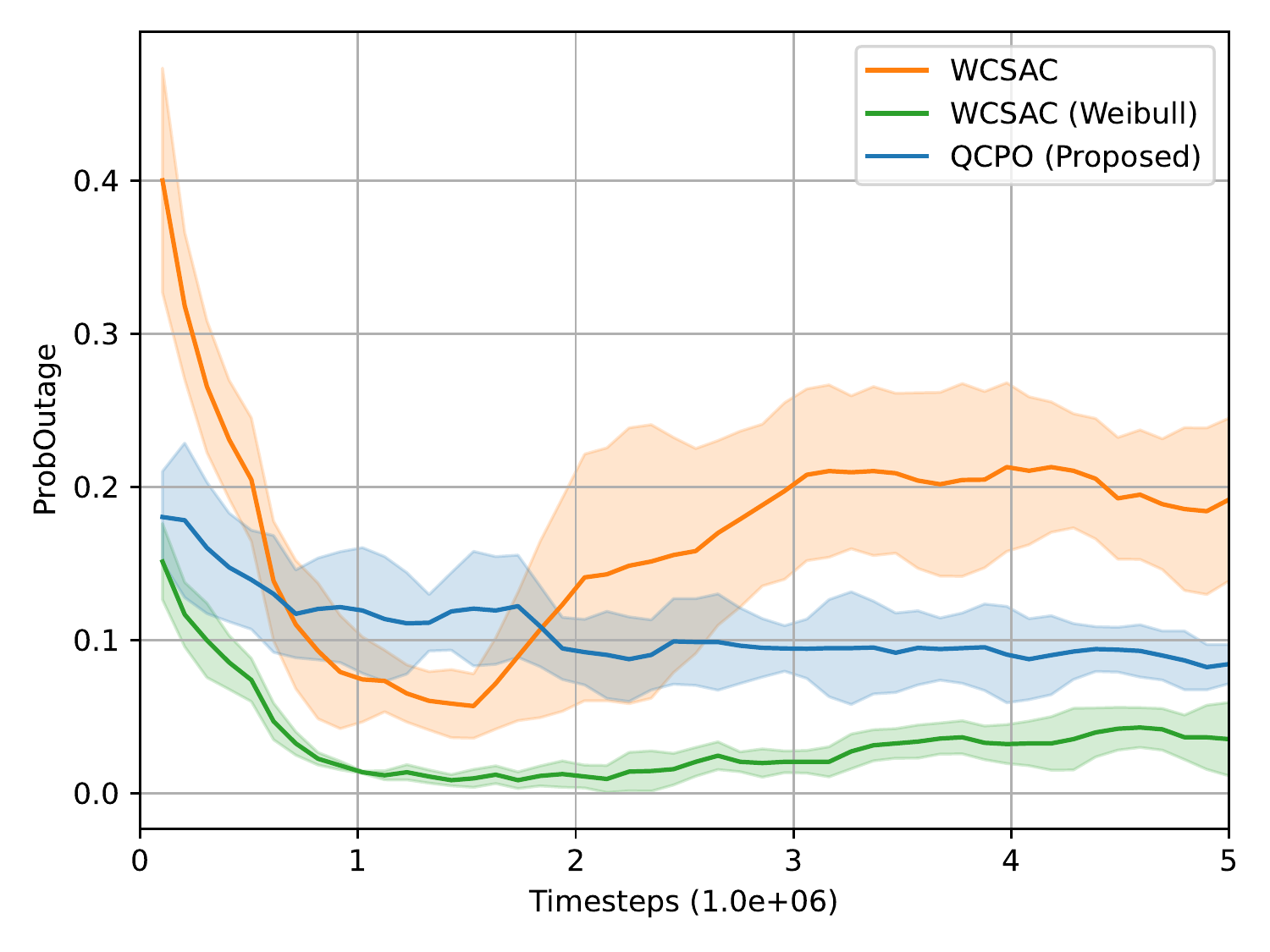}
        \caption{Outage Probability}
        \label{appendix:fig:WCSAC_Weibull_DynamicEnv_ProbOutage_01}
    \end{subfigure}
    \caption{Results of QCPO (blue), WCSAC\cite{yang2021wcsac} (orange) and WCSAC with Weibull distribution approximation (green) on DynamicEnv: (1st row) $\epsilon_0 = 0.2$, (2nd row) $\epsilon_0 = 0.1$, (left) average return of the most current 100 episodes and (right) outage probability of the most current 100 episodes.}
    \label{appendix:fig:WCSAC_Weibull_DynamicEnv_02_01}
\end{figure}

\subsection{Performance Comparison} \label{appendix:performance_comparison}

\begin{figure}[h]
    \centering
    \begin{subfigure}[b]{0.3\textwidth}
        \centering
        \includegraphics[width=\textwidth]{figures/SimpleButtonEnv_Return.pdf}
        \caption{Average Return}
        \label{appendix:fig:SimpleButtonEnv_Return}
    \end{subfigure}
    \begin{subfigure}[b]{0.3\textwidth}
        \centering
        \includegraphics[width=\textwidth]{figures/SimpleButtonEnv_ProbOutage.pdf}
        \caption{Outage Probability}
        \label{appendix:fig:SimpleButtonEnv_ProbOutage}
    \end{subfigure}
    \begin{subfigure}[b]{0.3\textwidth}
        \centering
        \includegraphics[width=\textwidth]{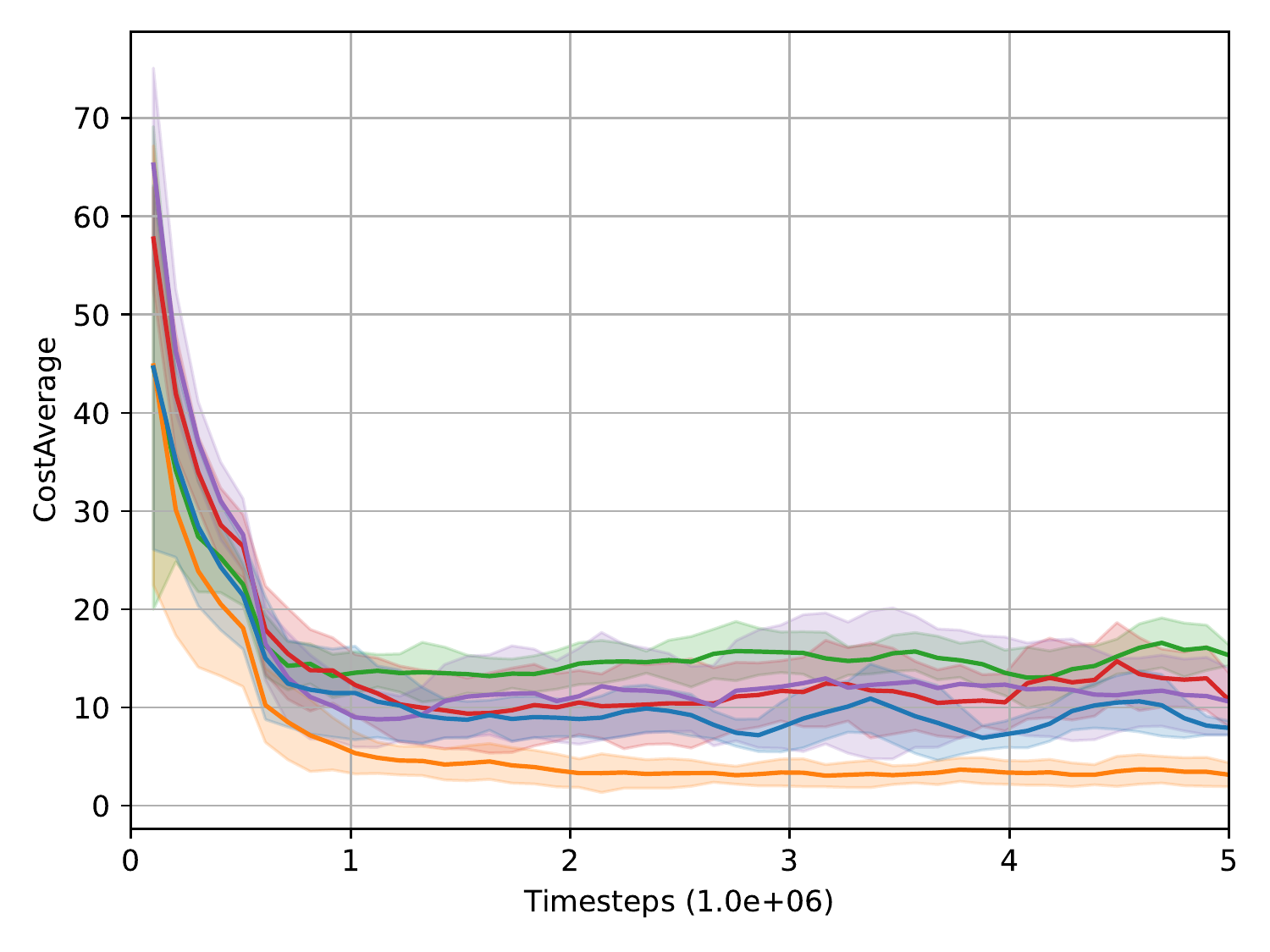}
        \caption{Average Sum Cost}
        \label{appendix:fig:SimpleButtonEnv_Cost}
    \end{subfigure}
    \begin{subfigure}[b]{0.3\textwidth}
        \centering
        \includegraphics[width=\textwidth]{figures/DynamicEnv_Return.pdf}
        \caption{Average Return}
        \label{appendix:fig:DynamicEnv_Return}
    \end{subfigure}
    \begin{subfigure}[b]{0.3\textwidth}
        \centering
        \includegraphics[width=\textwidth]{figures/DynamicEnv_ProbOutage.pdf}
        \caption{Outage Probability}
        \label{appendix:fig:DynamicEnv_ProbOutage}
    \end{subfigure}
    \begin{subfigure}[b]{0.3\textwidth}
        \centering
        \includegraphics[width=\textwidth]{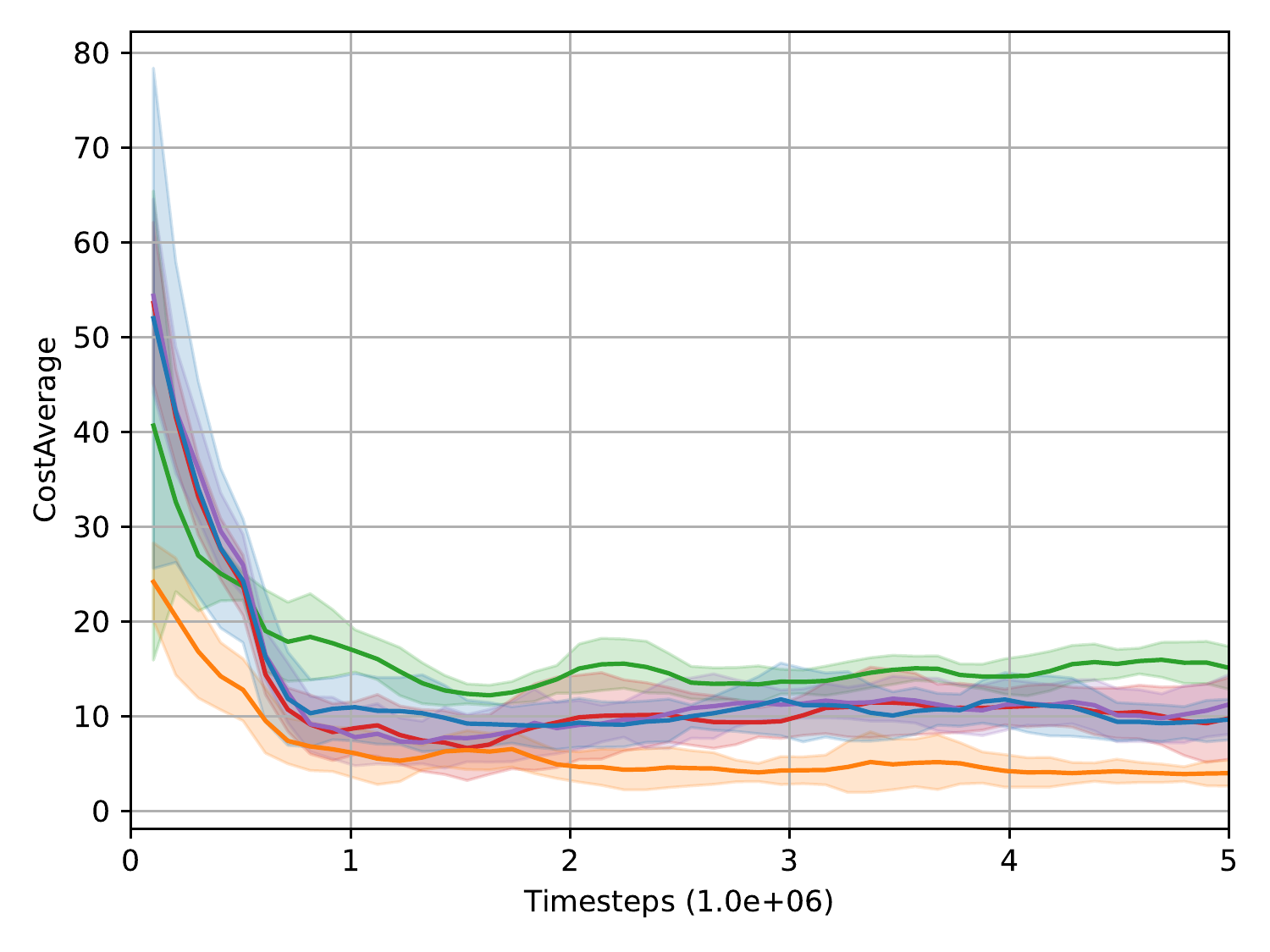}
        \caption{Average Sum Cost}
        \label{appendix:fig:DynamicEnv_Cost}
    \end{subfigure}
    \begin{subfigure}[b]{0.3\textwidth}
        \centering
        \includegraphics[width=\textwidth]{figures/GremlinEnv_Return.pdf}
        \caption{Average Return}
        \label{appendix:fig:GremlinEnv_Return}
    \end{subfigure}
    \begin{subfigure}[b]{0.3\textwidth}
        \centering
        \includegraphics[width=\textwidth]{figures/GremlinEnv_ProbOutage.pdf}
        \caption{Outage Probability}
        \label{appendix:fig:GremlinEnv_ProbOutage}
    \end{subfigure}
    \begin{subfigure}[b]{0.3\textwidth}
        \centering
        \includegraphics[width=\textwidth]{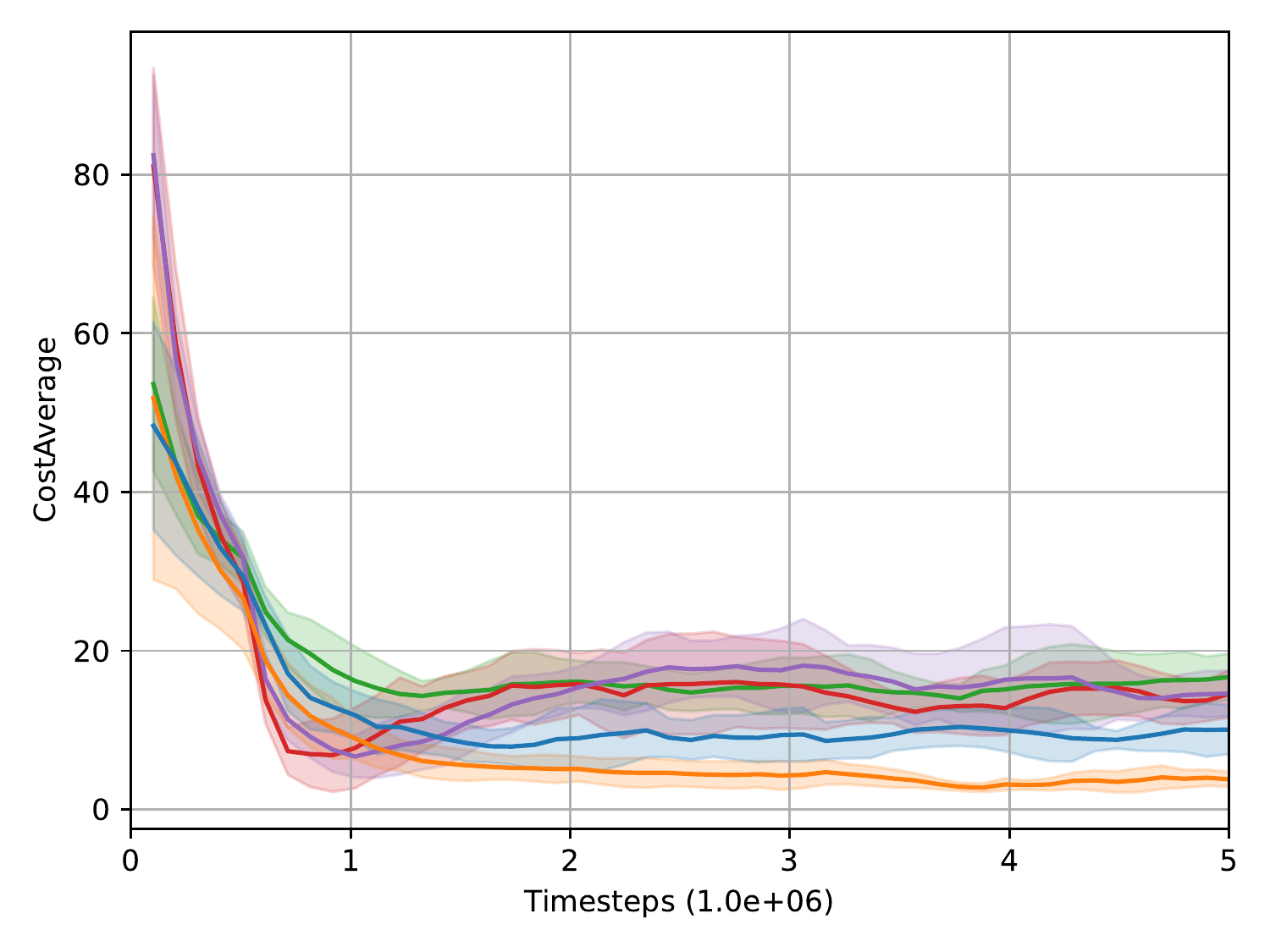}
        \caption{Average Sum Cost}
        \label{appendix:fig:GremlinEnv_Cost}
    \end{subfigure}
    \begin{subfigure}[b]{0.3\textwidth}
        \centering
        \includegraphics[width=\textwidth]{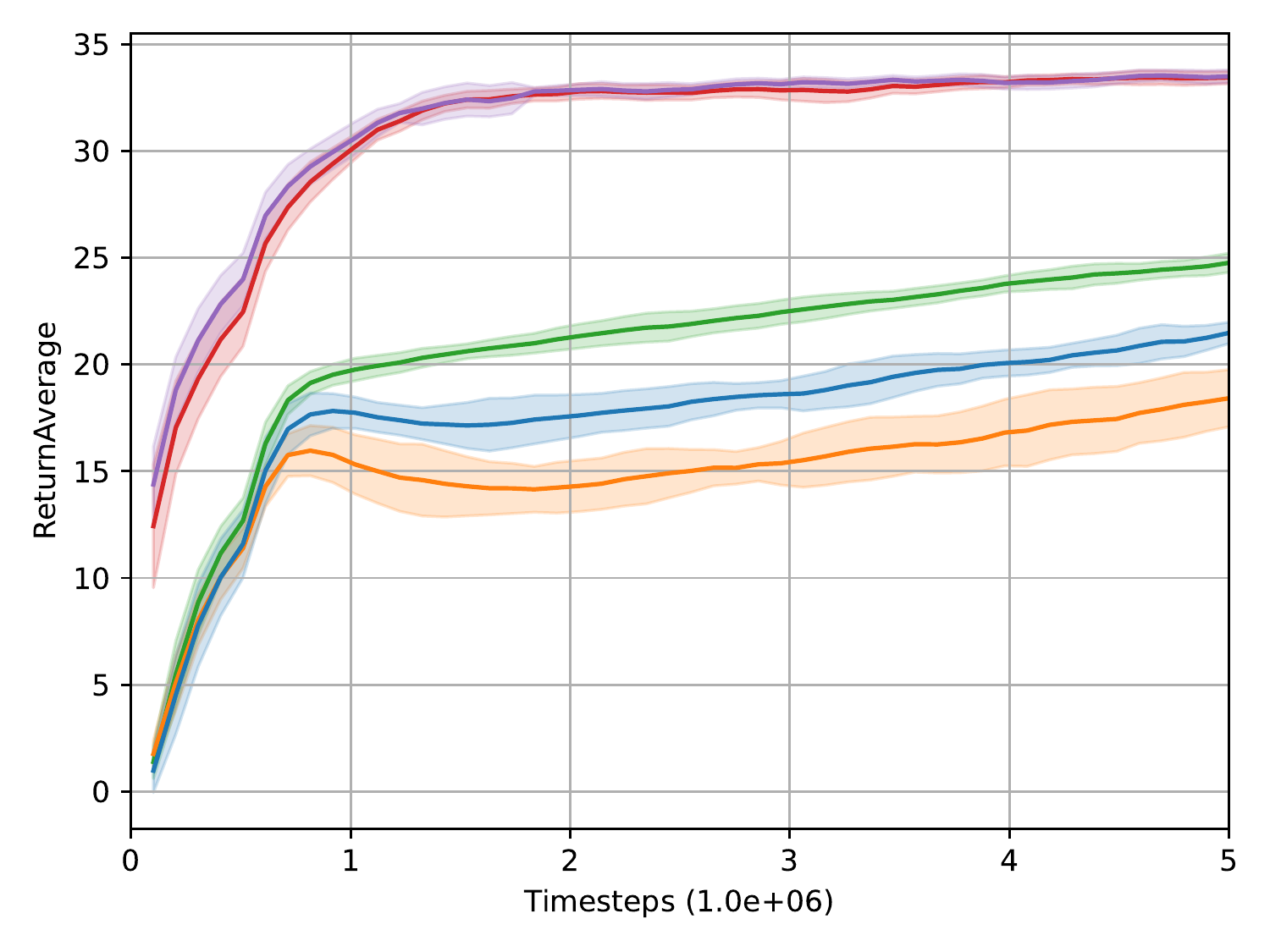}
        \caption{Average Return}
        \label{appendix:fig:DynamicButtonEnv_Return}
    \end{subfigure}
    \begin{subfigure}[b]{0.3\textwidth}
        \centering
        \includegraphics[width=\textwidth]{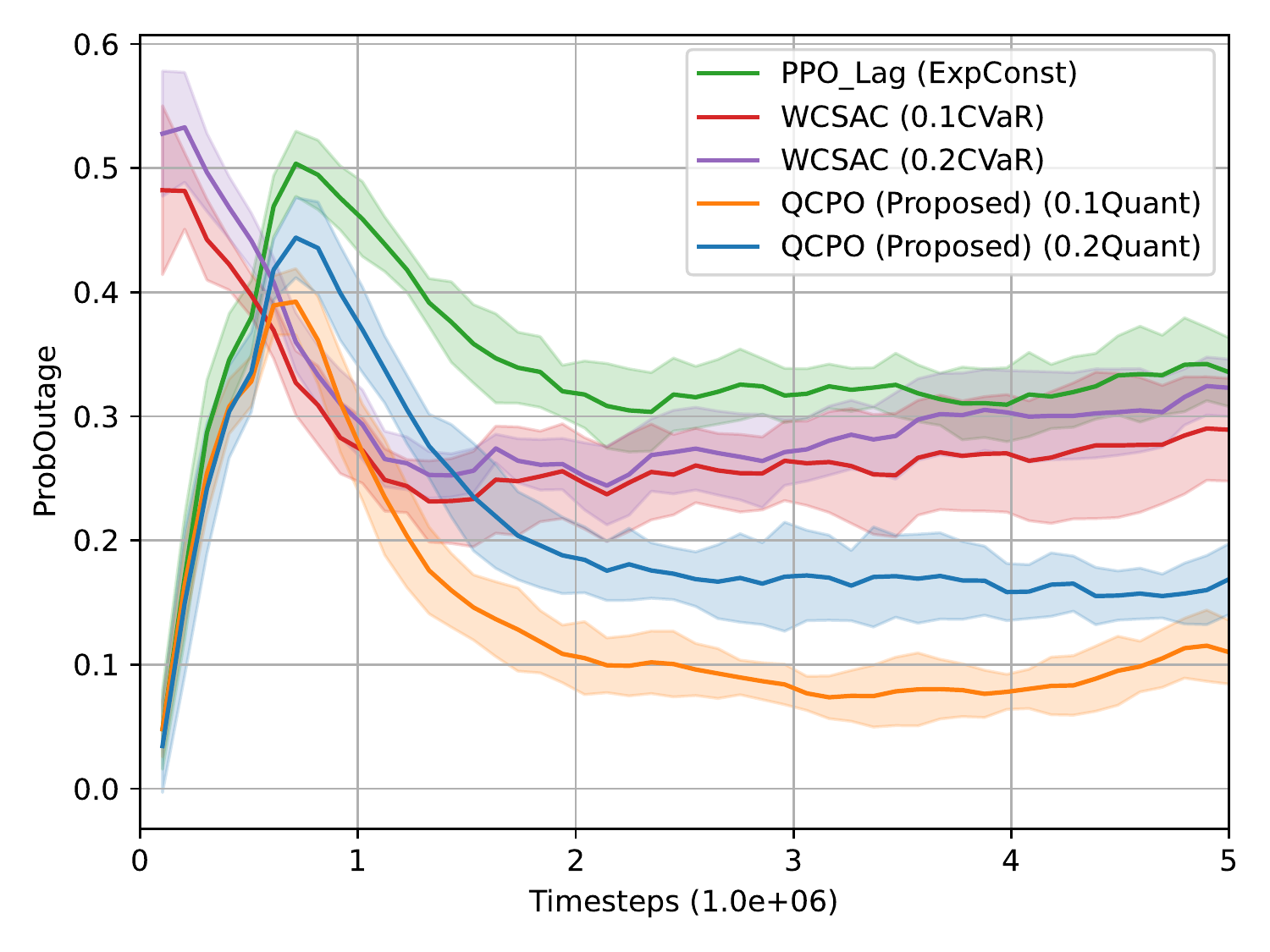}
        \caption{Outage Probability}
        \label{appendix:fig:DynamicButtonEnv_ProbOutage}
    \end{subfigure}
    \begin{subfigure}[b]{0.3\textwidth}
        \centering
        \includegraphics[width=\textwidth]{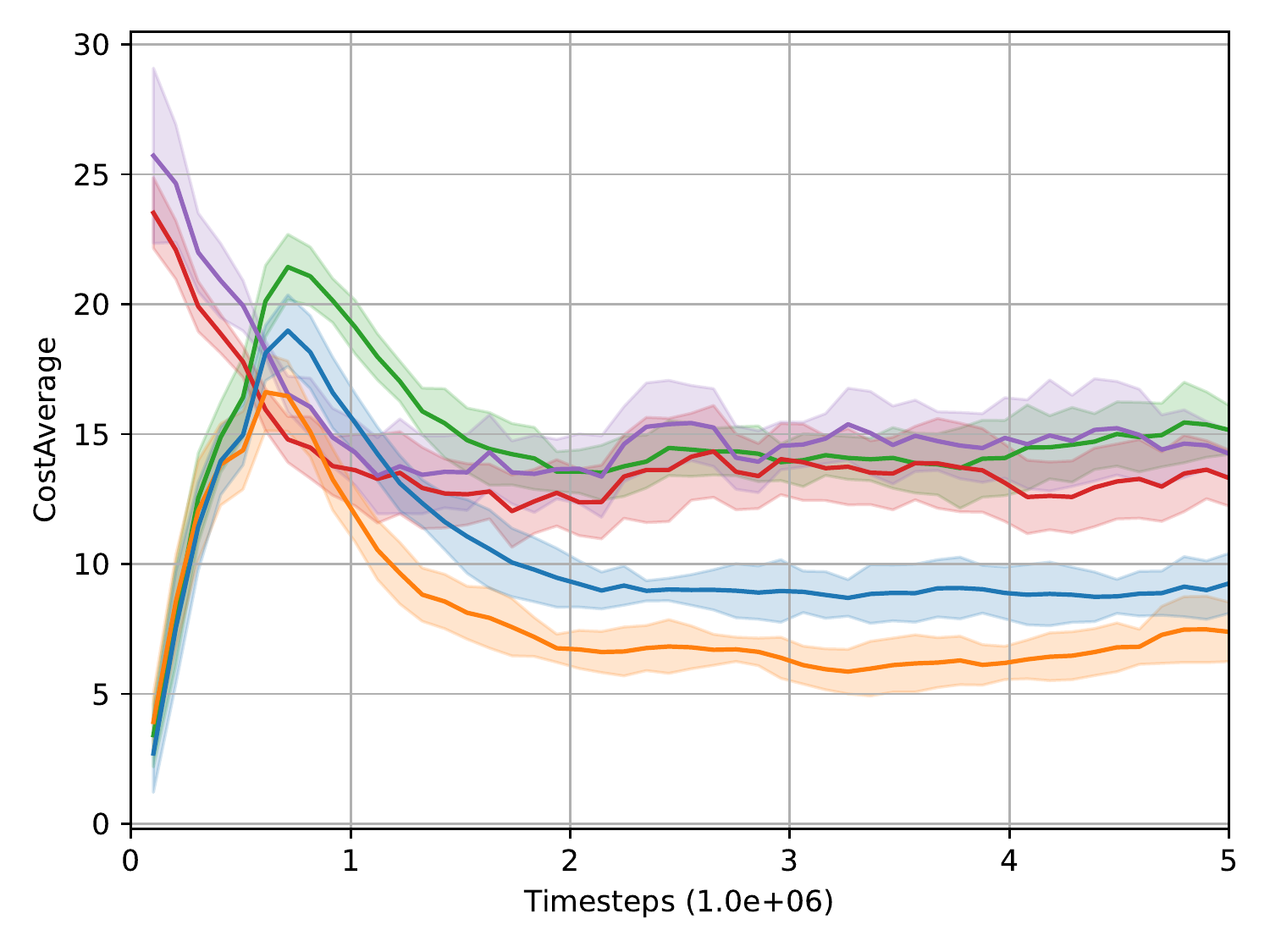}
        \caption{Average Sum Cost}
        \label{appendix:fig:DynamicButtonEnv_Cost}
    \end{subfigure}
    \caption{Results on SimpleButtonEnv (1st row), DynamicEnv (2nd row), GremlinEnv (3rd row), and DynamicButtonEnv (4th row): (1st column)  average return of the most current 100 episodes, (2nd column) outage probability of the most current 100 episodes, and (3rd column) average sum of costs of the most current 100 episodes.}
    \label{appendix:fig:result}
\end{figure}

Fig. \ref{appendix:fig:result} shows the results of the considered algorithms on SimpleButtonEnv, DynamicEnv, GremlinEnv, and DynamicButtonEnv explained in Appendix \ref{appendix:environments}. All experiments were done with 10 different random seeds, and the real line and the shaded area represent the average and average $\pm$ standard deviation, respectively. PPO with the Lagrangian multiplier method for (ExpCP) (green) keeps the average of the sum cost around the threshold $d_{th} = 15$ well (see Fig. \ref{appendix:fig:SimpleButtonEnv_Cost}, \ref{appendix:fig:DynamicEnv_Cost}, \ref{appendix:fig:GremlinEnv_Cost}, and \ref{appendix:fig:DynamicButtonEnv_Cost}), and its outage probability is around $0.35$ on SimpleButtonEnv and DynamicEnv (Fig. \ref{appendix:fig:SimpleButtonEnv_ProbOutage} and \ref{appendix:fig:DynamicEnv_ProbOutage}), and $0.3$ on GremlinEnv and DynamicButtonEnv (Fig.  \ref{appendix:fig:GremlinEnv_ProbOutage} and \ref{appendix:fig:DynamicButtonEnv_ProbOutage}). Note that the CVaR approach (WCSAC) should satisfy a sufficient condition for satisfying the outage probability constraint in (ProbCP). It is seen  that WCSAC ($\epsilon_0 = 0.2$ (purple), $\epsilon_0 = 0.1$ (red)) achieves a lower or similar outage probability to the threshold $\epsilon_0$ in Fig. \ref{appendix:fig:SimpleButtonEnv_ProbOutage}, but the algorithm does not satisfy the outage probability constraint exactly in Fig. \ref{appendix:fig:DynamicEnv_ProbOutage}, \ref{appendix:fig:GremlinEnv_ProbOutage}, and \ref{appendix:fig:DynamicButtonEnv_ProbOutage}. This means that the Gaussian distribution approximation of the distribution of $X^\pi(s)$ has limited capability to capture the decay rate of the tail probability. On the other hand, the proposed QCPO ($\epsilon_0 = 0.2$ (blue), $\epsilon_0 = 0.1$ (orange)) maintains the outage probability around the desired target outage probability very well, as shown in Fig. \ref{appendix:fig:SimpleButtonEnv_ProbOutage}, \ref{appendix:fig:DynamicEnv_ProbOutage}, \ref{appendix:fig:GremlinEnv_ProbOutage}, and \ref{appendix:fig:DynamicButtonEnv_ProbOutage}.

Now consider the average return of these algorithms. In constrained RL, in general, if an algorithm is allowed to have a higher sum of costs, then it has a higher return. Thus, as seen in Fig. \ref{appendix:fig:SimpleButtonEnv_ProbOutage}, \ref{appendix:fig:DynamicEnv_ProbOutage}, \ref{appendix:fig:GremlinEnv_ProbOutage}, and \ref{appendix:fig:DynamicButtonEnv_ProbOutage}, PPO\_Lag induces the highest outage probability, so it has the highest average return, as shown in Fig. \ref{appendix:fig:SimpleButtonEnv_Return}, \ref{appendix:fig:DynamicEnv_Return}, and \ref{appendix:fig:GremlinEnv_Return}. (For DynamicButtonEnv, WCSAC outperforms PPO\_Lag, and this is because that SAC, the base algorithm of WCSAC, is a better algorithm than PPO, the base algorithm of PPO\_Lag, on most unconstrained environments.) The direct comparison between WCSAC and QCPO is less meaningful in DynamicEnv, GremlinEnv, and DynamicButtonEnv since WCSAC does not satisfy the outage probability constraint, but it is fair in SimpleButtonEnv because both algorithms satisfy the outage probability constraint. As seen in Fig. \ref{appendix:fig:SimpleButtonEnv_Return}, QCPO achieves a higher average return than WCSAC for the same target probability constraint $\epsilon_0 = 0.1, 0.2$. This is because QCPO  satisfies the target outage probability exactly, i.e., uses the given cost budget fully for a higher return.

\newpage

\section{Implementation Details}\label{appendix:implementation_details}
The implementation of the proposed algorithm\footnote{\url{https://github.com/wyjung0625/QCPO}, (MIT License)}  is based on the implementation of \cite{stooke2020responsive}\footnote{\url{https://github.com/astooke/rlpyt/tree/master/rlpyt/projects/safe}, (MIT License)} 

\subsection{Network structures} \label{appendix:network_structures}

Since the proposed algorithm is based on PPO \cite{schulman2017proximal}, the network structure is  similar to the network structure of PPO. The networks for the value function, the quantile function, the policy, and the Weibull distribution parameters have a common shared network to extract a feature of its observation. The common network has two MLP layers of size $512$ with the tanh activation function, and an LSTM layer of size $512$ with tanh activation function. The current observation changes to its feature  through the two MLP layers, then concatenates this feature of the current observation, the previous action, the previous reward, and the previous cost to input the LSTM network. Thus, the common network outputs a feature of all previous information in the current trajectory. The output of the LSTM layer is then used as the input of the uncommon parts of the functions. The value function for reward has a linear MLP layer of size 1, and the quantile function for cost has a MLP layer of size $n_q$ (number of quantile estimates) with exponential activation $\exp(x)$. Thus, the feature computed by the common feature network goes through these MLP networks to compute its value $V^\pi(s)$ and its quantile $q^\pi_u(s)$ for $u \in \{ u_1, u_2, \ldots, u_{n_q} \}$. The policy network has a linear MLP layer of size 1, which outputs the mean parameter of Gaussian distribution, and a variable which indicates state-independent log standard deviation for Gaussian distribution. For the Weibull distribution parameters, there are two networks, one for $\alpha(s)$ and the other for $\beta(s)$, having a MLP layer of size 1. For $\alpha(s)$, the network has 4 * sigmoid activation function, and for $\beta(s)$, the network has the exponential activation function. 

\subsection{Loss Functions} \label{appendix:loss_functions}
The parameters are updated by minimizing their own loss functions. The loss function of the value parameter $\phi$ is 
\begin{equation}
    L(\phi) = \frac{1}{2} \hat{\Ebb}_{s \sim \rho^\pi}\left[ \left\Vert V_\phi(s) - R  \right\Vert^2 \right],
\end{equation}
where $R$ is a sampled return at $s$, and $\hat{\Ebb}$ is the sample mean for $s$ drawn from $\rho^\pi$. This loss function is the same as that of PPO. For the quantile function, the loss function is composed of two losses. The first one is the value parameter loss for cost,  defined as
\begin{equation}
    L_{value}(\psi) = \frac{1}{2} \hat{\Ebb}_{s \sim \rho^\pi}\left[ \left\Vert \frac{1}{n_q} \sum^{n_q}_{i=1} q_{\psi, u_i}(s) - C  \right\Vert^2 \right],
\end{equation}
where $C$ is the sampled cumulative sum cost at $s$. Note that value function for cost is computed as $C^\pi(s) := \Ebb_{\pi}\left[ \sum^\infty_{t=0} \gamma^t c(s_t, a_t) \right] = \int^1_0 q^\pi_u(s) ~du \approx \frac{1}{n_q} \sum^{n_q}_{i=1} q_{\psi, u_i}(s)$. The second loss for the quantile function is the quantile loss $l_{Huber, u_i}(x)$ (for definition, please see Appendix \ref{appendix:distributional_rl}) with the Huber loss $L_\kappa(x)$,  defined as
\begin{align*}
    L_{quant}(\psi) &= \frac{1}{n_q^2} \sum^{n_q}_{i,j=1} \hat{\Ebb}_{(s, a, s') \sim \pi} \left[ l_{Huber, u_i}(\delta_{ij}(s, a, s')) \right]\\
    \delta_{ij}(s, a, s') &= c(s, a) + \gamma q_{\psi_{old}, u_j}(s') - q_{\psi, u_i}(s),
\end{align*}
where $(s, a, s') \sim \pi$ means that $s \sim \rho^\pi(\cdot)$, $a \sim \pi(\cdot | s)$, and $s' \sim M(\cdot | s, a)$, $\psi_{old}$ is a copied parameter of $\psi$ which does not update when $\psi$ updates. Thus, the loss function for the quantile function parameter $\psi$ is given by
\begin{equation}
    L(\psi) = L_{value}(\psi) + L_{quant}(\psi).
\end{equation}
The parameters $\xi$ and $\zeta$ for estimating the Weibull distribution parameters $\alpha_\xi(s)$ and $\beta_\zeta(s)$ at state $s$ are updated by minimizing the following loss function:
\begin{equation}
    L(\xi, \zeta) = \hat{\Ebb}_{s \sim \rho^\pi} \left[ \frac{1}{k} \sum^{n_q}_{i=n_q-k+1} \frac{1}{2} \left\Vert \log{\beta_\zeta(s)} + \frac{\log{c_{u_i}}}{\alpha_\xi(s)} - \log{q_{\psi, u_i}(s)} \right\Vert^2  \right], \label{appendix:eq:weibull_tail_loss}
\end{equation}
where $c_u = -\log{(1-u)}$. Note that the $u$-quantile of Weibull distribution with parameters $\alpha$ and $\beta$ is $\beta \cdot (c_u)^{1/\alpha}$. Thus, \eqref{appendix:eq:weibull_tail_loss} is the mean square error of log-scale of the $u$-quantile for $u \in \{ u_{n_q-k+1}, \ldots, u_{n_q} \}$.

\subsection{Policy Loss Function} \label{appendix:policy_loss}

As aforementioned in Section 4, the basic policy loss function of QCPO for a given Lagrange multiplier $\lambda$ is  
\begin{align}
    L^{\pi_{old}}(\pi_\theta) - \tilde{C}_1 \max_s \KL(\pi_{old}(\cdot |s) ~\Vert~ \pi_{\theta}(\cdot |s)) 
\end{align}
where 
\begingroup
\allowdisplaybreaks
\begin{align}
    L^{\pi_{old}}(\pi_\theta)  &= \left( V^{\pi_{old}}(s_0) - \lambda q^{\pi_{old}}_{1-\epsilon_0}(s_0) \right) + \Ebb_{\pi_{old}}\left[ \sum^\infty_{t=0} \gamma^t  \Ebb_{a \sim \pi_{\theta}}\left[  A^{\pi_{old}}_r(s_t, a) - \lambda A^{\pi_{old}}_{1-\epsilon_0}(s_t, a) \right] \right] \nonumber \\
    &= \left( V^{\pi_{old}}(s_0) - \lambda q^{\pi_{old}}_{1-\epsilon_0}(s_0) \right) + \Ebb_{s \sim \rho^{\pi_{old}}, a \sim \pi_\theta}\left[  A^{\pi_{old}}_r(s_t, a) - \lambda A^{\pi_{old}}_{1-\epsilon_0}(s_t, a) \right] \\
    A^{\pi_{old}}_r (s, a) &= r(s,a) + \gamma \Ebb_{s' \sim M(\cdot | s, a)}\left[  V^{\pi_{old}}(s') \right] - V^{\pi_{old}}(s) \\
    A^{\pi_{old}}_{1-\epsilon_0}(s, a) &= c(s,a) + \tilde{c}^{\pi_{old}}_{1-\epsilon_0}(s, a) + \gamma \Ebb_{s' \sim M(\cdot | s, a)} \left[  q^{\pi_{old}}_{1-\epsilon_0}(s') \right] - q^{\pi_{old}}_{1-\epsilon_0}(s), \label{appendix:eq:advantage_quantile}
\end{align}
\endgroup

Note that
\begingroup
\allowdisplaybreaks
\begin{align}
    \Ebb_{\pi_{old}}\left[ A^{\pi_{old}}_{1-\epsilon_0}(s, a) \right] &= \Ebb_{\pi_{old}}\left[ c(s,a) + \tilde{c}^{\pi_{old}}_{1-\epsilon_0}(s, a) + \gamma q^{\pi_{old}}_{1-\epsilon_0}(s') - q^{\pi_{old}}_{1-\epsilon_0}(s) \right] \\
    &= \Ebb_{\pi_{old}}\left[ \frac{p_{X^{\pi_{old}}(s')}\left( \frac{q^{\pi_{old}}_{1-\epsilon_0}(s) - c(s, a)}{\gamma} \right)}{\gamma \cdot p_{X^{\pi_{old}}(s)}\left( q^{\pi_{old}}_{1-\epsilon_0}(s) \right)} \left\{ c(s,a) + \gamma q^{\pi_{old}}_{1-\epsilon_0}(s') - q^{\pi_{old}}_{1-\epsilon_0}(s) \right\} \right]
\end{align}
\endgroup
since
\begingroup
\allowdisplaybreaks
\begin{align}
    \tilde{c}^{\pi_{old}}_u(s, a) &= \left( \frac{p_{X^{\pi_{old}}(s')}\left( \frac{q^{\pi_{old}}_{1-\epsilon_0}(s) - c(s, a)}{\gamma} \right)}{\gamma \cdot p_{X^{\pi_{old}}(s)}\left( q^{\pi_{old}}_{1-\epsilon_0}(s) \right)} - 1 \right) \left\{ c(s,a) + \gamma q^\pi_{1-\epsilon_0}(s') \right\} \\
    1 &= \Ebb_{\pi_{old}} \left[ \frac{p_{X^{\pi_{old}}(s')}\left( \frac{q^{\pi_{old}}_{1-\epsilon_0}(s) - c(s, a)}{\gamma} \right)}{\gamma \cdot p_{X^{\pi_{old}}(s)}\left( q^{\pi_{old}}_{1-\epsilon_0}(s) \right)} \right] 
\end{align}
\endgroup
Therefore we use $\frac{p_{X^{\pi_{old}}(s')}\left( \frac{q^{\pi_{old}}_{1-\epsilon_0}(s) - c(s, a)}{\gamma} \right)}{\gamma \cdot p_{X^{\pi_{old}}(s)}\left( q^{\pi_{old}}_{1-\epsilon_0}(s) \right)} \left\{ c(s,a) + \gamma q^{\pi_{old}}_{1-\epsilon_0}(s') - q^{\pi_{old}}_{1-\epsilon_0}(s) \right\}$ as the advantage for the $(1-\epsilon_0)$-quantile.
\begin{equation}
    A^{\pi_{old}}_{1-\epsilon_0}(s, a) = \frac{p_{X^{\pi_{old}}(s')}\left( \frac{q^{\pi_{old}}_{1-\epsilon_0}(s) - c(s, a)}{\gamma} \right)}{\gamma \cdot p_{X^{\pi_{old}}(s)}\left( q^{\pi_{old}}_{1-\epsilon_0}(s) \right)} \left\{ c(s,a) + \gamma q^{\pi_{old}}_{1-\epsilon_0}(s') - q^{\pi_{old}}_{1-\epsilon_0}(s) \right\}
\end{equation}

Finally QCPO is based on PPO\cite{schulman2017proximal}, the actual policy loss function is as follows:
\begin{equation}
    L(\theta) = - \hat{\Ebb}_{\pi_{\theta_{old}}} \left[ \min \left\{ \hat{A}_1(s, a, s'), \hat{A}_2(s,a,s') \right\} \right]
\end{equation}
where
\begingroup
\allowdisplaybreaks
\begin{align}
    \hat{A}_1(s,a, s') &= \text{clip}\left( \frac{\pi_\theta(a|s)}{\pi_{\theta_{old}}(a|s)}, 1 - r_{clip}, 1 + r_{clip} \right) \times \hat{A} (s, a, s') \label{appendix:eq:clip_importance_ratio_advantage}\\
    \hat{A}_2(s,a, s') &= \frac{\pi_\theta(a|s)}{\pi_{\theta_{old}}(a|s)} \hat{A} (s, a, s') \\
    \hat{A}(s,a, s') &:= \hat{A}_r(s, a, s') - \lambda \hat{A}_{q, 1 - \epsilon}(s, a, s') \label{appendix:eq:advantage_total}\\
    \hat{A}_r(s,a, s') &:= r(s,a) + \gamma V_{\phi_{old}}(s') - V_{\phi_{old}}(s) \\
    \hat{A}_{1-\epsilon_0}(s, a, s') &= \frac{p_{\alpha_{old}(s'), \beta_{old}(s')}\left( \frac{q_{\psi_{old}, 1-\epsilon_0}(s) - c(s, a)}{\gamma} \right)}{\gamma \cdot p_{\alpha_{old}(s'), \beta_{old}(s')}\left( q_{\psi_{old}, 1-\epsilon_0}(s) \right)} \left\{ c(s,a) + \gamma q_{\psi_{old}, 1-\epsilon_0}(s') - q_{\psi_{old}, 1-\epsilon_0}(s) \right\} \label{appendix:eq:advantage_quantile_approx1}
\end{align}
\endgroup
Here,
\begin{align}
    p_{\alpha_{old}(s'), \beta_{old}(s')}(x) &= \frac{\alpha_{\xi_{old}}(s')}{\beta_{\zeta_{old}}(s')} \left( \frac{x}{\beta_{\zeta_{old}}(s')} \right)^{\alpha_{\xi_{old}}(s') - 1} \exp{\left( - \left( \frac{x}{\beta_{\zeta_{old}}(s')} \right)^{\alpha_{\xi_{old}}(s')} \right)}
\end{align}
is the probability density function (PDF) of the approximated weibull distribution with parameter $\alpha_{\xi_{old}}(s')$ and $\beta_{\zeta_{old}}(s')$.

However, the variance of the ratio $\frac{p_{\alpha_{old}(s'), \beta_{old}(s')}\left( \frac{q_{\psi_{old}, 1-\epsilon_0}(s) - c(s, a)}{\gamma} \right)}{\gamma \cdot p_{\alpha_{old}(s'), \beta_{old}(s')}\left( q_{\psi_{old}, 1-\epsilon_0}(s) \right)}$ is large with actual samples. Hence, for  implementation, using the Taylor series $\log x = (x-1)+\frac{1}{2}(x-1)^2+\cdots$ around $x=1$, we smooth the weight as
\begin{align}
    &\frac{p_{\alpha_{old}(s'), \beta_{old}(s')}\left( \frac{q_{\psi_{old}, 1-\epsilon_0}(s) - c(s, a)}{\gamma} \right)}{\gamma \cdot p_{\alpha_{old}(s'), \beta_{old}(s')}\left( q_{\psi_{old}, 1-\epsilon_0}(s) \right)} \nonumber \\
    &\approx \left( 1 + \text{clip}\left( \log{\frac{p_{\alpha_{old}(s'), \beta_{old}(s')}\left( \frac{q_{\psi_{old}, 1-\epsilon_0}(s) - c(s, a)}{\gamma} \right)}{\gamma \cdot p_{\alpha_{old}(s'), \beta_{old}(s')}\left( q_{\psi_{old}, 1-\epsilon_0}(s) \right)}}, -c_{\text{clip}}, c_{\text{clip}} \right) \right)
    \label{appendix:eq:probability_ratio_approx}
\end{align}
and apply \eqref{appendix:eq:probability_ratio_approx} into \eqref{appendix:eq:advantage_quantile_approx1} so the actual advantage estimate we used is
\begin{align}
    &\hat{A}_{1-\epsilon_0}(s, a, s') \nonumber \\
    &= \left( 1 + \text{clip}\left( \log{\frac{p_{\alpha_{old}(s'), \beta_{old}(s')}\left( \frac{q_{\psi_{old}, 1-\epsilon_0}(s) - c(s, a)}{\gamma} \right)}{\gamma \cdot p_{\alpha_{old}(s'), \beta_{old}(s')}\left( q_{\psi_{old}, 1-\epsilon_0}(s) \right)}}, -c_{\text{clip}}, c_{\text{clip}} \right) \right) \nonumber \\
    &~~~~~ \times \left\{ c(s,a) + \gamma q_{\psi_{old}, 1-\epsilon_0}(s') - q_{\psi_{old}, 1-\epsilon_0}(s) \right\} \label{appendix:eq:advantage_quantile_approx2}
\end{align}

\subsection{Lagrange Multiplier for Quantile Constraint} \label{appendix:lagrange_multiplier}
We also need the Lagrange multiplier $\lambda$ in \eqref{appendix:eq:advantage_total} for the policy loss function to satisfy the quantile constraint. The Lagrange multiplier is updated to minimize the Lagrange form of \eqref{appendix:problem:quantile} $L_{quant}(\pi, \lambda) := V^\pi(s_0) - \lambda \left( q^\pi_{1 - \epsilon}(s_0) - d_{th} \right)$ to satisfy the quantile constraint $q^\pi_{1-\epsilon}(s_0) \leq d_{th}$. Thus, the update rule of the Lagrange multiplier $\lambda$ is $\lambda \leftarrow \max \{ \lambda + \eta (q^\pi_{1-\epsilon}(s_0) - d_{th}), 0\}$, where $\eta$ is a learning rate. To constrain the outage probability of the sum of costs in a trajectory, we collect 100 trajectories, and compute the $(1-\epsilon)$-quantile of them to replace $q^\pi_{1-\epsilon}(s_0)$ in the Lagrange update rule.

\subsection{Hyper-parameters} \label{appendix:hyperparameters}

For the quantile network, we used $n_q = 25$, and $u_i = \frac{2i-1}{2 n_q} = \frac{2i-1}{50}$, $i = 1, \ldots, n_q (= 25)$ for $\{ u_1, u_2, \ldots, u_{n_q} \}$. For training the Weibull network, we used the rightmost $k$-quantiles among $n_q$, and the $k$ is $8$ ($\approx$ 30\% of $n_q$ quantiles). The discount factor $\gamma$ is $0.99$, and all learning rates for Adam optimizers for all parameters are $10^{-4}$. The $\eta$ for updating the Lagrange multiplier is $0.1$. The $r_{clip}$ in \eqref{appendix:eq:clip_importance_ratio_advantage} for updating the policy parameter is $0.1$, and the $c_{clip}$ in \eqref{appendix:eq:advantage_quantile_approx2} for computing $\tilde{c}^{\pi'}_u(s,a)$ is $0.5$. Since PPO is an on-policy algorithm, it first collects $12000$ samples by interaction with its environment. Then, it reshapes these samples by $120$ sub-trajectories of length $100$, and uses all sub-trajectories to update its parameters (this is because we use LSTM for the feature extraction network). This update is performed $8$ times for the same collected sub-trajectories, then we remove them and collect new $12000$ samples by interaction with the  environment. This procedure is performed until the maximum training timesteps $5 \times 10^6$.

\end{document}